\documentclass{article}
\usepackage[utf8]{inputenc}
\usepackage[margin=1in]{geometry}

\usepackage{microtype}
\usepackage{graphicx}
\usepackage{subfigure}
\usepackage{booktabs}
\usepackage{natbib}
\setcitestyle{open={(},close={)}}

\usepackage[colorlinks=true,citecolor=blue]{hyperref}

\usepackage{amsmath,amsthm,amsfonts,amssymb,mathdots,array,mathrsfs,bm,bbm,stmaryrd,graphicx,subfigure,xcolor}
\usepackage{algpseudocode,algorithm,algorithmicx}
\usepackage{breakcites}

\usepackage[T1]{fontenc}
\usepackage{enumerate}
\usepackage{inputenc}

\usepackage{graphicx} 
\usepackage{subfigure}

\usepackage{booktabs,balance}
\usepackage{rotating}
\usepackage{boldline}
\usepackage{makecell}
\usepackage{multirow}
\usepackage{balance}

\usepackage{tikz}

\newtheorem{theorem}{Theorem}[section]
\newtheorem{corollary}[theorem]{Corollary}
\newtheorem{lemma}[theorem]{Lemma}
\newtheorem{proposition}[theorem]{Proposition}
\newtheorem{property}{Property}

\newtheorem{definition}{Definition}[section]

\newtheorem{remark}{Remark}[section]
\newtheorem{observation}{Observation}[section]

\newcommand{\bsmat}{\begin{bmatrix} }
\newcommand{\esmat}{\end{bmatrix} }

\usepackage{environ}
\NewEnviron{smallequation}{%
    \begin{equation}
    \scalebox{0.97}{$\BODY$}
    \end{equation}
    }
    \NewEnviron{smallalign}{%
    \begin{equation}
    \scalebox{0.97}{$\BODY$}
    \end{equation}
    }

\begin{document}

\title{\bf Quantization Algorithms for Random Fourier Features}

\author{\textbf{Xiaoyun Li} \\
Department of Statistics\\
Rutgers University\\
110 Frelinghuysen Rd.
Piscataway, NJ 08854\\
  \texttt{xiaoyun.li@stat.rutgers.edu}
\and
\textbf{Ping Li} \\
Cognitive Computing Lab\\
Baidu Research\\
10900 NE 8th St. Bellevue, WA 98004\\
  \texttt{liping11@baidu.com}
}
\date{\vspace{0.1in}}
\maketitle

\begin{abstract}

\noindent The method of random projection (RP) is the standard technique in machine learning and many other areas, for dimensionality reduction, approximate near neighbor search, compressed sensing, etc. Basically, RP provides a simple and effective scheme for approximating pairwise inner products and Euclidean distances in massive data.  Closely related to RP, the method of random Fourier features (RFF) has also become popular, for approximating the Gaussian kernel.  RFF applies a specific nonlinear transformation on the projected data  from random projections. In practice, using the (nonlinear) Gaussian kernel often leads to better performance than the linear kernel (inner product), partly due to the tuning parameter $(\gamma)$ introduced in the Gaussian kernel. Recently, there has been a surge of interest in studying~properties~of~RFF. \\

\noindent After random projections, quantization is an important step for efficient data storage, computation, and transmission. Quantization for RP has also been extensive studied in the literature. In this paper, we focus on developing quantization algorithms for RFF. The task is in a sense challenging due to the tuning parameter $\gamma$ in the Gaussian kernel. For example, the quantizer and the quantized data might be tied to each specific tuning parameter $\gamma$. Our contribution begins with an interesting discovery, that the marginal distribution of RFF is actually free of the Gaussian kernel parameter $\gamma$. This small finding significantly simplifies the design of the Lloyd-Max (LM) quantization scheme for RFF in that there would be only one LM quantizer for RFF (regardless of $\gamma$). We also develop a variant named LM$^2$-RFF quantizer, which in certain cases is more accurate. Experiments confirm that the proposed quantization schemes perform well.

\end{abstract}

\section{Introduction}
In recent years, machine learning with extremely large-scale (and high-dimensional) datasets has become increasingly important given the rapid development of modern technologies. Many industrial applications involve massive data collected from a wide range of sources, e.g., internet and mobile devices. Designing efficient large-scale learning algorithms and feature engineering techniques, in terms of both speed and memory, has been an important topic in the machine learning \& data mining community. The method of Random Projection (RP) is a popular strategy to deal with massive data, for example, for efficient data processing, computations,  storage, or transmissions. The theoretical merit of RP is highlighted by the celebrated Johnson-Lindenstrauss Lemma~\citep{Article:JL84}, which states that with high probability the Euclidean distance between data points is approximately preserved in the projected space provided that the number of projections is sufficiently large. In the past two decades or so, RP has been used extensively in dimensionality reduction, approximate near neighbor search, compressed sensing, computational biology, etc. See some examples of relatively  early works on RP~\citep{Proc:Dasgupta_UAI00,Proc:Bingham_KDD01,Article:Buher_01,Article:Achlioptas_JCSS03,Proc:Fern_ICML03,Proc:Datar_SCG04,Article:Candes_IT06,Article:Donoho_IT06,Proc:Li_Hastie_Church_COLT06,Proc:Frund_NIPS07,Proc:Li_KDD07}. In this paper, we continue the line of research on random projections and focus on studying quantization schemes for using random Fourier features (RFF), which are nonlinear transformations of random projections, to accurately approximate the (nonlinear) Gaussian kernel.

\subsection{Linear Kernel and Gaussian Kernel}

Let $u,v\in \mathcal X \subseteq \mathbb R^d$ denote  two $d$-dimensional data vectors. The linear kernel is simply the inner product $\langle u, v\rangle = u^Tv$. For training large-scale linear learning algorithms such as linear support vector machine (SVM) and linear logistic regression, highly efficient (in both memory and time) training algorithms have been available and widely used in practice~\citep{Proc:Joachims_KDD06,Article:Shalev-Shwartz_MP11,Article:Fan_JMLR08}. Despite their high efficiency,  the drawback of linear learning methods is that they often do not provide a good accuracy as they neglect the nonlinearity of data. This motivates researchers to find an efficient training algorithms for nonlinear kernels such as the Gaussian kernel~\citep{Book:Hastie_Tib_Friedman,Book:Scholkopf_02,Book:Bottou_07}, which is defined through a real-valued kernel function
\begin{align*}
    K_\gamma(u,v)=\langle  \xi(u),\xi(v)\rangle= e^{-\frac{\gamma^2 \Vert u-v \Vert^2}{2}},
\end{align*}
where  $\xi(\cdot):\mathcal X \mapsto \mathbb H$ is the implicit feature map and $\gamma$ is a hyper-parameter.  $\mathbb H$ denotes the Reproducing Kernel Hilbert Space (RKHS) associated with the kernel. It is well-known that Gaussian kernel is shift-invariant  and positive definite. Throughout this paper, we will assume that $\mathcal X$ lies on the unit sphere, i.e., all the data points are normalized to have unit $l_2$ norm. This will save us from book-keeping the sample norms in our calculations. Note that normalizing each data vector to unit $l_2$ norm before feeding the data to classifiers is a fairly standard (or recommended) procedure in practice. In this case, denoting the correlation coefficient $\rho=\cos(u,v)=u^Tv$, the Gaussian kernel can be formulated as
\begin{align}
    K_\gamma(u,v)=e^{-\frac{\gamma^2(2-2\rho)}{2}}=e^{-\gamma^2(1-\rho)}. \label{def:Gaussian kernel}
\end{align}
In the rest of the paper, we omit the subscript ``$\gamma$'' and only use $K$ to denote the kernel.

There are two major general issues with large-scale nonlinear kernel learning (not limited to the Gaussian kernel). Firstly, storing/materializing a kernel matrix for a dataset of $n$ samples would need $n^2$ entries, which may not be realistic even just for medium datasets (e.g., $n=10^6$). To avoid this problem, the entries of the kernel matrix are computed on the fly from the original dataset. This however will increase the computation time, plus storing the original high-dimensional dataset for on-demand distance computations can also be costly. Secondly, the training procedure for nonlinear kernel  algorithms is also well-known to be expensive~\citep{Proc:Platt_NIPS98,Book:Bottou_07}. Therefore, it has been an active area of research to speed up kernel machines, and using various types of random projections has become popular.

\subsection{Random Projections (RP) and Random Fourier Features (RFF)}

Again, consider two data vectors $u,v\in \mathbb R^d$. Further, we assume they are normalized to have unit $l_2$ norm and we denote $\rho = \langle u, v\rangle$.  We generate a random Gaussian vector $w\in\mathbb{R}^d$ with i.i.d. entries in $N(0,1)$.
\begin{align}\notag
\mathbb E[\langle w^Tu, w^Tv\rangle]=\langle u, v\rangle = \rho.
\end{align}
This is the basic idea of using random projections to approximate inner product. See~\cite{Proc:Li_Hastie_Church_COLT06} for the theoretical analysis (such as exact variance calculations) of this approximation scheme.

We can also use random projections to approximate the (nonlinear) Gaussian kernel with an additional step. The Random Fourier Feature (RFF)~\citep{Book:Rudin_90,Proc:Rahimi_NIPS07} is defined as
\begin{align}\label{def:RFF-I}
    \textbf{\textrm{RFF:}}\quad &F(u)=\sqrt 2 \cos(\gamma w^Tu+\tau),
\end{align}
where $\tau\sim uniform(0,\ 2\pi)$, the  uniform random variable.  Some basic probability calculations reveal that
\begin{align}\notag
\mathbb E\left[F(u)F(v)\right] &= K(u,v) = e^{-\gamma^2(1-\rho)}.
\end{align}
In other words, the inner product between the RFFs of two data samples provides an unbiased estimate of the Gaussian kernel. The simulations need to be repeated for a sufficient number of times in order to obtain reliable estimates.  That is, we generate $m$ independent RFFs using i.i.d. $w_1,...,w_m$ and $\tau_1,...,\tau_m$, and approximate the kernel $K(u,v)$ by the following unbiased estimator:
\begin{align}\label{est:full RFF-I}
&\hat K(u,v)=\frac{1}{m}\sum_{i=1}^m F_{i}(u) F_{i}(v),
\end{align}
where $F_i$ denotes the RFF generated by $w_i,\tau_i$. Furthermore, \cite{Proc:Li_KDD17} showed that one can actually reduce the estimation variances by normalizing the RFFs.

In large-scale learning, using above estimator simply requires taking the inner product between the RFF vectors of $u$ and $v$. Therefore, feeding the RFFs into a linear machine will approximate training a non-linear kernel machine, known as \textit{kernel linearization}, which may significantly accelerate training and alleviate memory burden for storing the kernel~matrix. This strategy has become  popular in the literature, e.g.,~\citep{Proc:Raginsky_NIPS09,Proc:Yang_NIPS12,Proc:Affandi_NIPS13,Proc:Hern_NIPS14,Proc:Dai_NIPS14,Proc:Yen_NIPS14,Proc:Hsieh_NIPS14,Proc:Shah_NIPS15,Chwialkowski_NIPS15,Richard_NIPS15,Proc:Sutherland_UAI15,Proc:Li_KDD17,Proc:Avron_ICML17,Proc:Sun_NeurIPS18,Proc:Tompkins_aaai18,Proc:Li_ECAI20}.

\subsection{Quantized Random Projections (QRP)}

One can further compress the projected data by quantization, into discrete integer values, or even binary values in the extreme case. The so-called quantized random projection (QRP) has found useful in many problems, e.g., theory, similarity search, quantized compressed sensing, classification and regression~\citep{Article:Goemans_JACM95,Proc:Charikar_STOC02,Proc:Datar_SCG04,Article:zymnis_SPL10,Article:Jacques_JIT13,Proc:Leng_SIGIR14,Proc:Li_ICML14,Proc:Li_NIPS17,Article:Slawski_IT18,Proc:Li_NeurIPS19_asymmetric,Proc:Li_NIPS19}.  The motivation is straightforward. If one can represent each RP (or RFF) using (e.g.,) 4 bits and still achieve similar accuracy as using the full-precision (e.g., 32 or 64 bits), it is then a substantial saving in storage space. Typically, savings in storage can directly translate into savings in data transmissions and subsequent computations. In addition to space (computation) savings, there is  another motivation for QRP. That is, quantization also provides the capability of indexing due to the integer nature of quantized data, which can be used to build hash tables for approximate near neighbor search~\citep{Proc:Indyk_Motwani_STOC98}.

The simplest quantization scheme is the 1-bit (sign) random projections, including sign Gaussian random projections~\citep{Article:Goemans_JACM95,Proc:Charikar_STOC02} and sign Cauchy random projections~\citep{Proc:Li_NIPS13} (for approximating the $\chi^2$ kernel). Basically, one only keeps the signs of projected data. Even though the 1-bit schemes appear to be overly crude and simplistic, in some cases 1-bit random projections can achieve better performance than full-precision RPs in similarity search and nearest neighbor classification tasks. Nevertheless, in general, one would need more than just 1-bit in order to achieve sufficient accuracy. For example,~\cite{Proc:Li_NIPS17,Article:Slawski_IT18,Proc:Li_NeurIPS19_asymmetric} apply the (multi-bit) Lloyd-Max (LM) quantization~\citep{Article:Max60,Article:Lloyd_JIT82} on the projected data.

\subsection{Summary of Contributions}

Since each Lloyd-Max (LM) quantizer is associated with a specific random signal distribution, at the first glance, designing LM quantizers for the random Fourier features and  the Gaussian kernel might appear challenging, due to the tuning parameter $\gamma$, which is a crucial component of the Gaussian kernel. Initially, one might expect that a different LM quantizer would be  needed for a different  $\gamma$ value. In this paper, our contribution begins with an interesting finding that the marginal distribution of the RFF is actually free of the parameter $\gamma$. This result greatly simplifies the design of LM quantization schemes for the RFF, because only one quantizer would be needed for all $\gamma$ values.  Once we have derived the marginal distribution of the RFF, we incorporate the idea of distortion optimal quantization theory to nonlinear random feature compression by providing a thorough study on the theoretical properties and practical performance.  Extensive simulations and machine learning experiments validate the effectiveness of the proposed LM quantization~schemes~for~the~RFF.

\section{The Probability Distributions of RFF} \label{sec:distribution}

We start the introduction to our proposed method by providing analysis on the probability distribution of RFF~(\ref{def:RFF-I}), which is key to the design of quantization schemes in Section~\ref{sec:LM-RFF scheme}. First, we introduce some notations.

\newpage

Let $u,v\in \mathcal X \subseteq \mathbb R^d$ be two normalized data points, and $w\in\mathbb R^d$ be a random vector with i.i.d. $N(0,1)$ entries. The projected data $w^Tu$ and $w^Tv$ follow a bivariate normal distribution:
\begin{align}\notag
\begin{pmatrix}
w^Tu\\
w^Tv
\end{pmatrix}
\sim N\left(0,\begin{pmatrix}
1 & \rho \\
\rho & 1
\end{pmatrix}\right),\hspace{0.2in} \text{where } \rho = u^Tv, \ \|u\| = \|v\| = 1.
\end{align}
Throughout the paper, we will use the following two definitions for $\phi_\sigma(t)$ and $\Phi(t)$:
\begin{align}\notag
&\phi_\sigma(t) = \frac{1}{\sqrt{2\pi}\sigma}e^{-\frac{t^2}{2\sigma^2}},\\\notag
&\Phi(t) = \int_{-\infty}^t \frac{1}{\sqrt{2\pi}}e^{-z^2/2} dz = \int_{-\infty}^t \phi_1(z)dz.
\end{align}
That is, $\Phi(t)$ is the cumulative distribution function (cdf) of the standard normal $N(0,1)$ and $\phi_\sigma(t)$ is the probability density function (pdf) of $N(0,\sigma^2)$.

\subsection{Marginal Distributions}

We first consider the marginal distributions of the RFF, which serve as the foundation of our proposed quantization schemes.  The following Lemma is a  result of the convolution of normal and uniform distributions.
\begin{lemma} \label{lemma1}
Suppose $X\sim N(0,1)$ and $\tau\sim uniform(0,2\pi)$ are independent, $\gamma>0$. Then
\begin{align*}
    \gamma X+\tau&\sim \frac{1}{2\pi}\left[ \Phi(\frac{2\pi-y}{\gamma})-\Phi(-\frac{y}{\gamma}) \right].
\end{align*}
\end{lemma}
In the following, we formally give the distribution of the RFF.
\begin{theorem}\label{theo:density of RFF1}
Let $X\sim N(0,1)$, $\tau\sim uniform(0,2\pi)$ be independent. Denote $Z=\cos(\gamma X+\tau)$, and $Z_2=\cos^2(\gamma X+\tau)$. We have the probability density functions
\begin{align}
    f_Z(z)&= \frac{1}{\pi\sqrt{1-z^2}},\quad z\in[-1,1],  \label{eqn:density-RFF1} \\
    f_{Z_2}(z)&=\frac{1}{\pi\sqrt{z-z^2}},\quad z\in[0,1].  \label{eqn:density-RFF1 square}
\end{align}
for any $\gamma>0$. In particular, $Z\overset{d}{\sim} \cos(\tau)$ in distribution.
\end{theorem}

\begin{figure}[h!]
    \centering
    \includegraphics[width=2.5in]{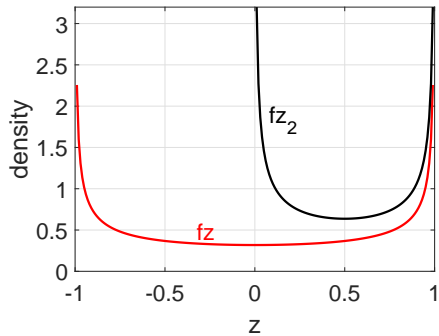}
    \caption{Probability density functions $f_Z$ and $f_{Z_2}$ derived in Theorem~\ref{theo:density of RFF1}. }
    \label{fig:density}
\end{figure}

The density plots can be found in Figure~\ref{fig:density}. Theorem~\ref{theo:density of RFF1} says that for any kernel parameter $\gamma$, the (unscaled) RFF follows the same distribution as the cosine of the uniform noise itself. Intuitively, this is because  cosine  is a $2\pi$-periodic function and normal distribution is symmetric. As will be introduced in Section~\ref{sec:LM-RFF scheme}, each Lloyd-Max (LM) quantizer is associated with a signal distribution. We will characterize two LM-type quantizers w.r.t. density (\ref{eqn:density-RFF1}) and (\ref{eqn:density-RFF1 square}), respectively. This interesting result is favorable for our purpose as it implies we only need to construct one LM quantizer, which covers all the Gaussian kernels with different $\gamma$ value. Thus, the design of LM quantizer for RFF is convenient.

\begin{remark} \label{remark 2}
In Theorem~\ref{theo:density of RFF1} we consider $X\sim N(0,1)$ because we assume data samples are normalized for conciseness. It is easy to see that this result also holds without data normalization (i.e., $X$ is Gaussian with arbitrary variance) since we can offset the variance of $X$ by altering $\gamma$. Therefore, Theorem~\ref{theo:density of RFF1} is a universal result implying that the   LM quantizer also works without data normalization.
\end{remark}

\subsection{Joint Distribution}

In the sequel, we analyze the joint distribution of RFFs of two data samples with correlation $\rho$'s. The joint distribution will play an important role in later theoretical analysis. The following Lemma~\ref{lemma2} leads to the desired result presented in Theorem~\ref{theo:joint_RFF1}.

\begin{lemma} \label{lemma2}
Denote $z_x=\gamma X+\tau$, $z_y=\gamma Y+\tau$ with $(X,Y)\sim N\big(0,\begin{pmatrix}
1 & \rho \\
\rho & 1
\end{pmatrix}\big)$, $\tau\sim uniform(0,2\pi)$. We have the joint distribution
\begin{align*}
    f(z_x,z_y)=&\frac{1}{2\pi}\phi_{\sqrt{2(1-\rho)}\gamma}(z_x-z_y) \Big[\Phi(\frac{4\pi-(z_x+z_y)}{\gamma\sqrt{2(1+\rho)}})-\Phi(-\frac{z_x+z_y}{\gamma\sqrt{2(1+\rho)}}) \Big].
\end{align*}
\end{lemma}

\begin{theorem} \label{theo:joint_RFF1}
Denote $z_x=\cos(\gamma X+\tau)$, $z_y=\cos(\gamma Y+\tau)$ where $X,Y,\tau$ are the same as Lemma~\ref{lemma2}. Then we have the joint density function for $(z_x,z_y)\in [-1,1]^2$,
\begin{align*}
    f(z_x,z_y)&=\frac{1}{\pi\sqrt{1-z_x^2}\sqrt{1-z_y^2}}\sum_{k=-\infty}^\infty \Big[ \phi_{\sqrt{2(1-\rho)}\gamma}(a_x^*-a_y^*+2k\pi)+\phi_{\sqrt{2(1-\rho)}\gamma}(a_x^*+a_y^*+2k\pi) \Big],
\end{align*}
where $a_x^*=\cos^{-1}(z_x),a_y^*=\cos^{-1}(z_y)$. In addition, $\big(\sin(\gamma X+\tau),\ \sin(\gamma Y+\tau)\big)$ follows the same distribution.
\end{theorem}

In Figure~\ref{fig:joint density}, we plot the joint density at several $\gamma$ values. We conclude several properties of the joint distribution. Firstly, it is obvious that $z_x$ and $z_y$ are exchangeable, i.e., $f(Z_x,Z_y)=f(Z_y,Z_x)$. Secondly, it is symmetric which means $f(Z_x,Z_y)=f(-Z_x,-Z_y)$. Moreover, we have the following important result, which is helpful for our analysis on the monotonicity  and variance of quantized kernel estimators in Section~\ref{sec:theory}.

\begin{proposition} \label{prop:joint inequality}
Let the density function $f$ be defined as Theorem~\ref{theo:joint_RFF1}. If $\sqrt{2(1-\rho)}\gamma\leq \pi$, then $f(z_x,z_y)>f(z_x,-z_y)$ for $\forall (z_x,z_y)\in (0,1]^2$ or $(z_x,z_y)\in [-1,0)^2$.
\end{proposition}

In Proposition~\ref{prop:joint inequality}, the quantity $\sqrt{2(1-\rho)}\gamma$ will be reduced if we either increase $\rho$ or decrease $\gamma$. In Figure~\ref{fig:joint density}, we see how this term characterizes the joint density of RFF. In particular, smaller $\sqrt{2(1-\rho)}\gamma$ reinforces the dependency between $z_x$ and $z_y$. The density around $(1,1)$ and $(-1,-1)$ reaches the highest when $\rho=0.9$ and $\gamma=1$. As $\rho$ decreases or $\gamma$ increases, the density is ``flattened''.

\begin{figure}[b!]
    \begin{center}
        \mbox{
        \includegraphics[width=2.2in]{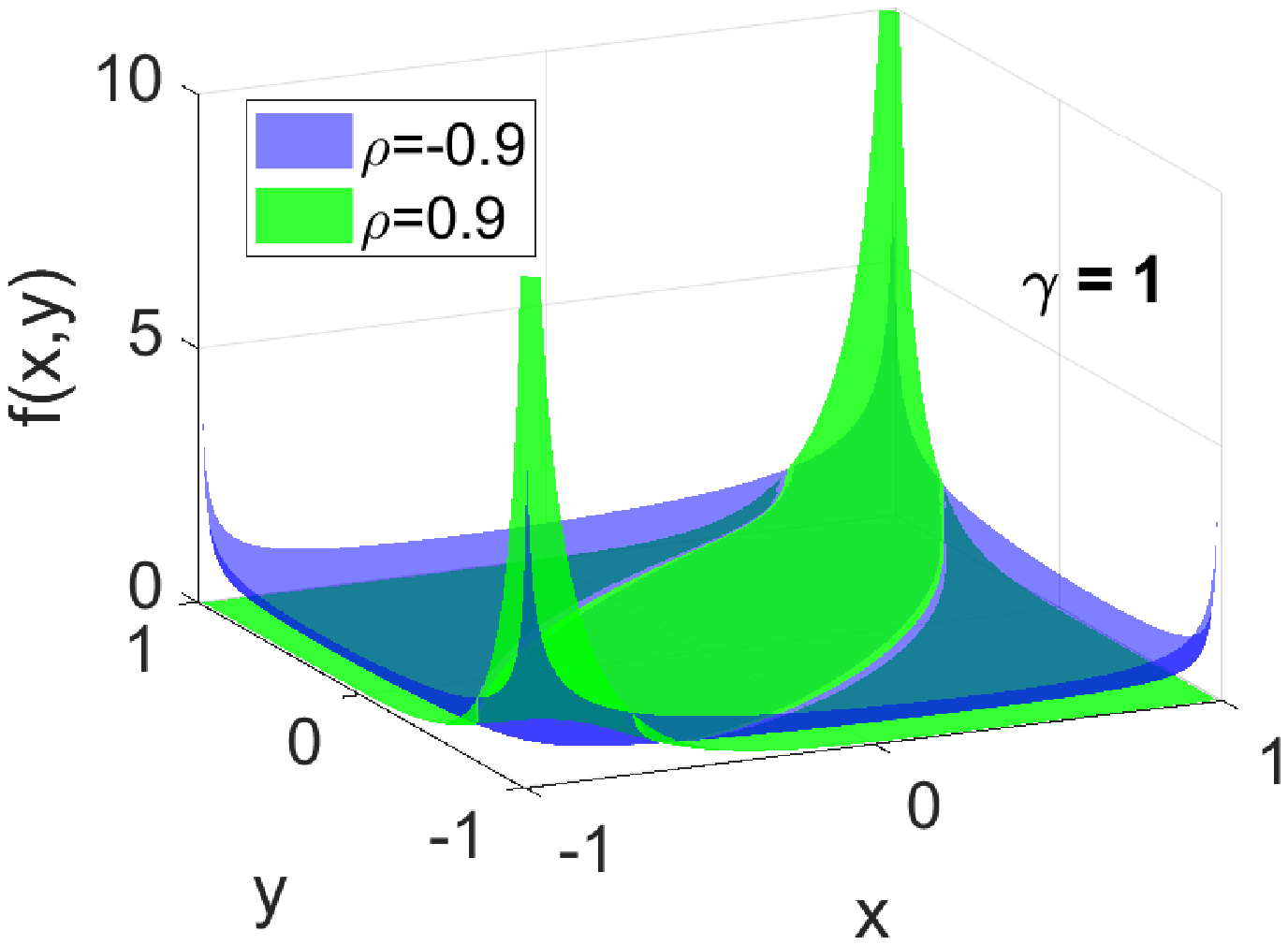}
        \includegraphics[width=2.2in]{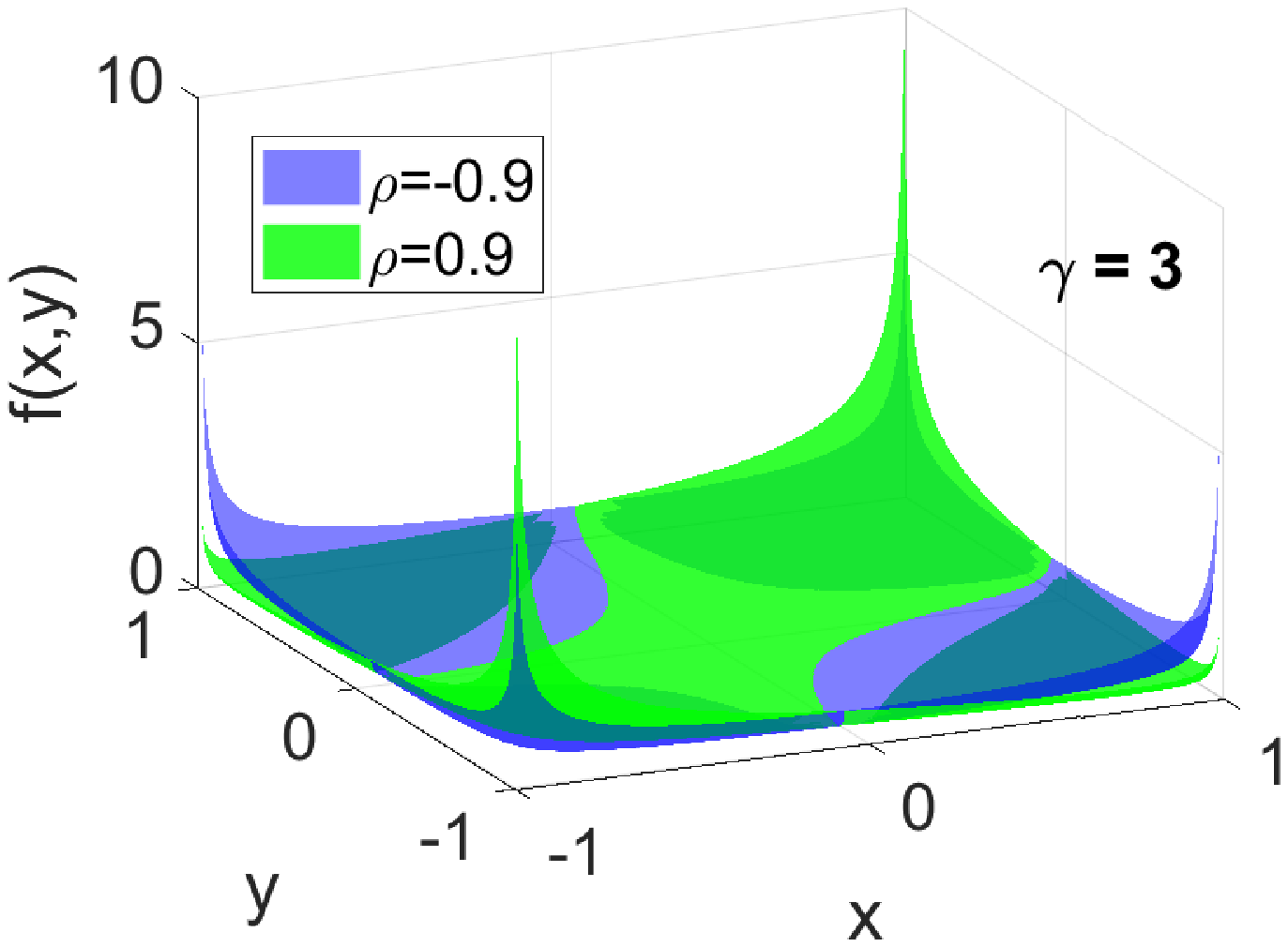}
        \includegraphics[width=2.2in]{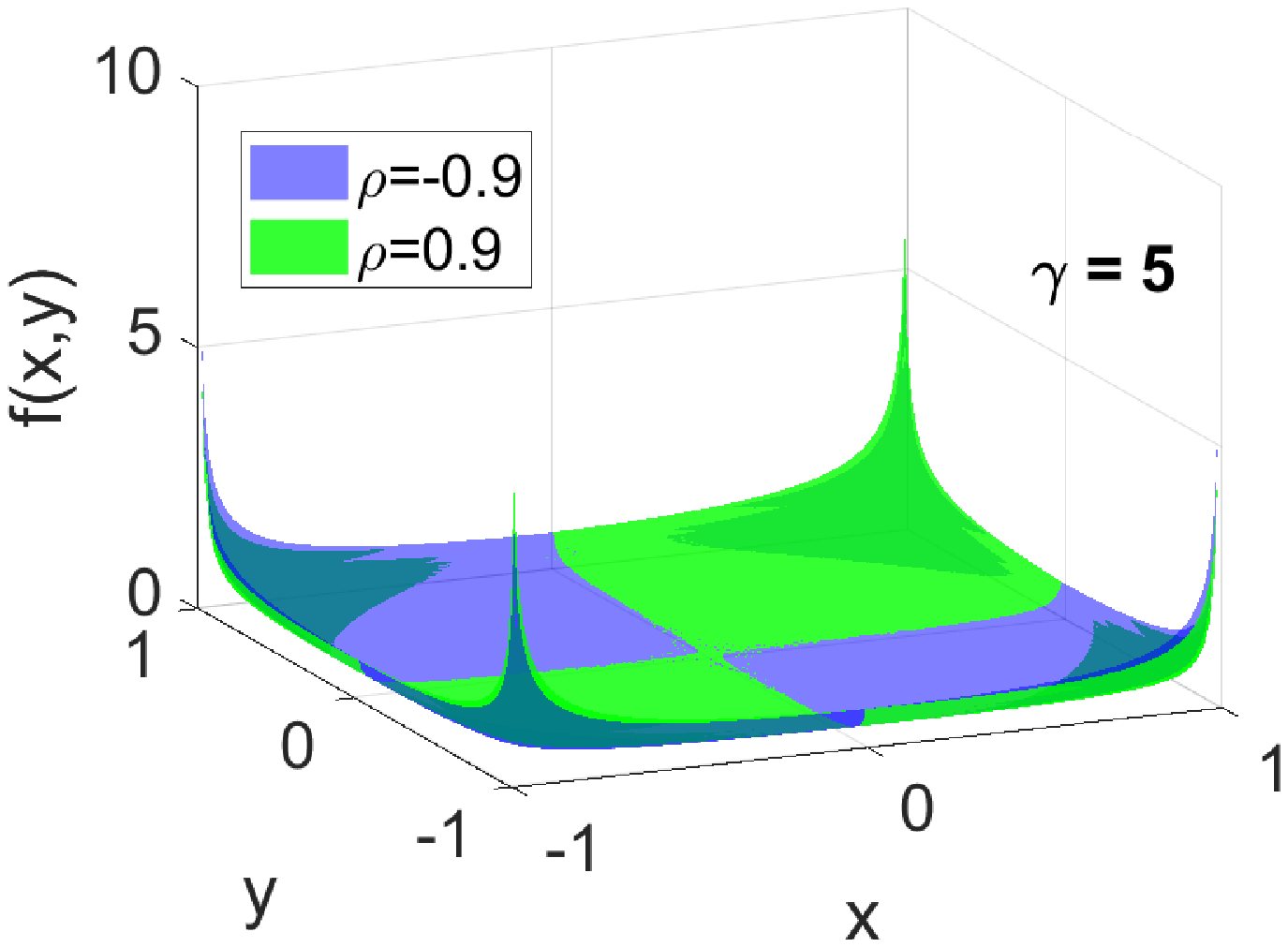}
        }
    \end{center}
    \vspace{-0.2in}
	\caption{The joint density of RFF (Theorem~\ref{theo:joint_RFF1}) with example $\rho=-0.9, 0.9$, and $\gamma=1,3,5$.}
	\label{fig:joint density}\vspace{-0.2in}
\end{figure}

\section{Quantization Schemes for RFF}  \label{sec:LM-RFF scheme}

Quantization is, to a good extent, an ancient topic in information theory and signal processing~\citep{Book:Widrow_2008}. On the other hand, many interesting research works appear in the literature even very recently, for achieving better efficiency in data storage, data transmission,  computation, and energy consumption. Quantization for random projections has been heavily studied. In this paper, we focus on developing quantization schemes for random Fourier features, in particular, based on the Lloyd-Max~(LM)~framework.

This article \url{https://www.eetimes.com/an-introduction-to-different-rounding-algorithms/}, published in {\em EETimes} 2006, provides a good summary of common quantization (rounding) schemes: ``rounding in decimal'', ``round-toward-nearest'',``round-half-up'', ``round-half-down'',``round-half-even'', ``Round-half-odd'', ``round-alternate'', ``round-random'', ``round-ceiling'', ``round-floor'', ``round-toward-zero'', ``round-away-from-zero'', ``round-up'', ``round-down'', ``truncation''. For random projections, quantization methods such as~\cite{Proc:Datar_SCG04,Proc:Li_ICML14,Proc:Li_AISTATS17} belong to those categories.  One should expect that, as long as we use a sufficient number (such as 32 or  64) of bits, any reasonable quantization scheme should achieve~a~good~accuracy.

In this paper, we mainly focus on quantization for RFF with a small number (such as 1, 2, 3, or 4) bits.  We consider general multi-bit quantizers, with 1-bit quantization as a special case, for the cosine feature in RFF bounded in $[-1,1]$. Here, ``$b$-bit'' means the quantizer has $2^b$ levels. For simplicity, we will denote $z=\cos(\gamma w^T u+\tau)$ as the item to be quantized. In particular, we focus on two algorithms: ``round-random'' (which we refer to as the ``stochastic quantization (StocQ)'') and ``Lloyd-Max (LM) quantization''.

\subsection{Stochastic Quantization (StocQ)}

A $b$-bit StocQ quantizer, also known as stochastic rounding, splits $[-1,1]$ into ($2^b-1$) intervals, with consecutive borders $-1=t_0<...<t_{2^b-1}=1$. Let $[t_i, t_{i+1}]$ be the interval containing $z$. StocQ pushes $z$ to either $t_i$ or $t_{i+1}$, depending on its distance to the borders. Concretely, denoting $\triangle_i=t_{i+1}-t_i$,
\begin{smallequation}
    P(Q(z)=t_i)=\frac{t_{i+1}-z}{\triangle_i},\ \ P(Q(z)=t_{i+1})=\frac{z-t_i}{\triangle_i}. \label{def:StocQ}
\end{smallequation}
It is not difficult to see that by the sampling procedure, conditional on the full-precision RFF $z$, the quantized value $Q(z)$ by StocQ is unbiased of $z$. On the other hand, also due to the Bernoulli sampling approach, StocQ has the extra variance especially when the number of bits is small (e.g., 1-bit or 2-bit quantization). Note that in~\cite{Proc:Zhang_AISTATS19}, the authors applied StocQ with uniform borders in large-scale machine learning tasks with RFF. Here we consider a more general approach where the borders are not necessarily uniform.

\subsection{Lloyd-Max (LM) Quantization}

In information theory and signal processing, the Lloyd-Max (LM)~\citep{Article:Max60,Article:Lloyd_JIT82} quantization scheme has a long history that also leads to many well-known methods (e.g., the $k$-means clustering). Interestingly, it has not been adopted to RFF in the prior literature. In contrast to StocQ, the proposed LM-RFF constructs a fixed quantizer $Q_b$. We call $[\mu_1,...,\mu_{2^b}]\in\mathbb C$ the reconstruction levels, with $\mathbb C$ the ``codebook'' of $Q_b$. Also, $-1=t_0<...<t_{2^b}=1$ represent the borders of the quantizer. Then, LM-RFF quantizer $Q_b$ defines a mapping: $[-1,1]\mapsto \mathbb C$, where $Q_b(x)=\mu_i$ if $t_{i-1}< x\leq t_i$. By choosing the error function as the squared difference, given an underlying signal distribution $f(z)$ with support $\mathcal S$, the LM quantizer minimizes the \textit{distortion} defined as $$D_Q=\int_{\mathcal S}  (z-Q(z))^2f(z)dz,$$ aiming to keep most amount of information of the original signal. For the signal distribution, we consider two variants. First, it is natural to set the target distribution as the distribution of RFF itself~(\ref{eqn:density-RFF1}). Consequently, the first LM quantizer is subject to the distortion:
\begin{smallequation}
    \textbf{\textrm{LM-RFF:}}\quad D_1\triangleq \int_{[-1,1]} \big( z-Q(z)\big)^2 \frac{1}{\pi\sqrt{1-z^2}}dz. \label{def:distortion}
\end{smallequation}
Conceptually, optimizing (\ref{def:distortion}) minimizes the average difference between RFF $z$ and $Q(z)$. Alternatively, we may also choose to address more on the high similarity region (which is more important sometimes). Denote $z_u$ and $z_v$ as the RFFs of $u$ and $v$. As $\rho\rightarrow 1$, we have $z_x\cdot z_y\rightarrow z_x^2$, with density given by (\ref{eqn:density-RFF1 square}). Thus, our second variant, LM$^2$-RFF, is designed to approximate $z^2$ by $Q(z)^2$ with distortion
\begin{smallequation}
    \textbf{\textrm{LM$^2$-RFF:}}\quad D_2\triangleq \int_{[-1,1]} \big( z^2-Q(z)^2\big)^2 \frac{1}{\pi\sqrt{1-z^2}}dz. \label{def:distortion2}
\end{smallequation}
Minimizing (\ref{def:distortion}) and (\ref{def:distortion2}) leads to our proposed two LM-type quantizers for RFF compression in this paper.~\footnote{Note that, in Eq.~(\ref{def:distortion2}),   $\int_0^1 \frac{1}{\pi\sqrt{1-z^2}}dz = \int_0^1 \frac{1}{\pi\sqrt{1-z}}dz^{1/2} =\frac{1}{2}\int_0^1 \frac{1}{\pi\sqrt{z-z^2}}dz = \frac{1}{2}\int_0^1 f_{Z_2}(z) dz$, where $f_{Z_2}$ is the density (\ref{eqn:density-RFF1 square}).}

\subsection{Optimization for LM Quantizers}

To solve the introduced optimization problems, we exploit classical Lloyd's algorithm, which alternatively updates two parameters until convergence.  By e.g.,~\cite{Article:Wu_IT92}, the algorithm converges to the globally optimal solution since the squared loss is convex and symmetric. The algorithm terminates when the total absolute change in borders and reconstruction levels in two consecutive iterations is smaller than a given threshold (e.g., $10^{-5}$). We provide the concrete steps for LM-RFF and LM$^2$-RFF in Algorithm~\ref{alg:LM-RFF} and Algorithm~\ref{alg:LM-RFF2}, respectively. For LM-RFF, we see that the procedure is standard (exactly the same as above derivation). Denote $z$ as the unscaled RFF, $z=\cos(\gamma X+\tau)$, $X\sim N(0,1)$ and $\tau\sim uniform(0,2\pi)$. For LM$^2$-RFF, recall that our objective is to minimize (by change of random variable)
\begin{align*}
    D_Q=\int_{-1}^1 (Q(z)^2-z^2)^2 f_Z(z) dz=\int_0^1 (\tilde Q(s)-s)^2 f_{Z_2}(s) ds,
\end{align*}
where $f_Z(z)$ is the density in Theorem~\ref{theo:density of RFF1} (\ref{eqn:density-RFF1}) of RFF (supported on $[-1,1]$), and $f_{Z_2}(s)$ is the density in Theorem~\ref{theo:density of RFF1} (\ref{eqn:density-RFF1 square}) of squared RFF $z^2$ (on $[0,1]$). Here we let $\tilde Q(s)=Q(z)^2$ on $[0,1]$. This provides us a way to first optimize the positive part according to the density of $z^2$, then get the complete LM$^2$-RFF quantizer by taking the square root and symmetry. In particular, the \texttt{SymmetricExpand} function in Algorithm~\ref{alg:LM-RFF2} applies reverted expansion with symmetry. For example, the output of $\{\bm t=(0,0.5,1),\bm\mu=(0.2,0.7)\}$ would be $\{\bm t=(-1,-0.5,0,0.5,1),\bm\mu=(-0.7,-0.2,0.2,0.7)\}$. For practitioners to use our quantizers forthrightly, in Table~\ref{tab:LM quantizer} and \ref{tab:LM square quantizer} we summarize precisely the derived LM-RFF and LM$^2$-RFF quantizers for $b=1,2,3,4$.

\begin{algorithm}{
		\textbf{Input:} Density $f_Z(z)$ (Theorem~\ref{theo:density of RFF1}, (\ref{eqn:density-RFF1})), Number of bits $b$

		\textbf{Output:} LM-RFF quantizer $\bm t=[t_0,...,t_{2^b}]$, $\bm\mu=[\mu_1,...,\mu_{2^b}]$
		
		\vspace{0.05in}
		Fix $t_0=-1,t_{2^b}=1$
		
		While \textit{true}
		
		\hspace{0.2in}For $i=1$ to $2^b$
		
		\hspace{0.4in}Update $\mu_i$ by $\mu_i=\frac{\int_{t_{i-1}}^{t_i} z f_Z(z) dz}{\int_{t_{i-1}}^{t_i} f_Z(z) dz}$
		
		\hspace{0.2in}End For
		
		\hspace{0.2in}For $i=1$ to $2^b-1$
		
		\hspace{0.4in}Update $t_i$ by $t_i=\frac{\mu_i+\mu_{i+1}}{2}$
		
		\hspace{0.2in}End For
		
		Until Convergence

	}\caption{Construction of LM-RFF quantizer}
	\label{alg:LM-RFF}
\end{algorithm}

\begin{algorithm}{
		\textbf{Input:} Density $f_{Z_2}(z)$ (Theorem~\ref{theo:density of RFF1}, (\ref{eqn:density-RFF1 square})), Number of bits $b$

		\textbf{Output:} LM-RFF quantizer $\bm t=[t_0,...,t_{2^b}]$, $\bm\mu=[\mu_1,...,\mu_{2^b}]$
		
		\vspace{0.05in}
		Fix $t_{2^{b-1}}=0,t_{2^b}=1$
		
		While \textit{true}
		
		\hspace{0.2in}For $i=2^{b-1}+1$ to $2^b$
		
		\hspace{0.4in}Update $\mu_i$ by $\mu_i=\frac{\int_{t_{i-1}}^{t_i} z f_{Z_2}(z) dz}{\int_{t_{i-1}}^{t_i} f_{Z_2}(z) dz}$
		
		\hspace{0.2in}End For
		
		\hspace{0.2in}For $i=2^{b-1}+1$ to $2^b-1$
		
		\hspace{0.4in}Update $t_i$ by $t_i=\frac{\mu_i+\mu_{i+1}}{2}$
		
		\hspace{0.2in}End For
		
		Until Convergence
		
		$\{\bm t,\bm\mu\}=\texttt{SymmetricExpand}(\sqrt{\bm t},\sqrt{\bm\mu})$

	}\caption{Construction of LM$^2$-RFF quantizer}
	\label{alg:LM-RFF2}
\end{algorithm}

\begin{table}[h]
\centering
\begin{tabular}{V{3}| c|c|c |V{3}}\Xhline{3\arrayrulewidth}
    \textbf{$b$}  &  \textbf{Borders} &  \textbf{Reconstruction Levels}  \tabularnewline \hline
$1$ &    $\{0,1\}$     &     $\{0.637\}$                  \\ \hline
$2$ &   $\{0,0.576,1\}$      &   $\{0.297,0.854 \}$                    \\ \hline
$3$ &     $\{0,0.286,0.563,0.819,1 \}$    &    $\{0.144,0.428,0.699,0.939 \}$                   \\ \hline
$4$ &    $\{0,0.142,0.283,0.421,0.557,0.687,0.811,0.922,1 \}$     &  $\{0.071,0.213,0.353,0.49,0.624,0.751,0.87,0.974 \}$
\\ \Xhline{3\arrayrulewidth}
\end{tabular}
\caption{Constructed borders and reconstruction levels of LM-RFF quantizers, $b=1,2,3,4$, keeping three decimal places. Since the quantizers are symmetric about 0, we only present the positive part for conciseness.}
\label{tab:LM quantizer}
\end{table}

\begin{table}[h]
\centering
\begin{tabular}{V{3}| c|c|c |V{3}}\Xhline{3\arrayrulewidth}
    \textbf{$b$}  &  \textbf{Borders} &  \textbf{Reconstruction Levels}  \tabularnewline \hline
$1$ &    $\{0,1\}$     &     $\{0.707\}$                  \\ \hline
$2$ &   $\{0,0.707,1\}$      &   $\{0.426,0.905 \}$                    \\ \hline
$3$ &     $\{0,0.461,0.707,0.888,1 \}$    &    $\{0.27,0.593,0.805,0.963 \}$                   \\ \hline
$4$ &    $\{0,0.301,0.467,0.596,0.707,0.802,0.884,0.954,1 \}$     &  $\{0.175,0.39,0.535,0.654,0.756,0.845,0.92,0.985 \}$
\\ \Xhline{3\arrayrulewidth}
\end{tabular}
\caption{Constructed borders and reconstruction levels of LM$^2$-RFF quantizers, $b=1,2,3,4$, keeping three decimal places. Since the quantizers are symmetric about 0, we only present the positive part for conciseness.}
\label{tab:LM square quantizer}
\end{table}

In Figure~\ref{fig:quantizer}, we plot the $2$-bit LM-RFF and LM$^2$-RFF quantizer as an example, along with the distortions of LM quantizers and uniform stochastic quantization (StocQ), with various number of bits. We see that both LM methods give non-uniform quantization borders and codes. LM$^2$-RFF ``expands'' more towards two ends since it tries to approximate $z^2$. From the distortion plots, we validate that LM-RFF provides smallest $D_1$ and LM$^2$-RFF gives smallest $D_2$.

As a final remark before ending this subsection, we note that LM quantization is more convenient and faster compared with StocQ in practical implementation, for quantizing the full-precision RFFs. While LM is a fixed quantization approach, StocQ requires generating an extra random number for each sketch and each data point. For large datasets, producing these additional random numbers might be rather slow in practice.

\begin{figure}[h]
        \mbox{
        \includegraphics[width=2.2in]{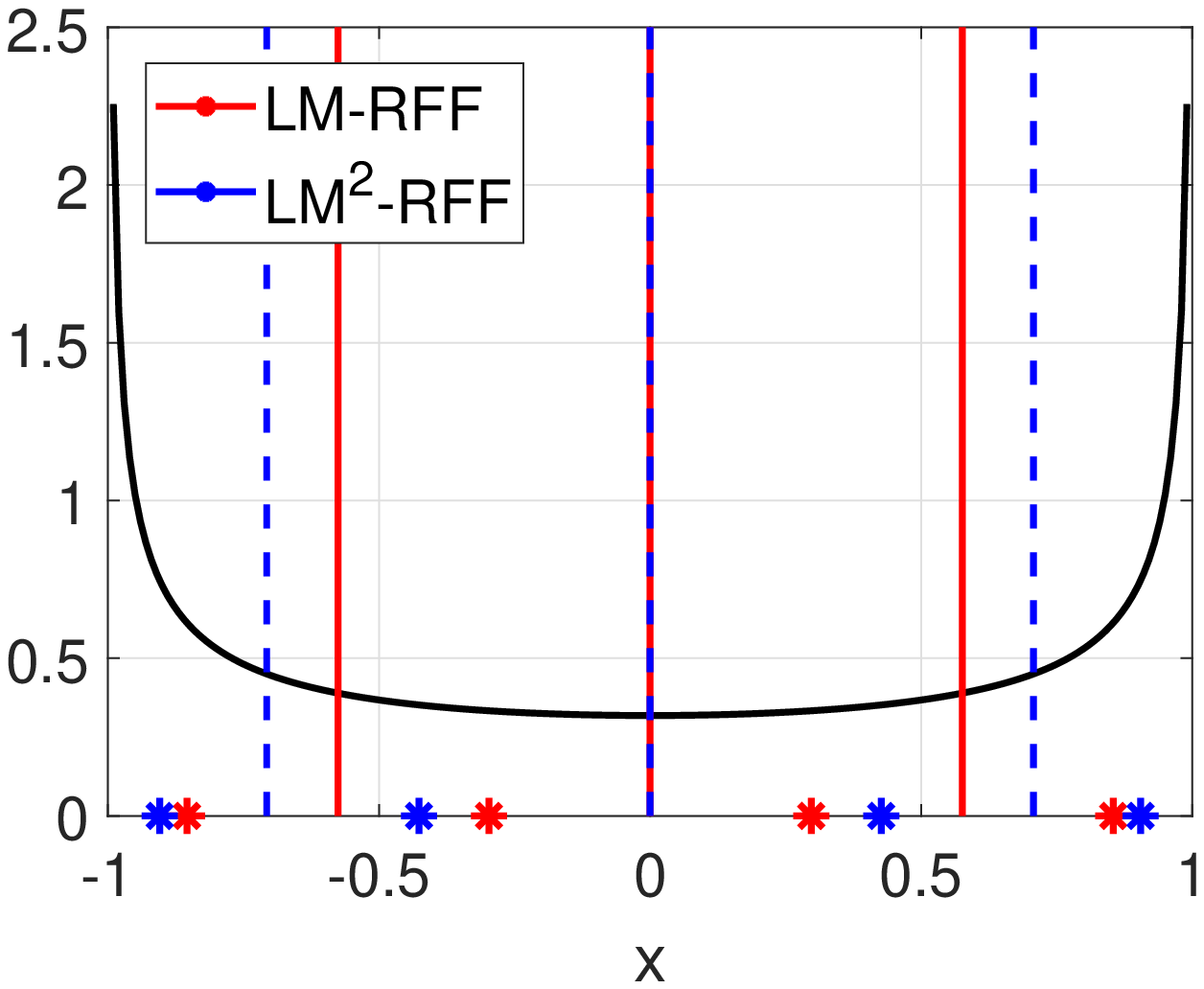}
        \includegraphics[width=2.2in]{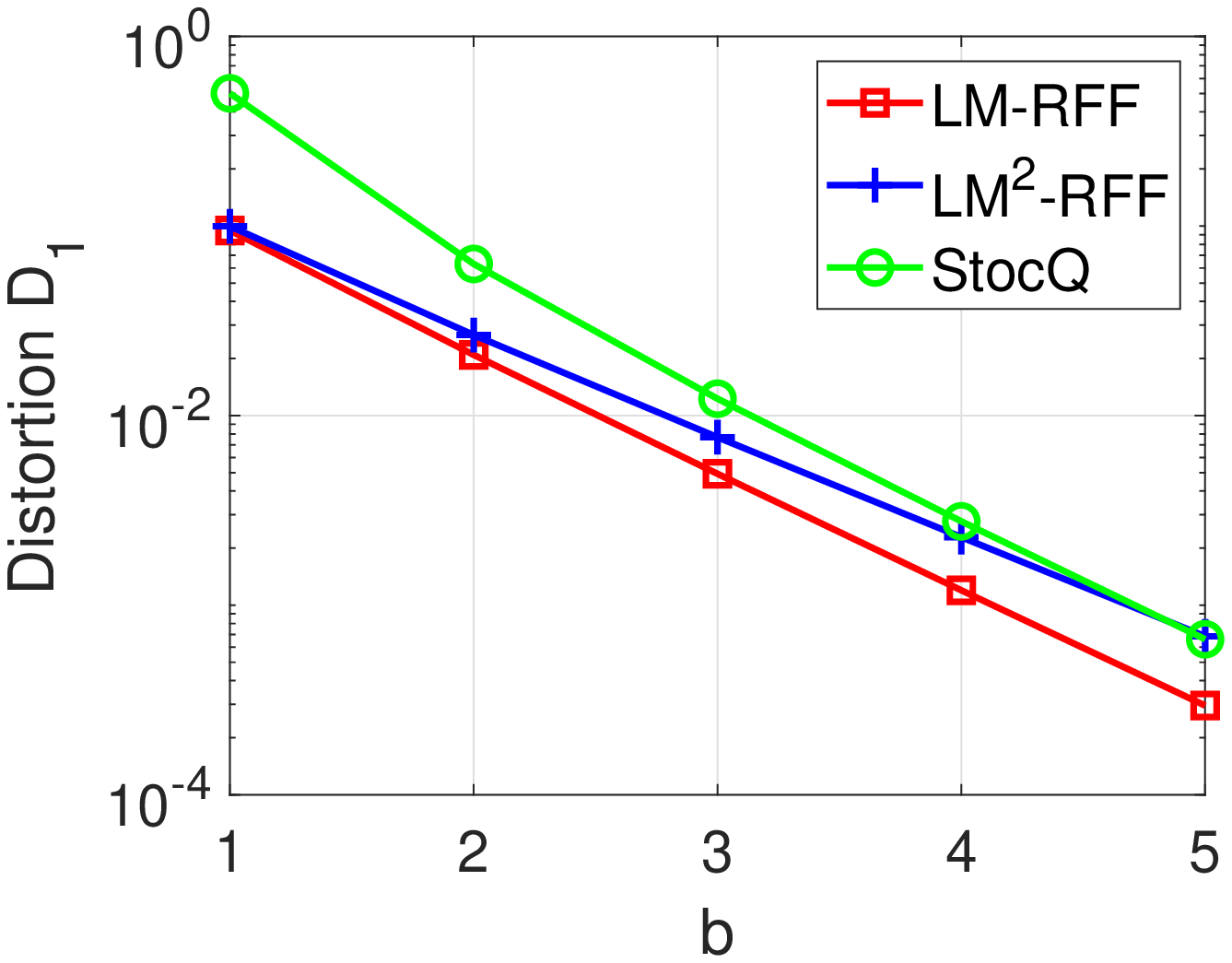}
        \includegraphics[width=2.2in]{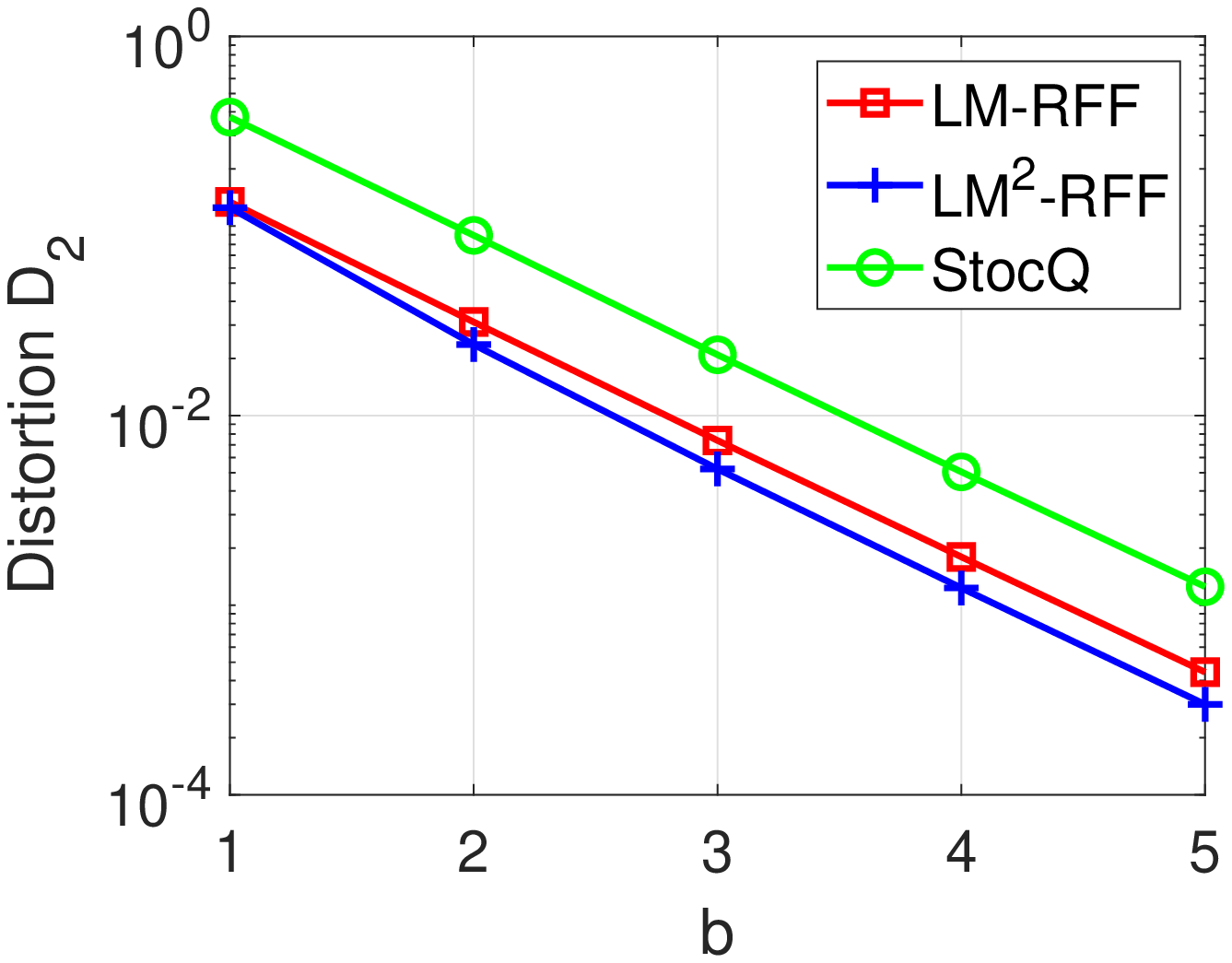}
        }
    \vspace{-0.15in}
	\caption{\textbf{Left}: LM-RFF and LM$^2$-RFF quantizers, $b=2$. Black curve is the RFF marginal density. \textbf{Right two}: Distortion $D_1$ and $D_2$ of LM-RFF, LM$^2$-RFF and StocQ (with uniform borders).}
	\label{fig:quantizer}
\end{figure}

\newpage\clearpage

\subsection{Quantized Kernel Estimators} \label{sec:def-estimators}

In the following, we define the kernel estimators built upon quantized RFFs, which lie at the core of understanding the behavior of different compression schemes. For a fixed $\gamma$ (Gaussian kernel parameter), with a general RFF quantizer $\tilde Q$ (i.e., proper kind introduced above), a simple quantized kernel estimator using $m$ random features can be constructed as
\begin{align}
&\hat K_{\tilde Q}(u,v)=\frac{2}{m}\sum_{i=1}^m \tilde Q(z_{u,i}) \tilde Q(z_{v,i}), \label{est:Q estimator-I}
\end{align}
with $z_{u,i}=\cos(w_i^Tu+\tau_i)$ and $z_{v,i}=\cos(w_i^Tv+\tau_i)$ the $i$-th unscaled RFF of $u$ and $v$, respectively. Moreover, for the proposed LM quantizers, we consider \textit{normalized estimator},
\begin{align}
    \hat K_{n,Q}(u,v)=\frac{\sum_{i=1}^m Q(z_{u,i}) Q(z_{v,i})}{\sqrt{\sum_{1}^m Q(z_{u,i})^2}\sqrt{\sum_{1}^m Q(z_{v,i})^2}}. \label{est:normalize}
\end{align}
This estimator can also be conveniently used, as we only need to normalize the quantized RFFs (per data point) before learning. We will use $\hat K_{Q,(2)}$ and $\hat K_{n,Q,(2)}$ to denote the corresponding estimators using LM$^2$-RFF quantization. We will analyze and compare the estimators using different quantization methods, theoretically and practically, in the remaining sections of the paper.

\section{Theoretical Analysis} \label{sec:theory}

In this section, we first analyze the mean, variance, and monotonicity property of the quantized kernel estimators, then discuss kernel matrix approximation property based on some new evaluation metrics that can well align with the generalization performance. The proofs are deferred to Appendix~\ref{append sec:proof}.

\subsection{StocQ Estimator} \label{sec:StocQ}

We start this section by analyzing the stochastic rounding method for RFF. In~\cite{Proc:Zhang_AISTATS19}, the exact variance of the kernel estimator is not provided. In the following, we establish the precise variance calculation in Theorem~\ref{theo:joint_RFF1}, which is in fact a more general result on any symmetric stochastic quantizer.

\vspace{0.1in}
\begin{theorem}[StocQ] \label{theo:StocQ}
Suppose a $b$-bit StocQ quantizer~(\ref{def:StocQ}) applies stochastic rounding corresponding to arbitrary bin split $-1=t_0<...<t_{2^b-1}=1$ that is symmetric about 0. Denote $S_i=t_{i-1}+t_i$ and $P_i=t_{i-1}t_i$, $i=1,...,2^b-1$. Let $f(z_u,z_v)$ be the RFF joint distribution in Theorem~\ref{theo:joint_RFF1}. Denote $\kappa_{i,j}=\int_{t_{i-1}}^{t_i}\int_{t_{j-1}}^{t_j}z_uz_vf(z_u,z_v)dz_udz_v$, and $p_{i,j}=\int_{t_{i-1}}^{t_i}\int_{t_{j-1}}^{t_j}f(z_u,z_v)dz_udz_v$. Then we have $\mathbb E[\hat K_{StocQ}]=K(u,v)$ and $Var[\hat K_{StocQ}]=\frac{V_{StocQ}}{m}$ with
\begin{align}
    V_{StocQ}&=4 \sum_{i=1}^{2^b-1}\sum_{j=1}^{2^b-1} \Big[ S_iS_j\kappa_{i,j}+P_iP_j p_{i,j} \Big]-K(u,v)^2, \nonumber
\end{align}
which is always greater than $Var[\hat K]$ defined in (\ref{est:full RFF-I}).
\end{theorem}

The important take-away messages are: 1) the StocQ kernel estimator is unbiased of the Gaussian kernel; 2) the variance is always larger than full-precision RFF estimate. Further, we have the following result for 1-bit StocQ, which is a straightforward consequence of Theorem~\ref{theo:StocQ}

\vspace{0.1in}

\begin{corollary}
With 1-bit, $Var[\hat K_{StocQ}]=4-K(u,v)^2$.
\end{corollary}

\subsection{LM Estimators} \label{sec:mean}

In this subsection, we study the moments of the proposed LM kernel estimators. Since our following results generalize to both LM-RFF and LM$^2$-RFF, we will unify the notation as $Q$ to denote a LM-type quantizer. First, we have the following formulation of the mean estimate of LM quantized estimator (\ref{est:Q estimator-I}) based on Chebyshev functional approximation.

\begin{theorem}[LM] \label{theo:mean-var}
Let $u,v$ be two normalized data samples with correlation $\rho$, and $\hat K_{Q}(u,v)$ be as (\ref{est:Q estimator-I}) with LM quantizer $Q$ (of either kind). Let $z_x=\cos(\gamma X+\tau)$, $z_y=\cos(\gamma Y+\tau)$ where $(X,Y)\sim N\big(0,\begin{pmatrix}
1 & \rho \\
\rho & 1
\end{pmatrix}\big)$, $\tau\sim uniform(0,2\pi)$, and denote $\theta_{s,t}=\mathbb E[z_x^s Q(z_x)^t]$, $\zeta_{s,t}=\mathbb E[Q(z_x)^sQ(z_y)^t]$. Further define $\alpha_i=\frac{2}{\pi}\int_{-1}^1 Q(x)T_i(x)\frac{dx}{\sqrt{1-x^2}}$, $\psi_{i,j}=\mathbb E[T_i(z_x)T_j(z_y)]$, where $T_i(x)$ is the $i$-th Chebyshev polynomial of the first kind. Then we have
\begin{align*}
    &\mathbb E[\hat K_{Q}(u,v)]=4\theta_{1,1}^2 K(u,v)+ \sum_{i=1,odd}^\infty \sum_{j=3,odd}^\infty \alpha_i\alpha_j\psi_{i,j},\\
    &Var[\hat K_{Q}(u,v)]=\frac{4}{m}(\zeta_{2,2}-\zeta_{1,1}^2).
\end{align*}
In particular, $\mathbb E[\hat K_{Q}(u,v)]=4\theta_{1,1}^2 K(u,v)$ when $\rho=0$, and $\mathbb E[\hat K_{Q}(u,v)]=2\theta_{1,1}$ when $\rho=1$.
\end{theorem}

\vspace{0.1in}

Note that, for LM-RFF quantizer, we have $4\theta_{1,1}^2=(1-2D_1)^2$ with $D_1$ defined in (\ref{def:distortion}). Since Chebyshev polynomials form an orthogonal basis of the function space on $[-1,1]$ with finite number of discontinuities, we can show that $\alpha_i=\sqrt{2\theta_{1,1}}c_i$ where $c_i$ is the cosine between $Q(x)$ and $T_i(x)$, and $\sum_{i=3,odd}^\infty \alpha_i^2=2(\theta_{1,1}-2\theta_{1,1}^2)$ which is typically very small and decreases as the quantizer has more bits. Also, we have $|\psi_{i,j}|\leq \mathbb E[T_i(z_x)^2]=1/2$. Consequently, in Theorem~\ref{theo:mean-var} the last term approximates zero in most cases. This translates into the following observation.

\vspace{0.1in}
\begin{observation} \label{obv1}
$\mathbb E[\hat K_{Q}(u,v)]\approx 4\theta_{1,1}^2K(u,v)$.
\end{observation}

Next, we provide an asymptotic analysis on the normalized quantized kernel estimate (\ref{est:normalize}) under LM scheme.

\begin{theorem}[Normalized estimator] \label{theo: mean-var-norm}
Under same setting as Theorem~\ref{theo:mean-var}, as $m\rightarrow \infty$,
\begin{align*}
    &\mathbb E[\hat K_{n,Q}]=\frac{\zeta_{1,1}}{\zeta_{2,0}}+\mathcal O(\frac{1}{m}),\ \ Var[\hat K_{n,Q}]=\frac{V_{n}}{m}+\mathcal O(\frac{1}{m^2}), \\
    &\hspace{0.1in} \textrm{with}\ \ V_n= \frac{\zeta_{2,2}}{\zeta_{2,0}^2}-\frac{2\zeta_{1,1}\zeta_{3,1}}{\zeta_{2,0}^3}+\frac{\zeta_{1,1}^2(\zeta_{4,0}+\zeta_{2,2})}{2\zeta_{2,0}^4}.
\end{align*}
In particular, $\mathbb E[\hat K_{n,Q}(u,v)]=1$ when $\rho=1$.
\end{theorem}


\vspace{0.1in}
\begin{observation} \label{obv3}
$\mathbb E[\hat K_{n,Q}]\approx \frac{2\theta_{1,1}^2}{\zeta_{2,0}}K(u,v)$ as $m\rightarrow \infty$.
\end{observation}

Observation~\ref{obv1} and \ref{obv3} says that, $\mathbb E[\hat K_{Q}]$ and $\mathbb E[\hat K_{n,Q}]$ approximately equal to some scaled version of true kernel, which will motivate our discussion in Section~\ref{sec:approximation error} on the robust kernel approximation error metrics.

\begin{figure}[h]
    \begin{center}
        \mbox{
        \includegraphics[width=2.2in]{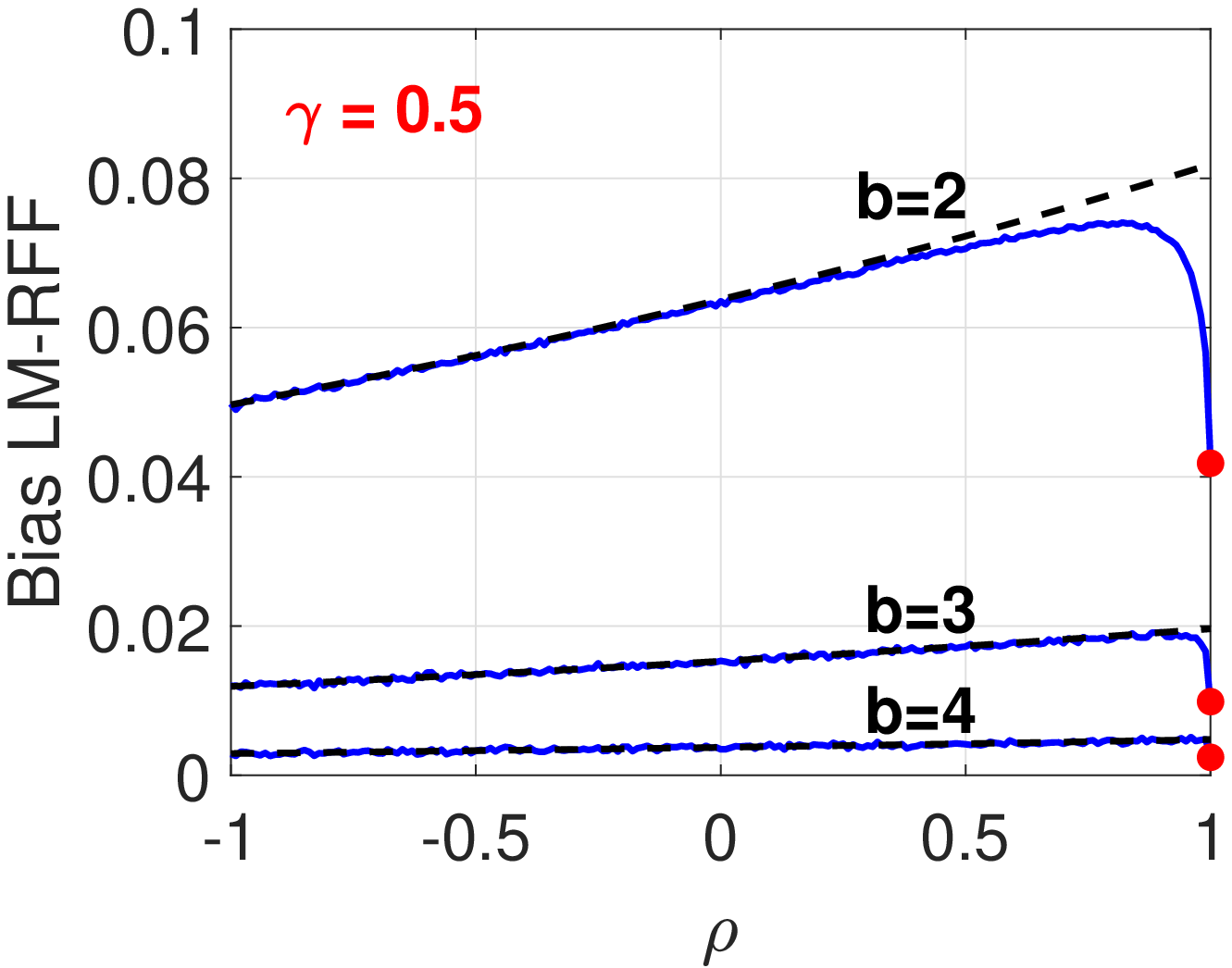}
        \includegraphics[width=2.2in]{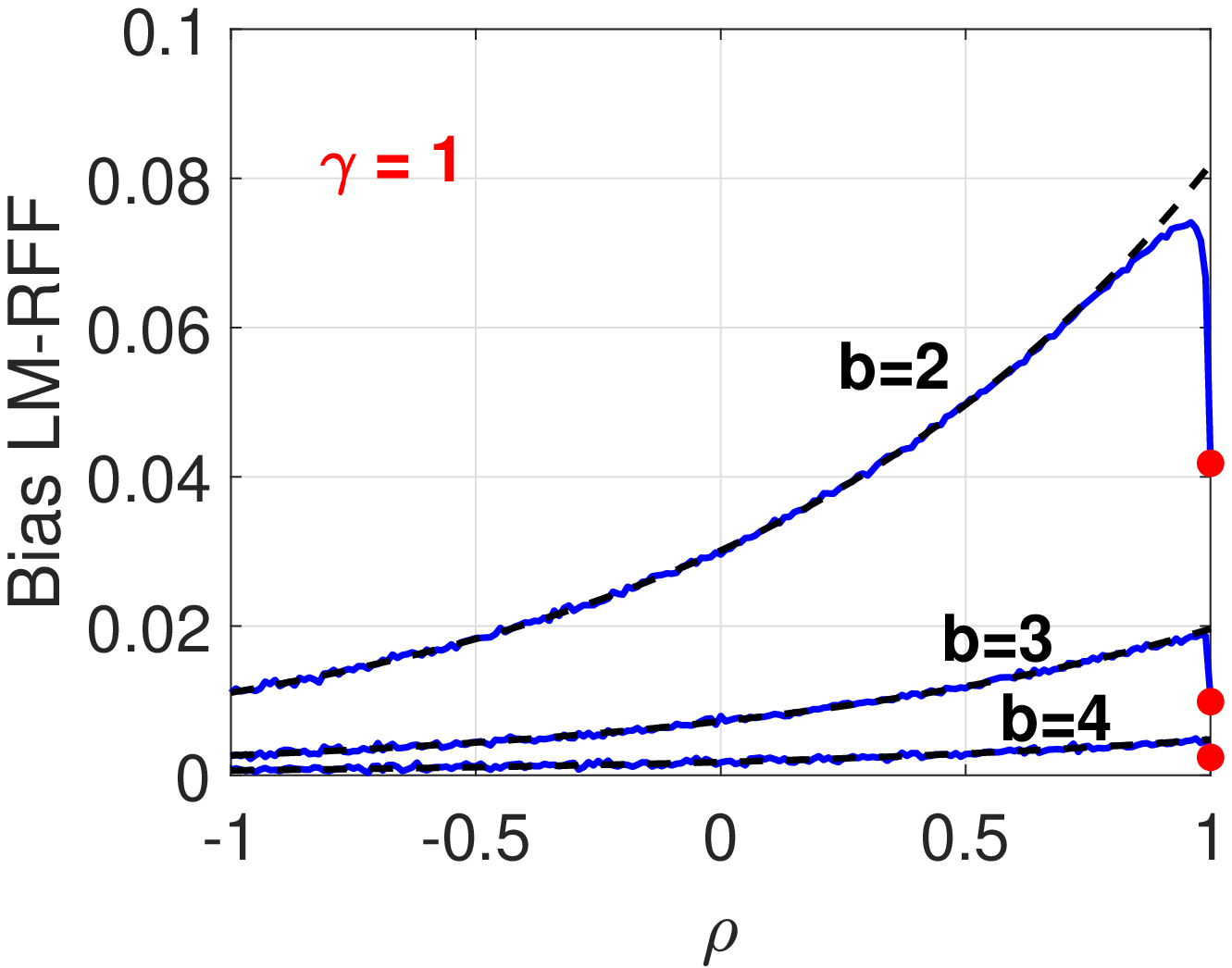}
        \includegraphics[width=2.2in]{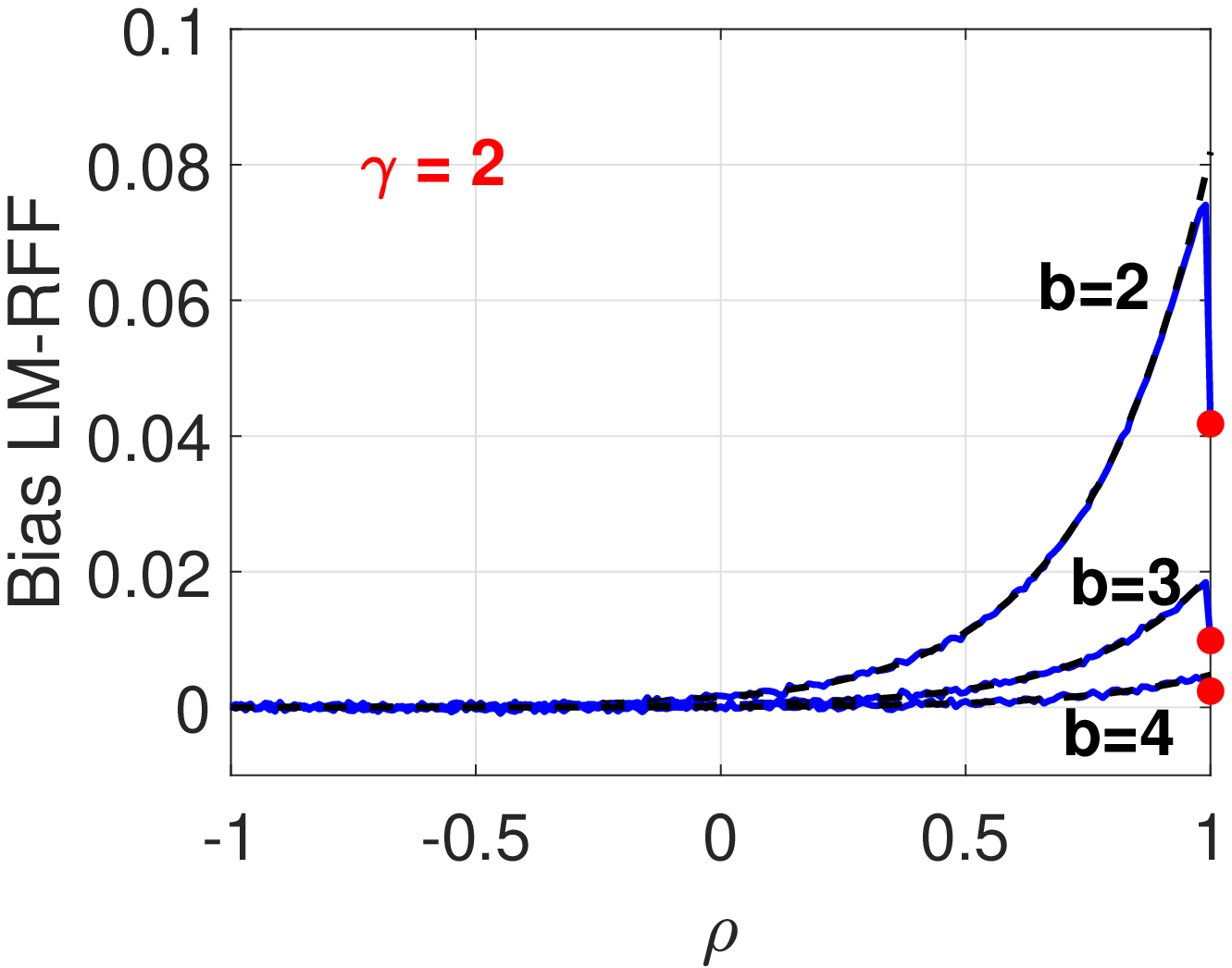}
        }
        \mbox{
        \includegraphics[width=2.2in]{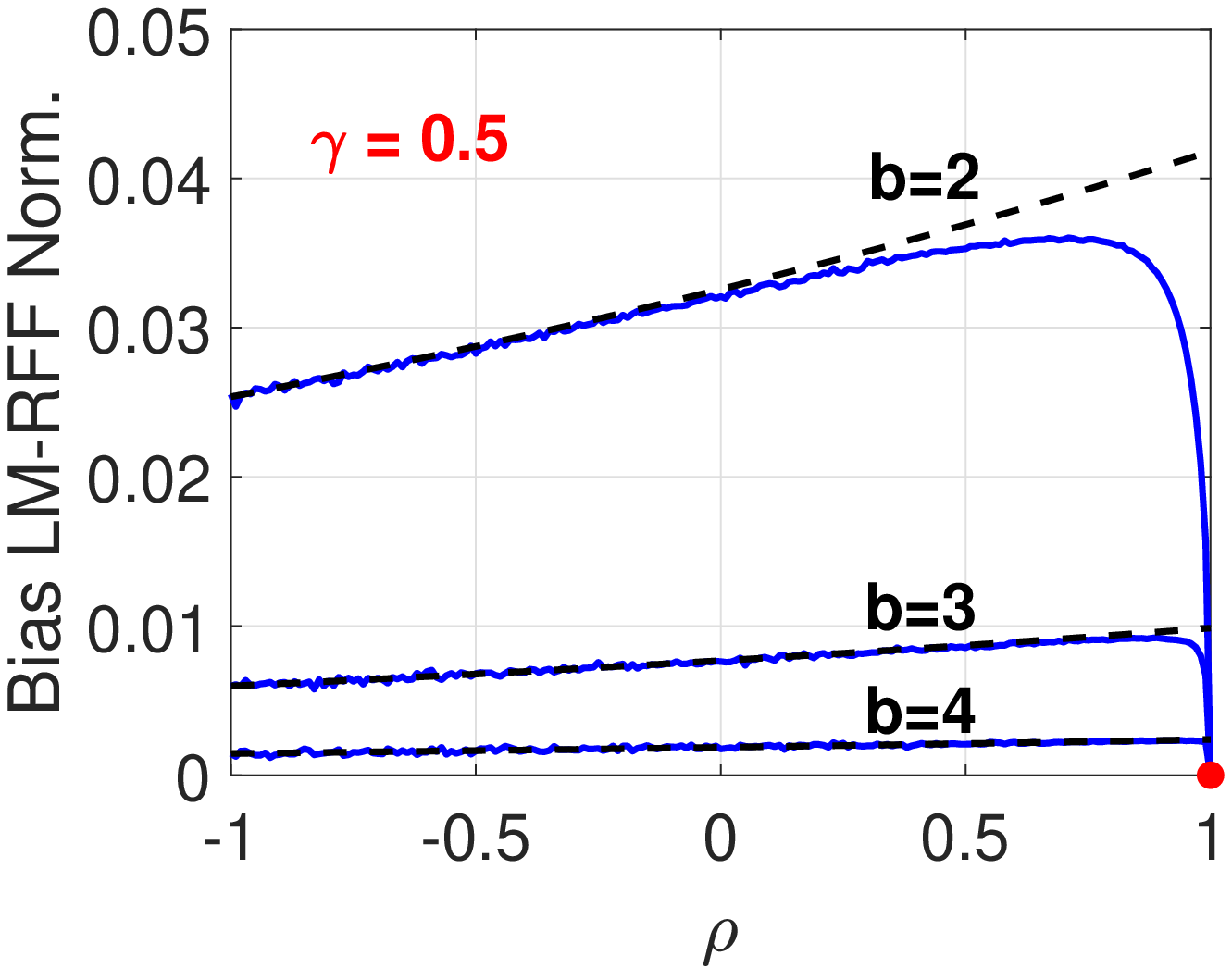}
        \includegraphics[width=2.2in]{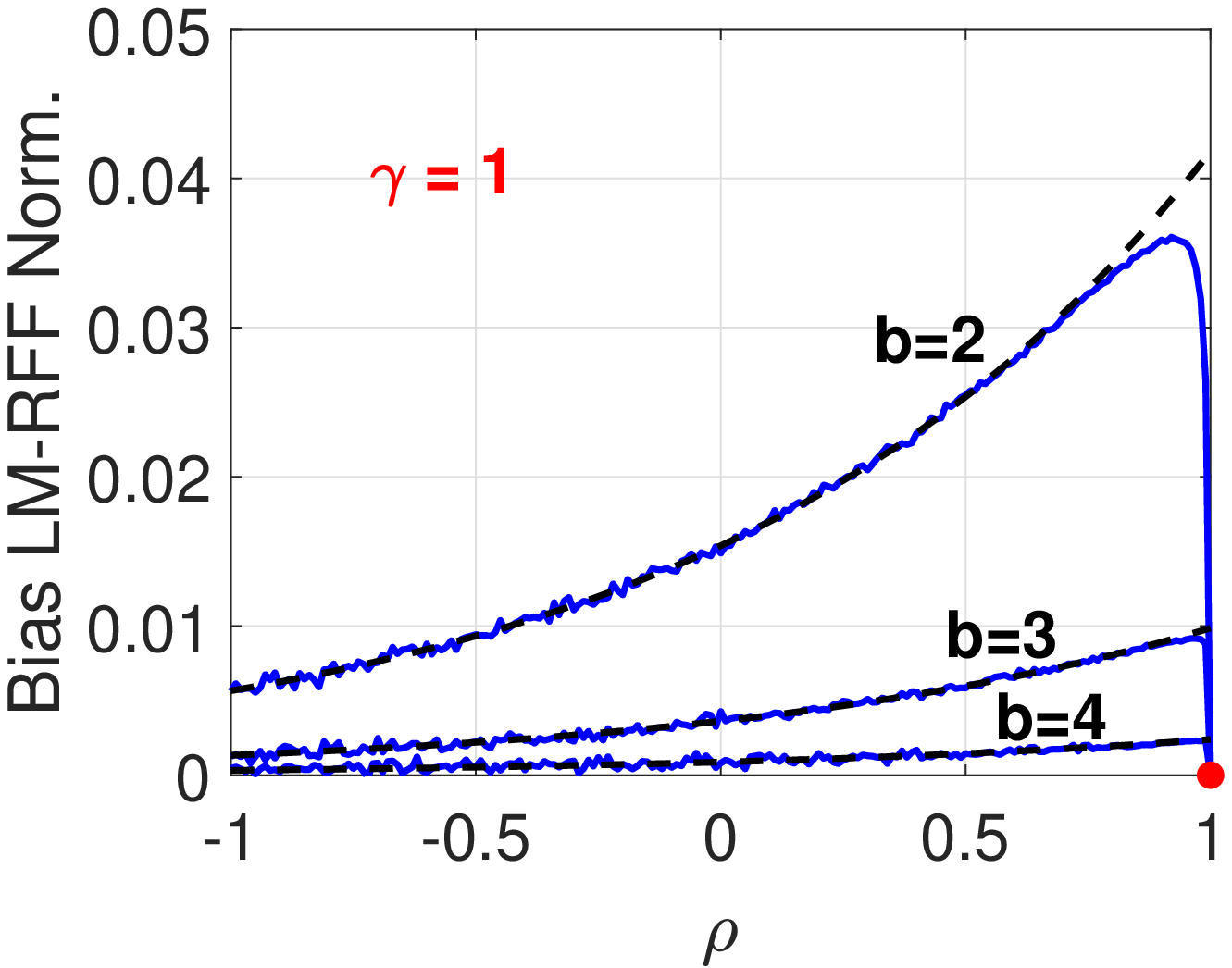}
        \includegraphics[width=2.2in]{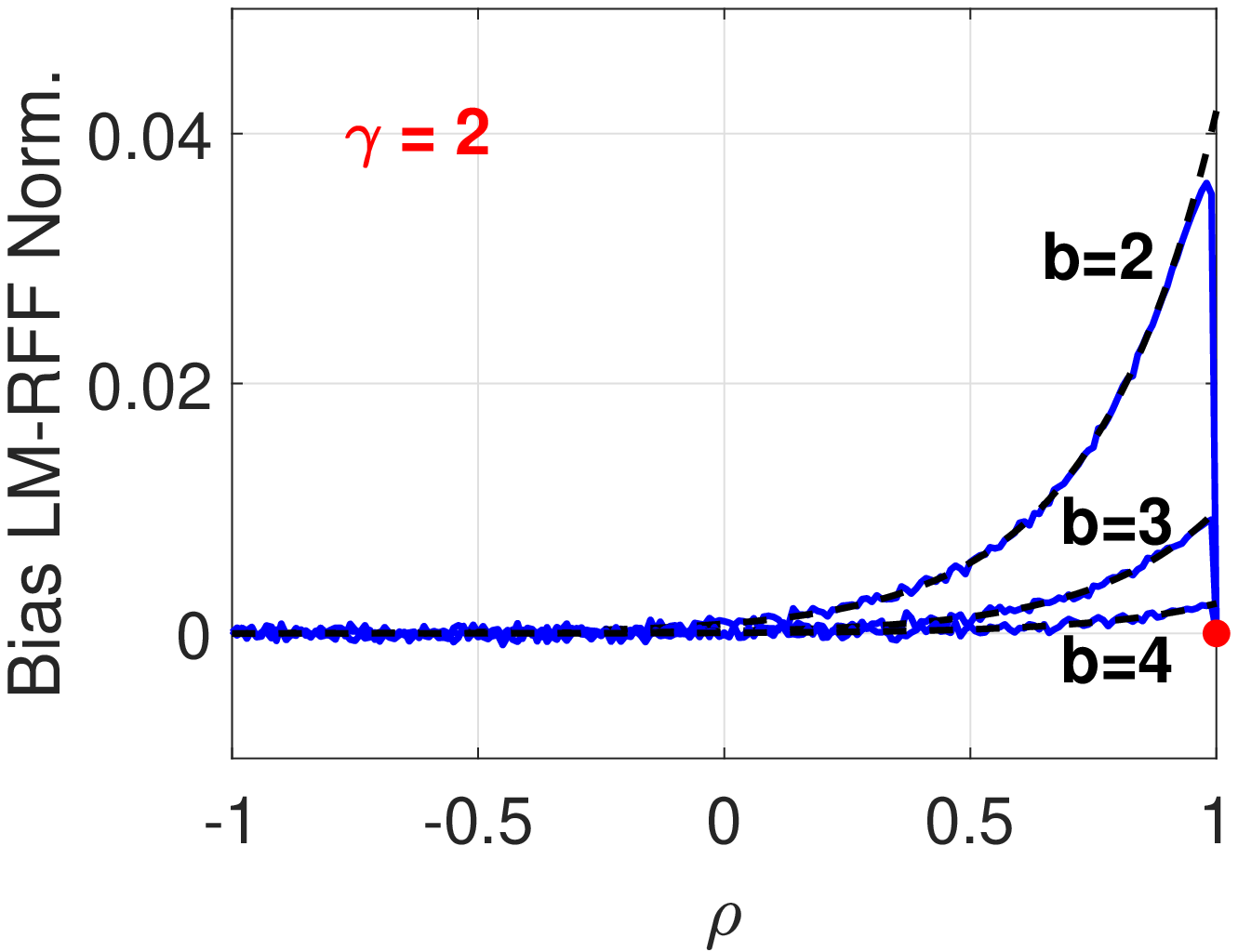}
        }
    \end{center}
    \vspace{-0.3in}
	\caption{Observation~\ref{obv1} and~\ref{obv3} (black dash curves) vs. empirical bias (blue curves) of LM-RFF. Red dots are the biases given in the theorems at specific $\rho$ values.}
	\label{fig_bias}
\end{figure}

\vspace{0.1in}
\noindent\textbf{Validation.}\hspace{0.1in}We plot the empirical bias of LM-RFF against Observations~\ref{obv1} and~\ref{obv3} in Figure~\ref{fig_bias}. As we see, the proposed surrogates for bias align with true biases very well when $\rho$ is not very close to $1$. The biases shrink to $0$ as $b$ increases (e.g., negligibly $\mathcal O(10^{-3})$ with $b=4$). As $\rho\rightarrow 1$, at some "disjoint point" the absolute biases have sharp drops and quickly converge to the theoretical values (red dots) given in Theorem~\ref{theo:mean-var} and~\ref{theo: mean-var-norm}. As $b$ or $\gamma$ increases, the ``disjoint point'' gets closer to $\rho=1$.

\begin{figure}[h!]
    \begin{center}
        \mbox{
        \includegraphics[width=2.2in]{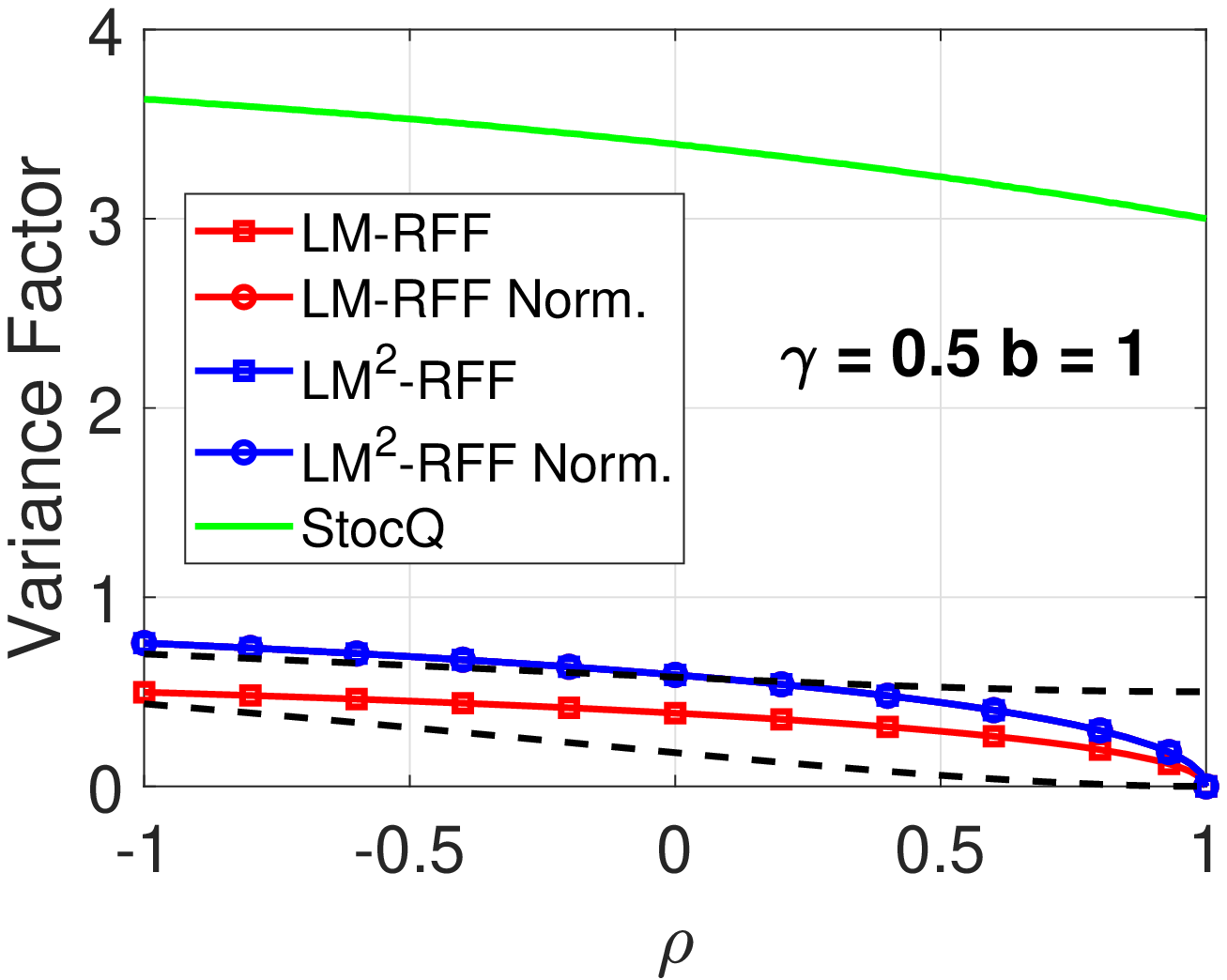}
        \includegraphics[width=2.2in]{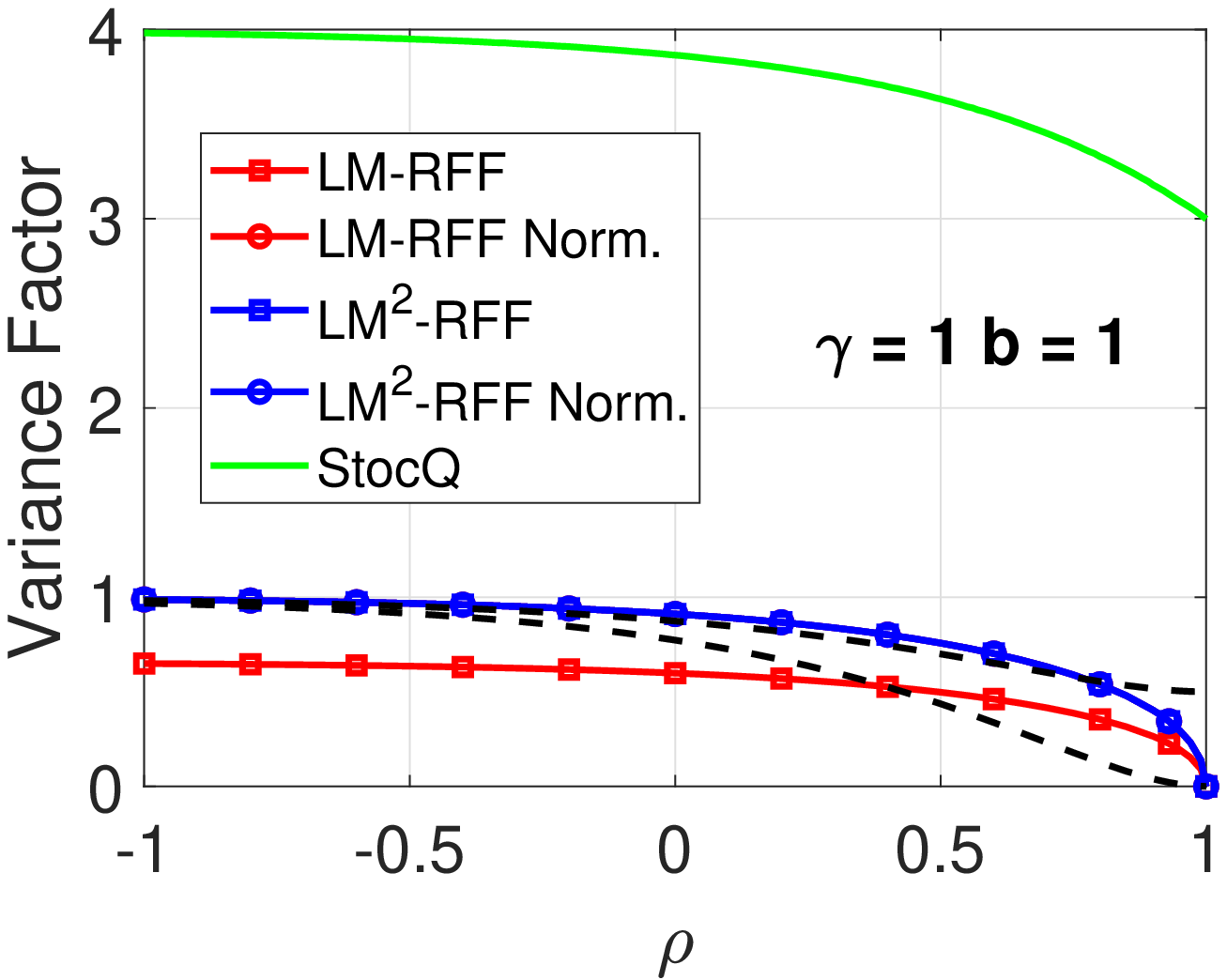}
        \includegraphics[width=2.2in]{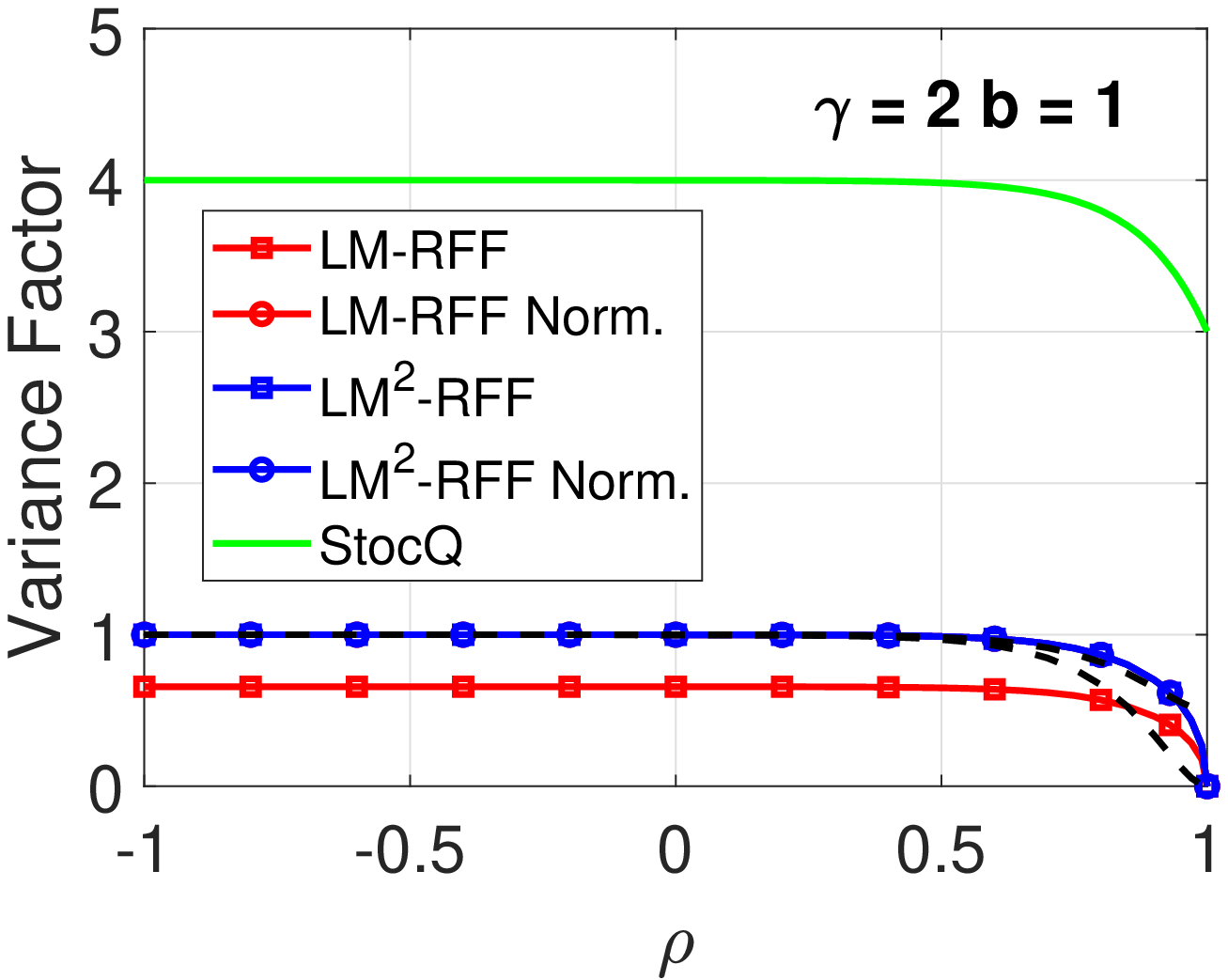}
        }
        \mbox{
        \includegraphics[width=2.2in]{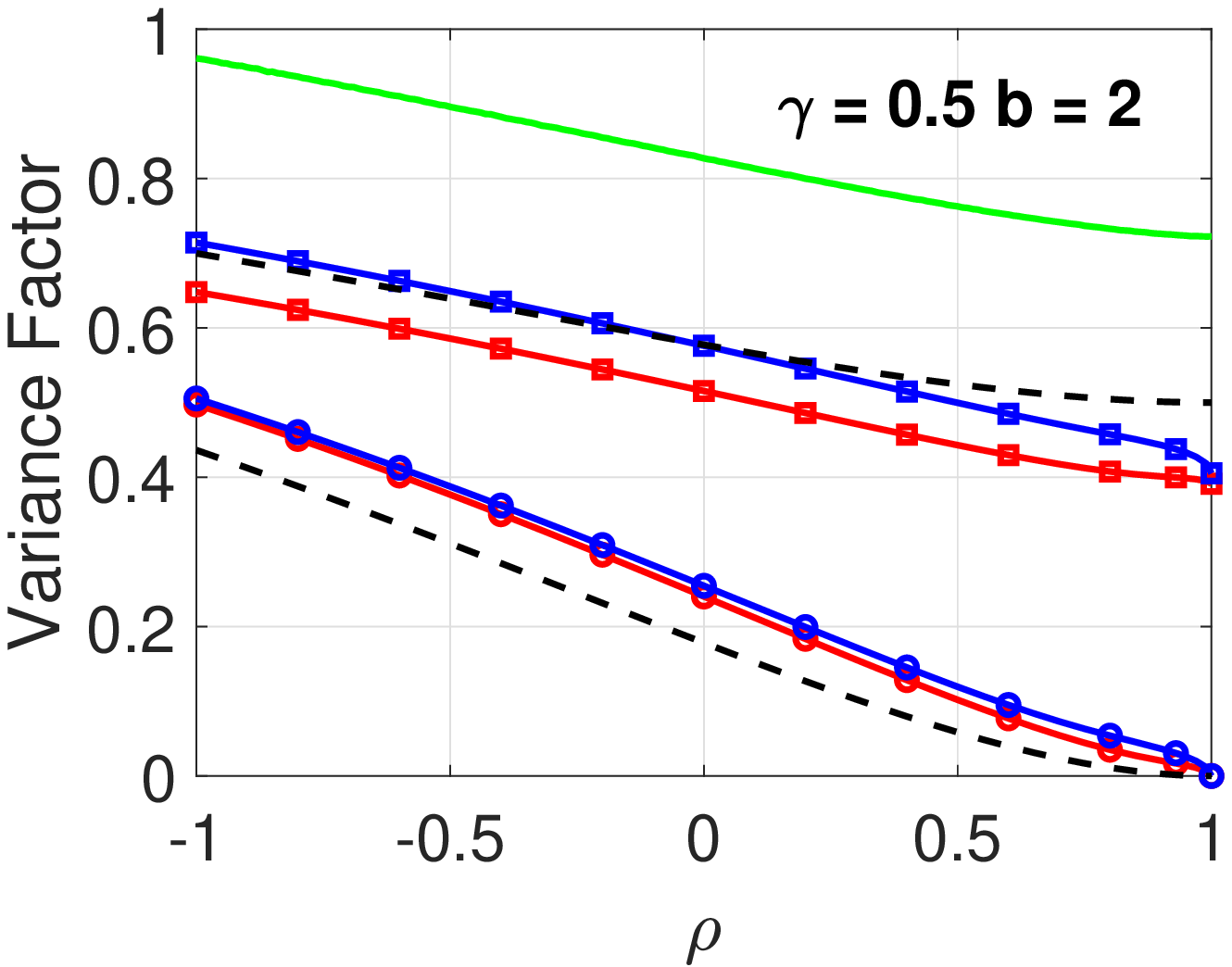}
        \includegraphics[width=2.2in]{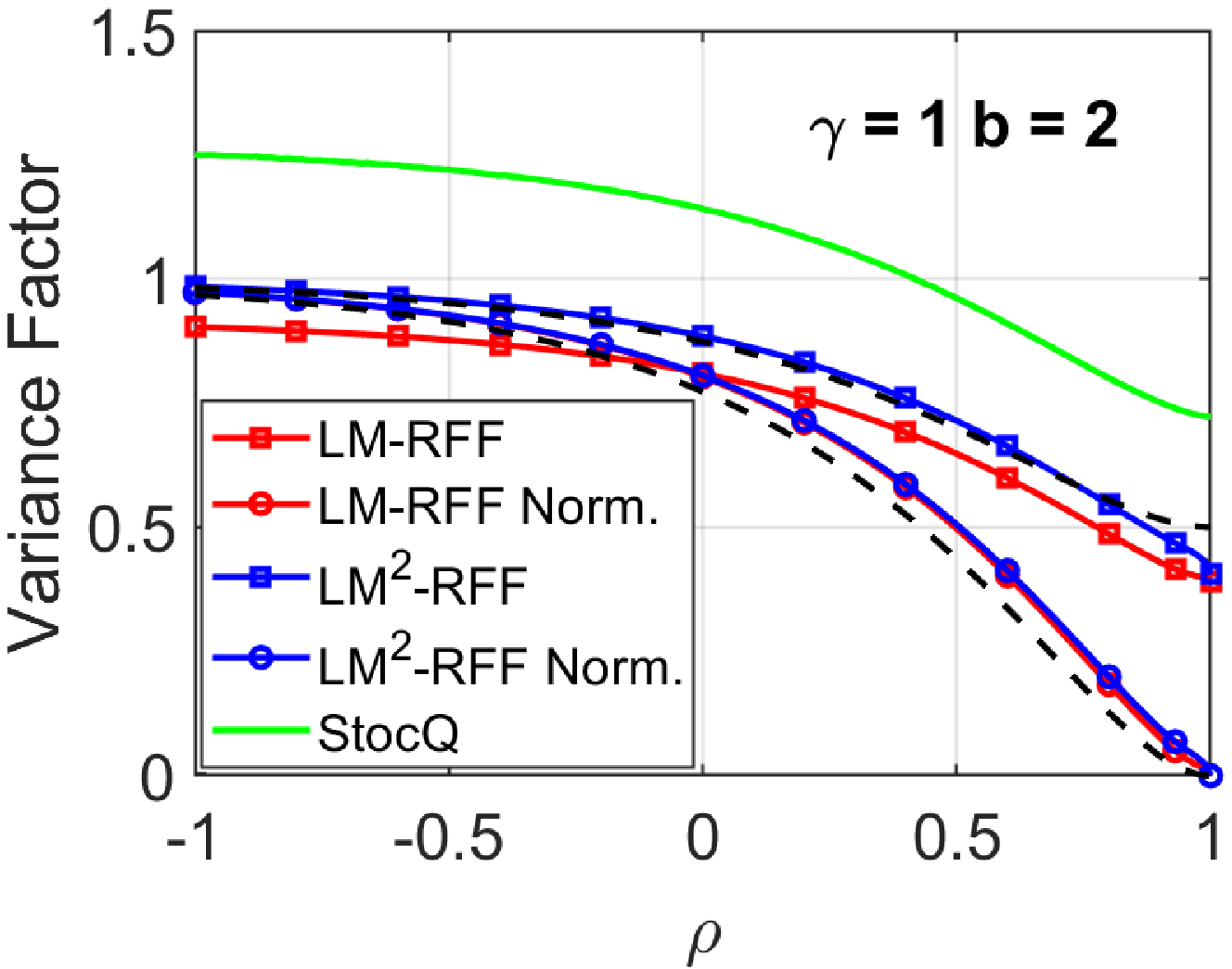}
        \includegraphics[width=2.2in]{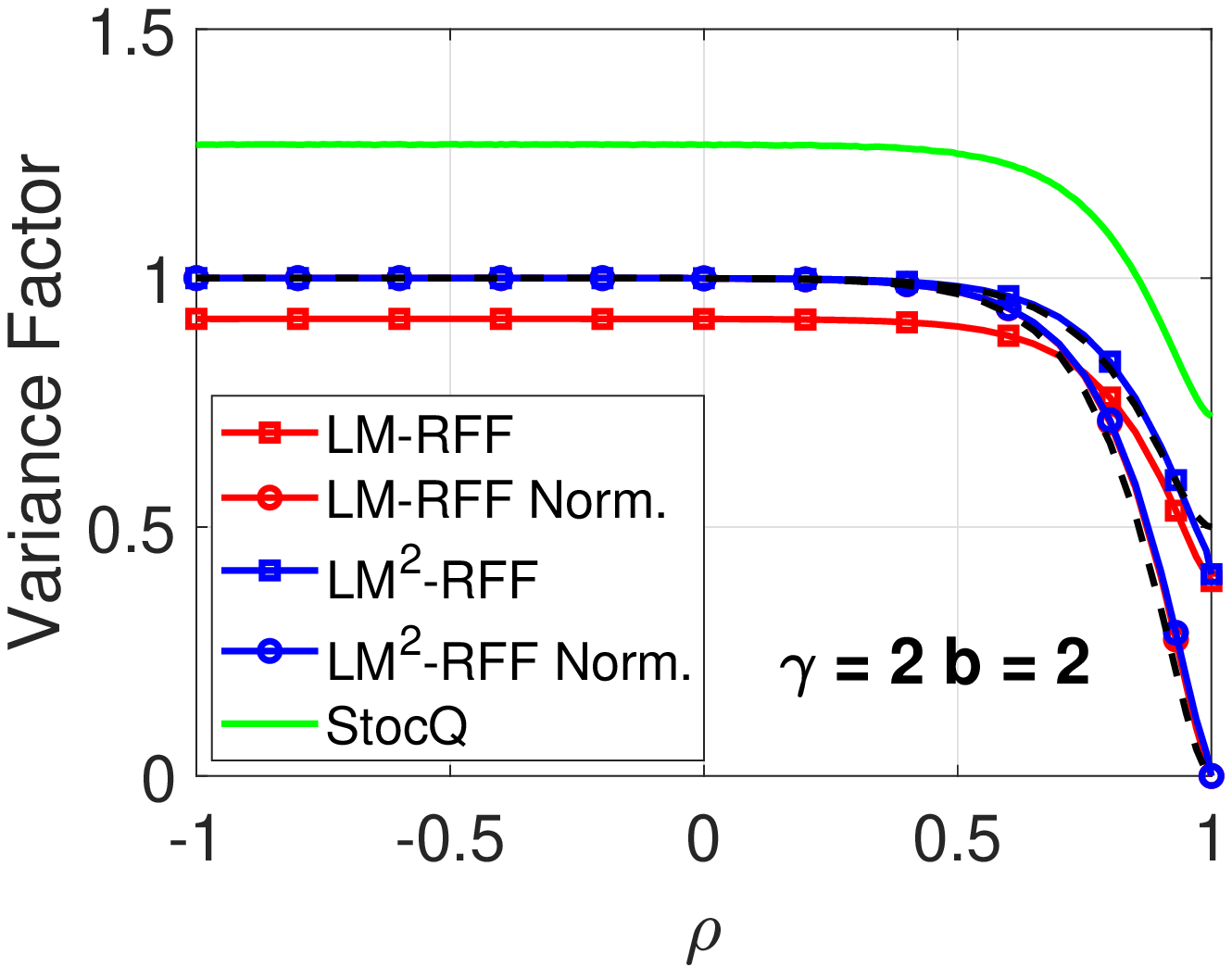}
        }
    \end{center}
    \vspace{-0.25in}
	\caption{Variance (scaled by $m$) of StocQ, LM, and normalized LM kernel estimators, with $\gamma=0.5, 1, 2$. The dashed curves are the variances of corresponding full-precision estimators, to which the variance of quantized estimators converges as $b$ increases.}
	\label{fig:exact_variance}
\end{figure}

\subsection{Variance Comparisons}\label{sec:variance}

Figure~\ref{fig:exact_variance} provides variance comparisons, where full-precision estimator variances are plotted for reference. As $b$ gets larger, the variances of LM-type quantized estimators converges to those of full-precision estimators. The variance of StocQ is significantly larger than RFF and LM quantization, especially when $b=1,2$. This to a good extent explains why StocQ performs poorly in approximate kernel learning with low bits (Section~\ref{sec:experiment}).

\noindent\textbf{Variance of debiased kernel estimates.}\hspace{0.1in} As shown previously, LM estimators are slightly biased which brings theoretical challenges on finding a method to ``properly'' compare their variances. In this paper, we investigate the concept of ``debiased variance'', which refers to the estimator variance after  bias corrections.

\vspace{0.1in}

\begin{definition}[DB-variance] \label{def:debias variance}
For data points $u$ and $v$ with $\rho=u^Tv$, and a kernel estimator $\hat K(u,v;\rho)$ with $\mathbb E[\hat K]=E(\rho)>0$ and $Var[\hat K]=V(\rho)$, the debiased variance of $\hat K(u,v;\rho)$ at $\rho$ is defined as $Var^{db}[\hat K]=V(\rho)\cdot K(\rho)^2/E(\rho)^2$.
\end{definition}

Note that, the debiasing step is only for analytical purpose. Intuitively, Definition~\ref{def:debias variance} is reasonable in that it compares the variation of different estimation procedures given that they have same mean. It is worth mentioning that, DB-variance is invariant of linear scaling, i.e., $c\hat K$ and $\hat K$ have same DB-variance for $c>0$. Classical metrics for estimation quality, such as the Mean Squared Error (MSE), might be largely affected by such simple scaling. Note that, the DB-variance of all 1-bit estimators (both simple and normalized) from fixed quantizers are essentially identical. This can be easily verified by writing every 1-bit quantizer as $Q(z)=sign(z)\cdot C_Q$ for some $C_Q>0$ and substituting it into (\ref{est:Q estimator-I}) and (\ref{est:normalize}). Thus, we will focus on multi-bit quantizers (i.e., $b\geq 2$).

\vspace{0.1in}
\noindent\textbf{LM-RFF v.s. LM$^2$-RFF.}\hspace{0.1in}In Figure~\ref{fig:var-type1 vs square}, we provide the DB-variance ratio of LM$^2$-RFF estimator against that of LM-RFF estimator in the 2-bit case. (The observed pattern is the same for more bits.) For the simple kernel estimator, we see that in general LM-RFF has smaller DB-variance. Yet, the DB-variance of LM$^2$-RFF sharply drops towards 0 and beats LM-RFF as $\rho\rightarrow 1$, i.e., in high similarity region, which meets the goal of LM$^2$-RFF quantizer design (to favor high similarity region). However, for normalized estimators, $\hat K_{n,Q}$ has consistently smaller DB-variance than $\hat K_{n,Q,(2)}$.

\begin{figure}[h]
    \begin{center}
        \mbox{
        \includegraphics[width=2.7in]{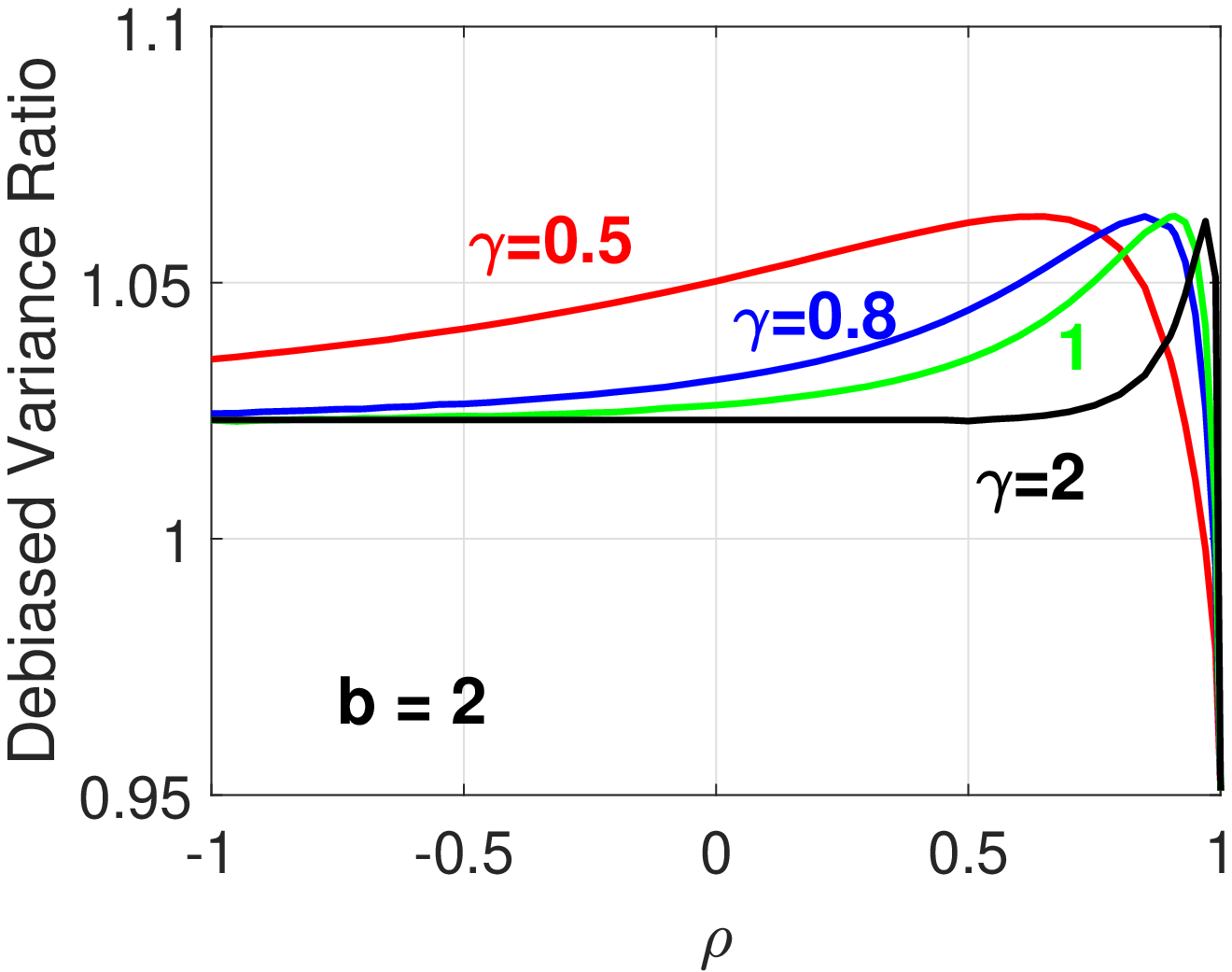}
        \includegraphics[width=2.7in]{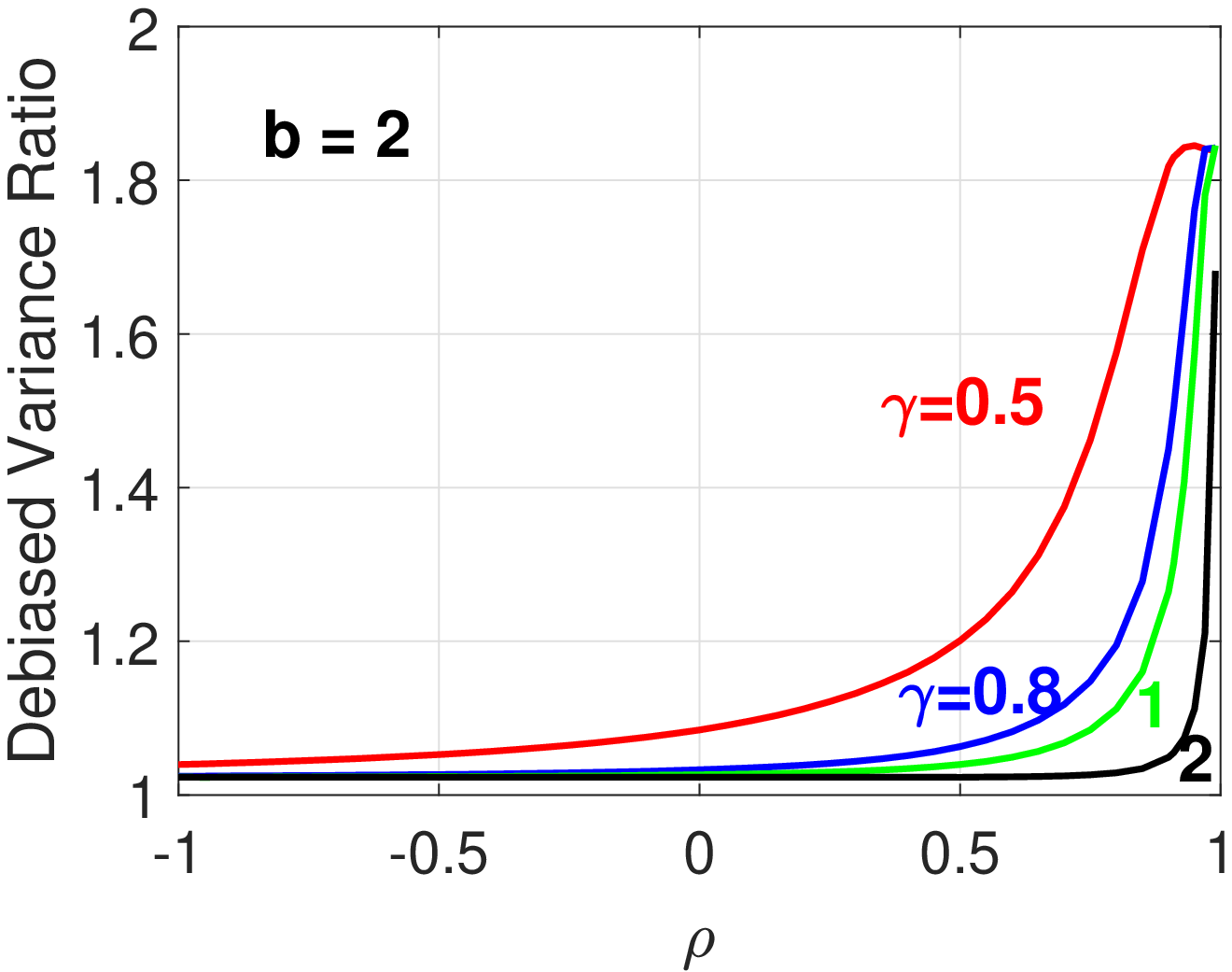}
        }
    \end{center}
    \vspace{-0.2in}
	\caption{Debiased variance ratio of LM$^2$-RFF over LM-RFF, at various $\gamma$, $b=2$. \textbf{Left panel:} vanilla estimator $\frac{Var^{db}[\hat K_{Q,(2)}]}{Var^{db}[\hat K_{Q}]}$. \textbf{Right panel:} normalized estimator $\frac{Var^{db}[\hat K_{n,Q,(2)}]}{Var^{db}[\hat K_{n,Q}]}$.}
	\label{fig:var-type1 vs square}
	\vspace{-0.05in}
\end{figure}

\begin{figure}[h]
    \begin{center}
        \mbox{
        \includegraphics[width=2.7in]{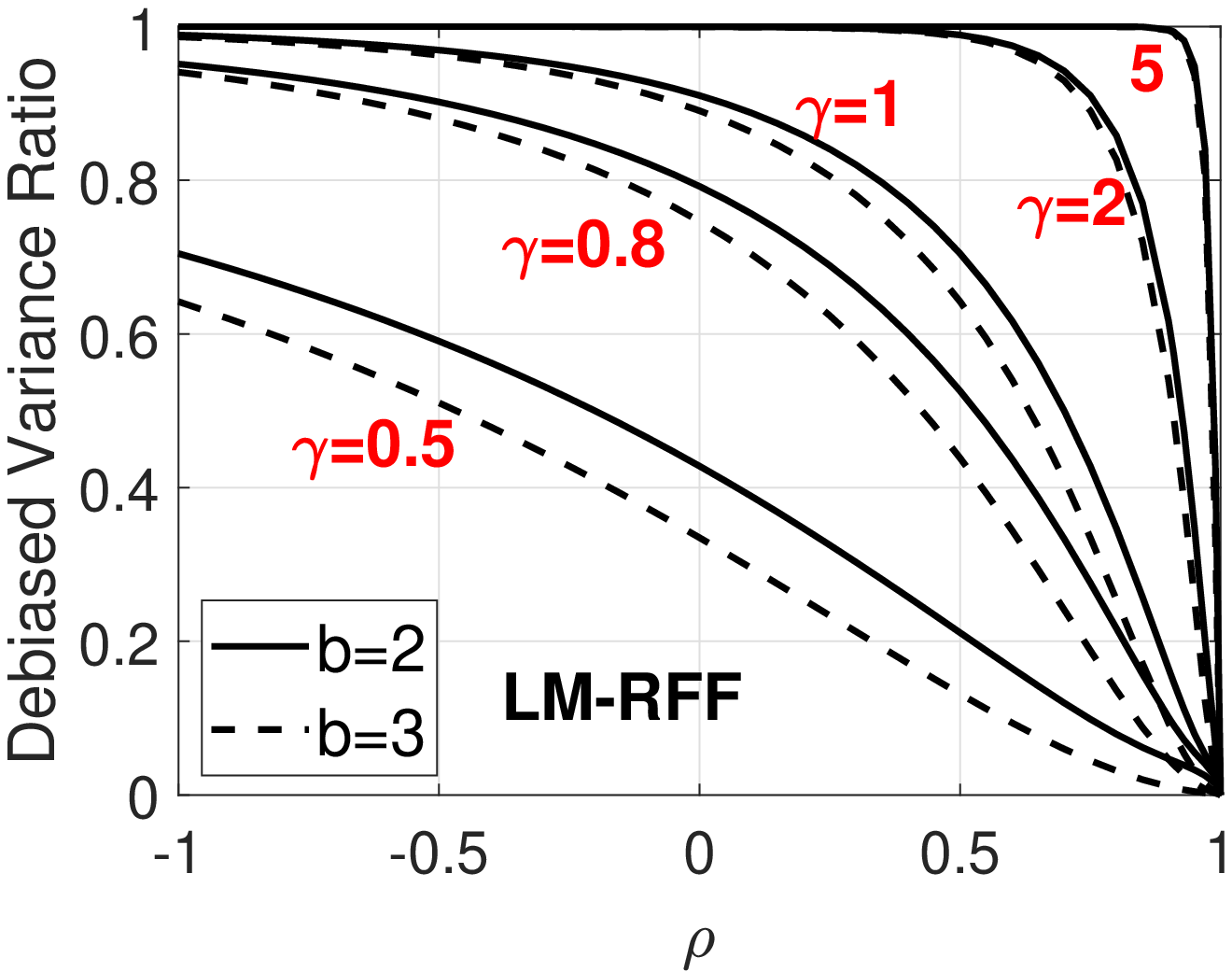}
        \includegraphics[width=2.7in]{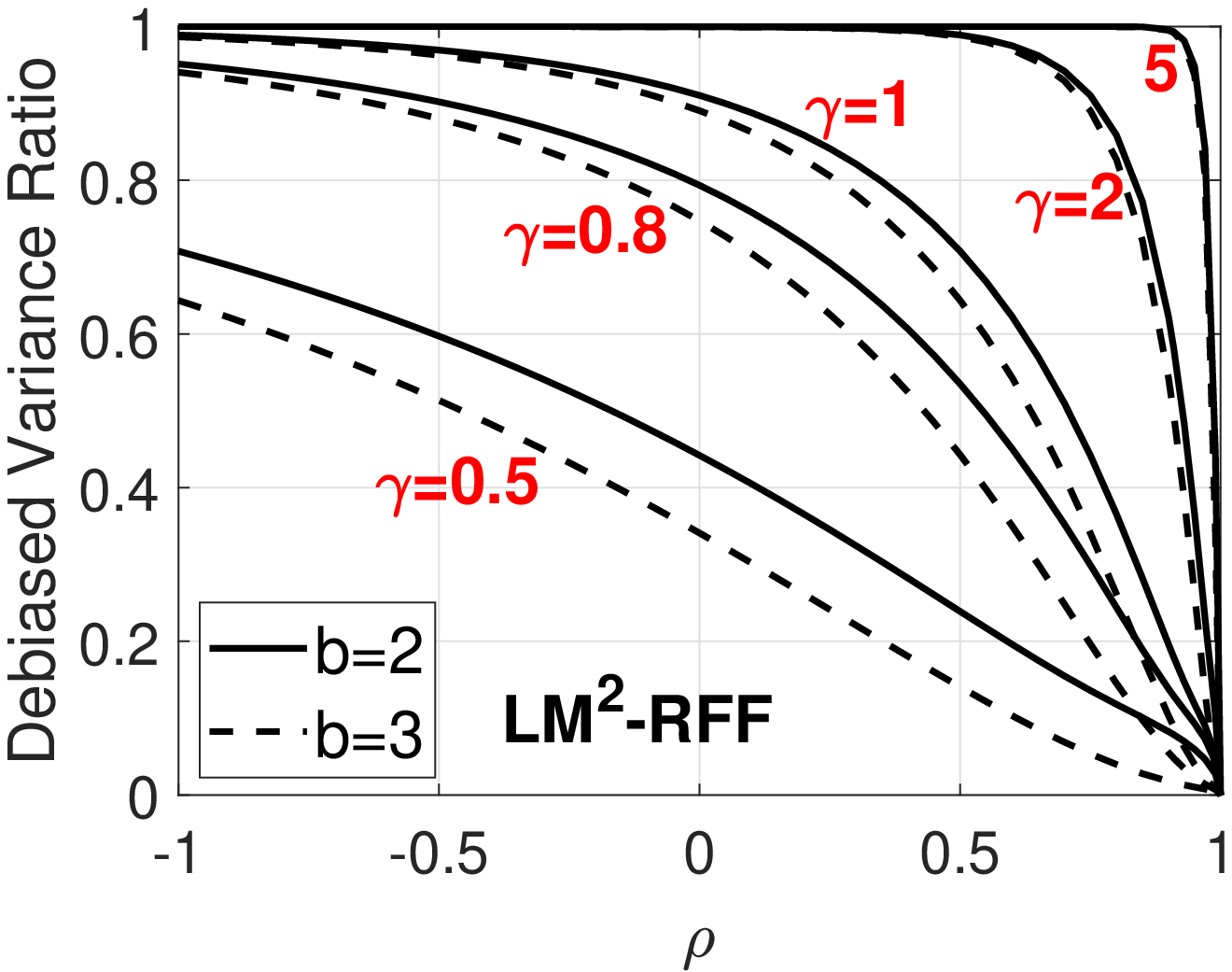}
        }
    \end{center}
    \vspace{-0.25in}
	\caption{Debiased variance ratio of normalized LM estimators against simple LM estimators, $b=2,3$. \textbf{Left panel:} LM-RFF, $\frac{Var^{db}[\hat K_{n,Q}]}{Var^{db}[\hat K_{Q}]}$. \textbf{Right panel:} LM$^2$-RFF, $\frac{Var^{db}[\hat K_{n,Q,(2)}]}{Var^{db}[\hat K_{Q,(2)}]}$}
	\label{fig:norm_variance_ratio}\vspace{-0.05in}
\end{figure}

\vspace{0.1in}
\noindent\textbf{Benefit of normalization.}\hspace{0.1in}Next we prove the theoretical merit of normalizing RFFs, in terms of DB-variance.

\begin{theorem}  \label{theo:BD-variance}
Suppose $u,v$ are two samples with correlation $\rho$. Let the simple and normalized kernel estimator, $\hat K_Q$ and $\hat K_{n,Q}$, be defined as~(\ref{est:Q estimator-I}) and (\ref{est:normalize}), respectively, where $Q$ is any LM-type quantizer. Assume $\gamma\leq \pi/\sqrt 2$. Then, $Var^{db}[\hat K_{n,Q}]\leq Var^{db}[\hat K_{Q}]$ on $\rho\in[0,1]$ as $m\rightarrow\infty$.
\end{theorem}

Theorem~\ref{theo:BD-variance} says that when $\gamma\leq \pi/\sqrt 2\approx 2.2$, normalization is guaranteed to reduce the DB-variance at any $\rho\in[0,1]$. In Figure~\ref{fig:norm_variance_ratio}, we plot the DB-variance ratio of $\frac{Var^{db}[\hat K_{n,Q}(x,y)]}{Var^{db}[\hat K_{Q}(x,y)]}$ at multiple $\gamma$ and $b$, for LM-RFF and LM$^2$-RFF respectively. We corroborate the advantage of normalized estimates over simple estimators in terms of DB-variance (ratio always $<1$), especially with large $\rho$.

\subsection{Monotonicity of Mean Kernel Estimation} \label{sec:monotone}

For a kernel estimator $\hat K(\rho)$ (written as a function of $\rho$), the monotonicity of its mean estimation $\mathbb E[\hat K(\rho)]$ against $\rho$ is important to ensure its ``correctness''. It guarantees that asymptotically ($m\rightarrow\infty$), the comparison of estimated kernel distances is always correct, i.e., $\hat K(u,v_1)>\hat K(u,v_2)$ if $K(u,v_1)>K(u,v_2)$ for data points $u,v_1,v_2$. Otherwise (say, $\mathbb E[\hat K]$ decreasing in $\rho$ on $[s,t]$), the comparison of estimated kernel would be wrong for $\rho\in [s,t]$ even with infinite much data. By Theorem~\ref{theo:StocQ}, StocQ estimator is unbiased with $\mathbb E[\hat K_{StocQ}]=e^{-\gamma^2(1-\rho)}$ strictly increasing in $\rho$. Hence, we will focus on the fixed LM quantization.

\begin{lemma} \label{lemma:continuous monotone}
Suppose $X,Y,\tau$ are same as Theorem~\ref{theo:mean-var}, and denote $s_x=\gamma X+\tau$, $s_y=\gamma Y+\tau$, such that $z_x=\cos(s_x)$ and $z_y=\cos(s_y)$ are RFFs. Assume $g_1,g_2:[-1,1]\mapsto \mathbb R$ are twice differentiable and bounded functions. Then,
$$\frac{\partial \mathbb E[g_1(z_x)g_2(z_y)]}{\partial \rho}=\gamma^2\mathbb E[g_1'(z_x)\sin(s_x) g_2'(z_y)\sin(s_y)].$$
Furthermore, when $\sqrt{2(1-\rho)}\gamma \leq \pi$, if $g_1$ and $g_2$ are both increasing odd functions or non-constant even functions, then the mean is increasing in $\rho$, i.e., $\frac{\partial \mathbb E[g_1(z_x)g_2(z_y)]}{\partial \rho}>0$.
\end{lemma}


\begin{figure}[h]
\vspace{-0.18in}
    \begin{center}
        \mbox{
        \includegraphics[width=2.2in]{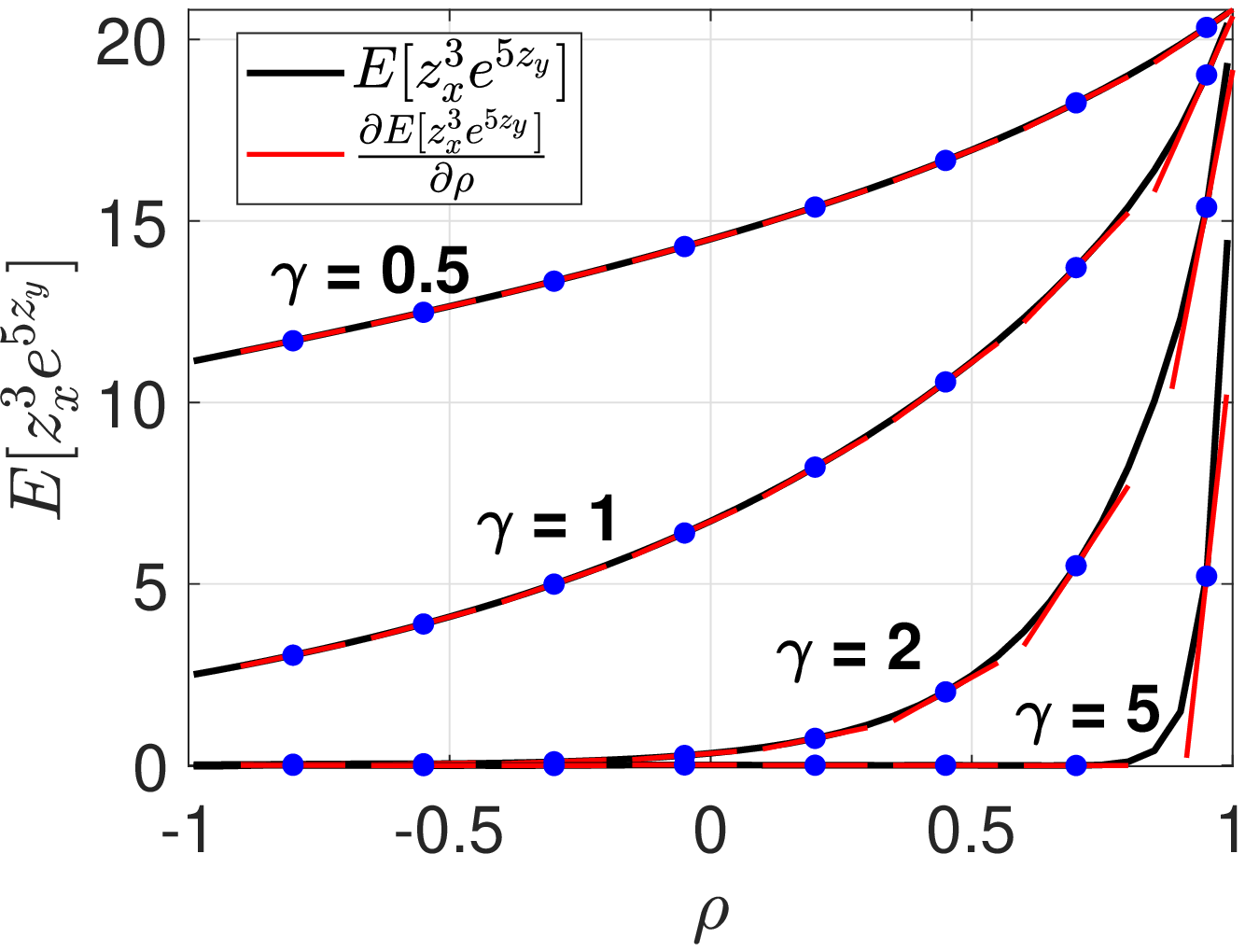}
        \includegraphics[width=2.2in]{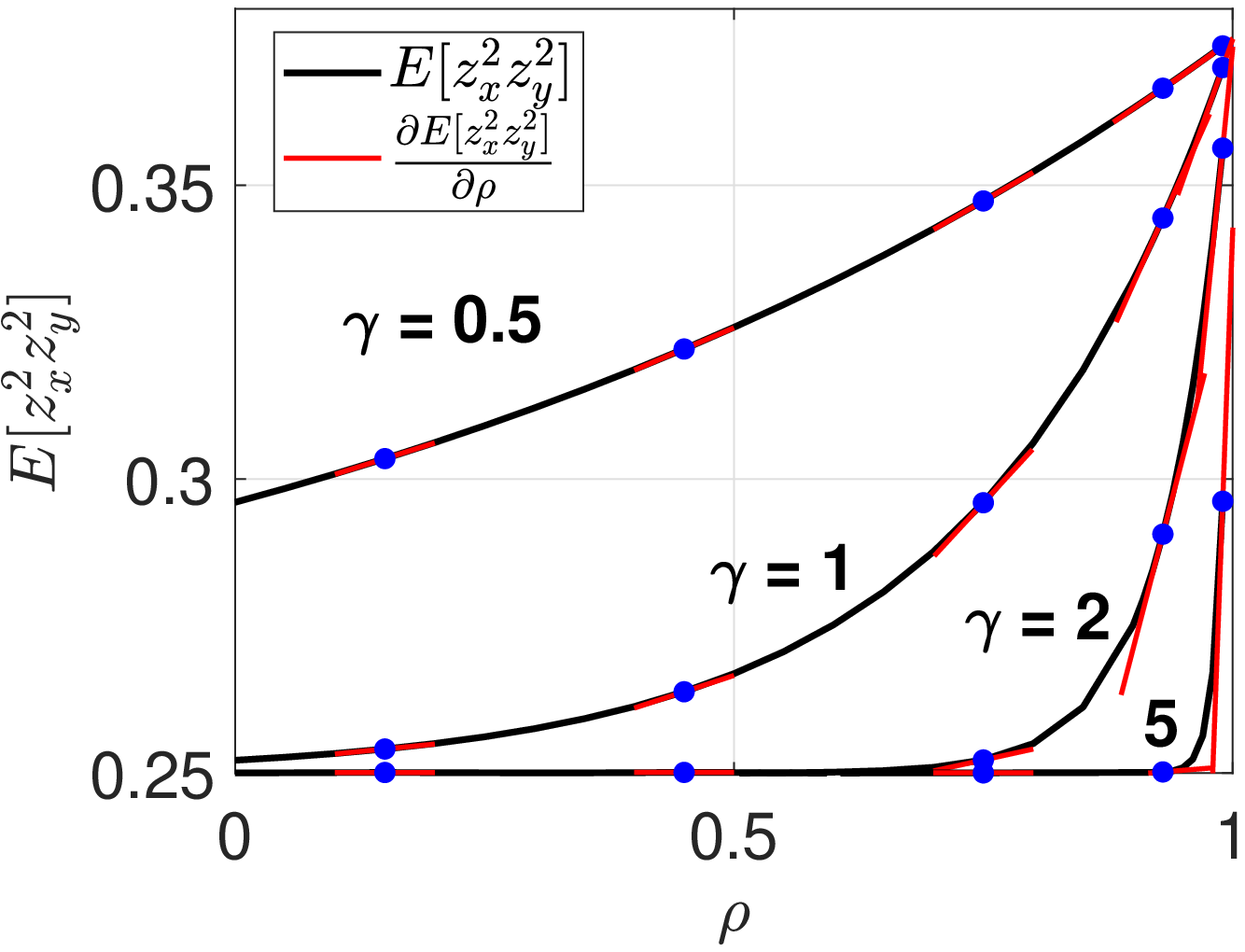}
        \includegraphics[width=2.2in]{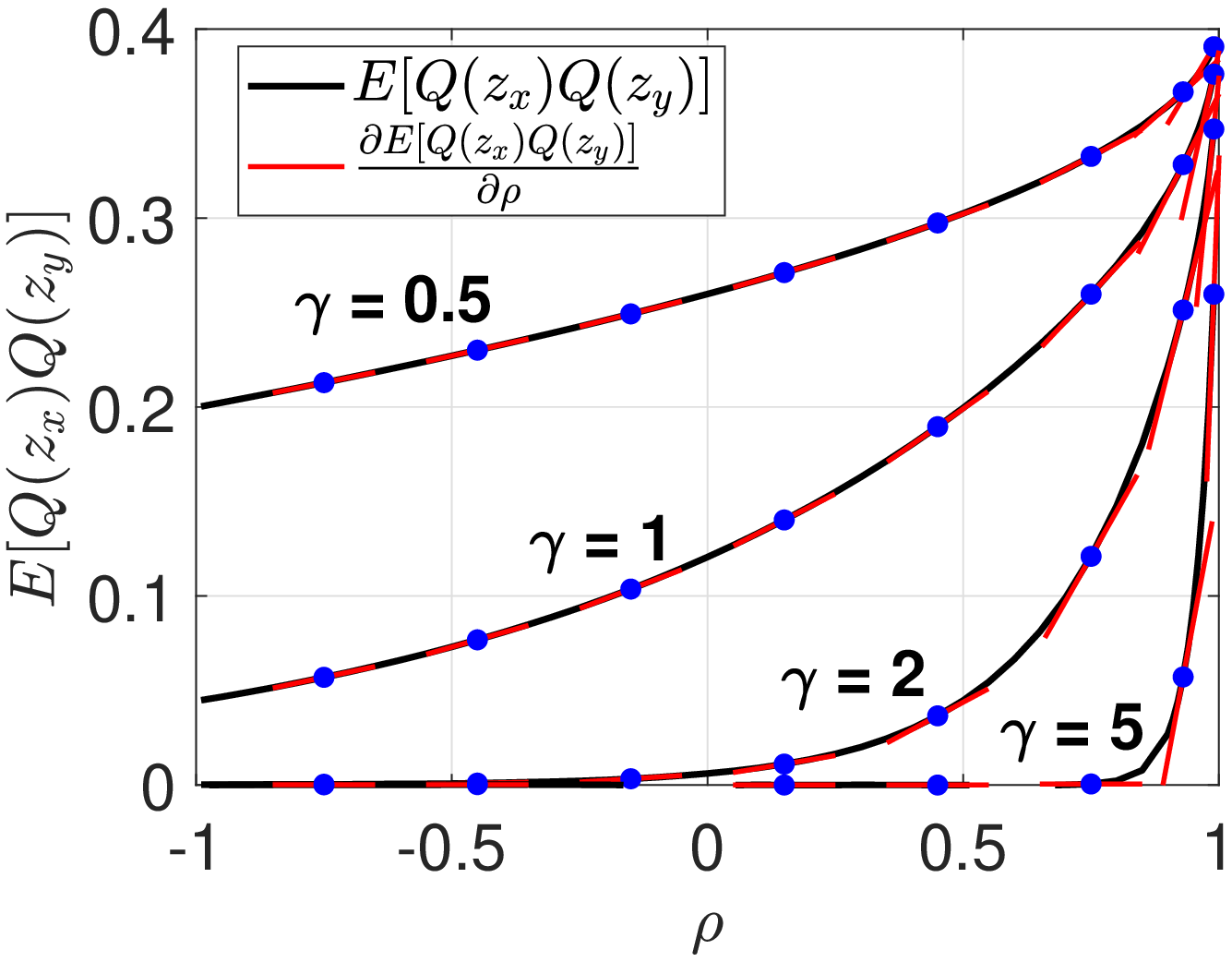}
        }
    \end{center}
    \vspace{-0.25in}
	\caption{The derivatives obtained by Lemma~\ref{lemma:continuous monotone} and Theorem~\ref{theo:discrete monotone} with different $g_1$ and $g_2$ functions, at multiple $\gamma$. Black curves are the function value, and red lines are the theoretical derivatives at multiple points. \textbf{Left panel:} $g_1(x)=x^3$, $g_2(x)=e^{5x}$, increasing functions. \textbf{Middle panel:} $g_1(x)=g_2(x)=x^2$, even functions. The region $\rho\geq 0$ is magnified for clarity. \textbf{Right panel:} $g_1(x)=g_2(x)=Q$, where $Q$ is the 1-bit LM-RFF quantizer. Here the derivative is approximated by the continuous approximation $\tilde g(x)=\mu_2 sign(x)(1-e^{-50|x|})$, where $\mu_2$ is the positive reconstruction level of $Q$.}
	\label{fig:monotone}
\end{figure}

\vspace{0.1in}

Lemma~\ref{lemma:continuous monotone} is a general result for the monotonicity when RFFs are processed by continuous functions. In the left and mid panels of Figure~\ref{fig:monotone}, we plot two examples of $\mathbb E[g_1(s_x)g_2(s_y)]$ against $\rho$, with continuously increasing functions $g_1(x)=x^3$ and $g_2(x)=e^{5x}$ and even functions $g_1(x)=g_2(x)=x^2$, respectively. As we can see, the expectation is increasing in $\rho$ with true derivatives given by Lemma~\ref{lemma:continuous monotone}.

Next, in the following theorem, we extend the above result to discrete functions, which include our proposed LM quantizers as special cases.

\begin{theorem} \label{theo:discrete monotone}
Suppose $Q_1$ and $Q_2$ are bounded, discrete, and non-decreasing odd functions or non-constant even functions, with finite many discontinuities. Let $z_x$ and $z_y$ be defined as Lemma~\ref{lemma:continuous monotone}. If $\sqrt{2(1-\rho)}\gamma\leq\pi$, then $\mathbb E[Q_1(z_x)Q_2(z_y)]$ is increasing in $\rho$.
\end{theorem}

\begin{remark}
The condition $\sqrt{2(1-\rho)}\gamma\leq\pi$ in Theorem~\ref{theo:discrete monotone} implies that $\mathbb E[Q(z_x)Q(z_y)]$ for any LM quantizer increases in $\rho\in[\max(-1,1-\frac{\pi^2}{2\gamma^2}),1]$. Thus, larger $\gamma$ typically requires higher $\rho$ for this condition to hold. For example, when $\gamma=1$ and $\gamma=5$, monotonicity is ensured for $\rho\in[-1,1]$ and $\rho\geq 0.8$, respectively, which is consistent with numerical observations, e.g., see Figure~\ref{fig:monotone}.
\end{remark}

In the right panel of Figure~\ref{fig:monotone}, we plot the mean kernel estimation $\mathbb E[Q(z_x)Q(z_y)]$ (with $z_x,z_y$ being the RFFs) of 1-bit LM-RFF quantization, along with the derivatives predicted by Lemma~\ref{lemma:continuous monotone}. Here we use the continuous approximation $\tilde g(x)=\mu_2 sign(x)(1-e^{-50|x|})$, where $\mu_2$ is the reconstruction level of $Q$, as a surrogate to compute $Q'(x)$. We see that the mean estimation increases in $\rho$, and our theory on the derivative  well aligns with the numerical result.

\section{Numerical Experiments} \label{sec:experiment}

We conduct experiments with compressed RFFs on approximate kernel SVM (KSVM) classification and kernel ridge regression (KRR) tasks. Our results illustrate the effectiveness of large-scale kernel learning with highly compressed RFFs,  highlighting the superior advantage of the proposed LM-RFF quantization. Moreover, we also propose and evaluate robust kernel approximation error metrics to consolidate our claims.

\begin{figure}[h]
        \mbox{\hspace{-0.12in}
        \includegraphics[width=1.72in]{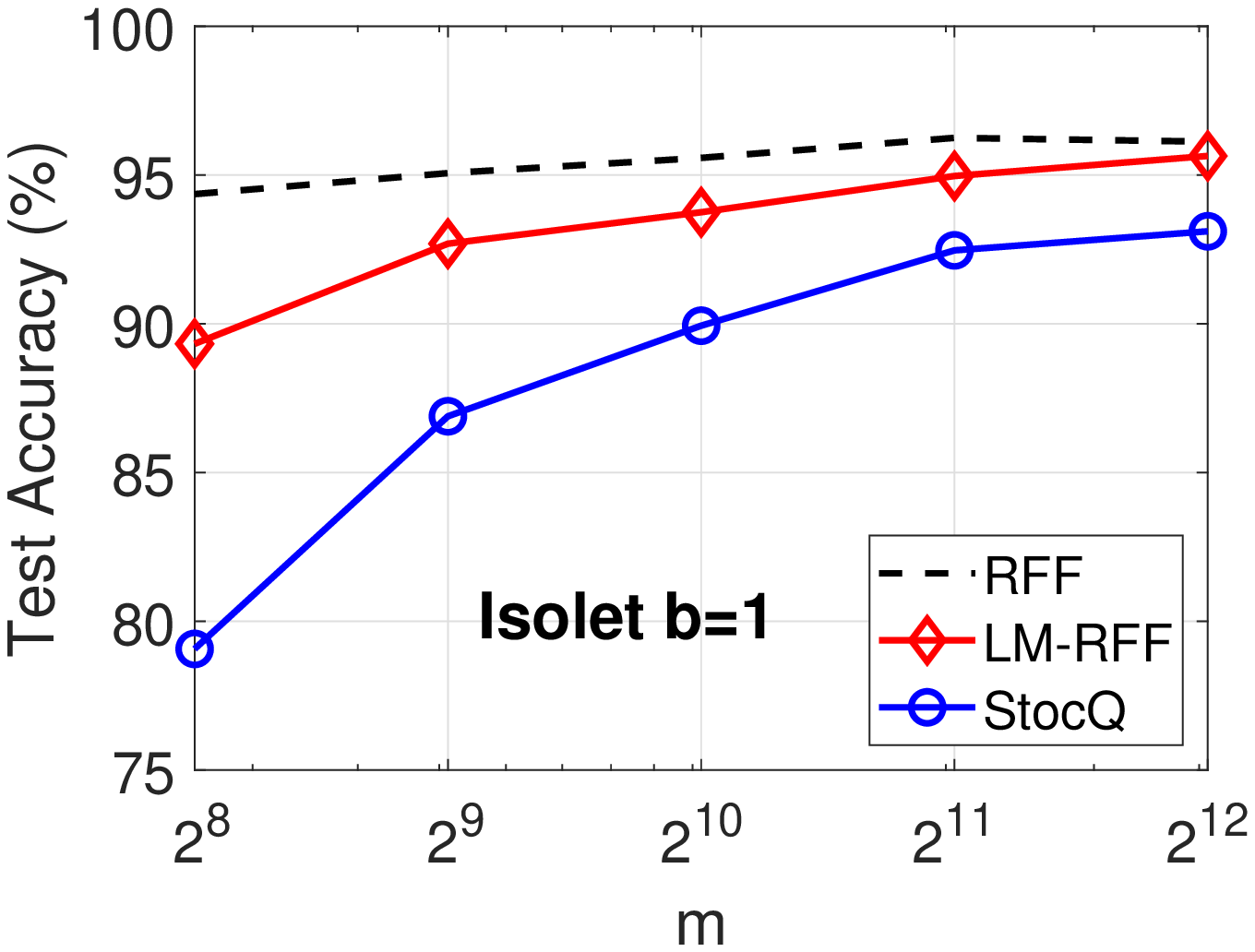}\hspace{-0.1in}
        \includegraphics[width=1.72in]{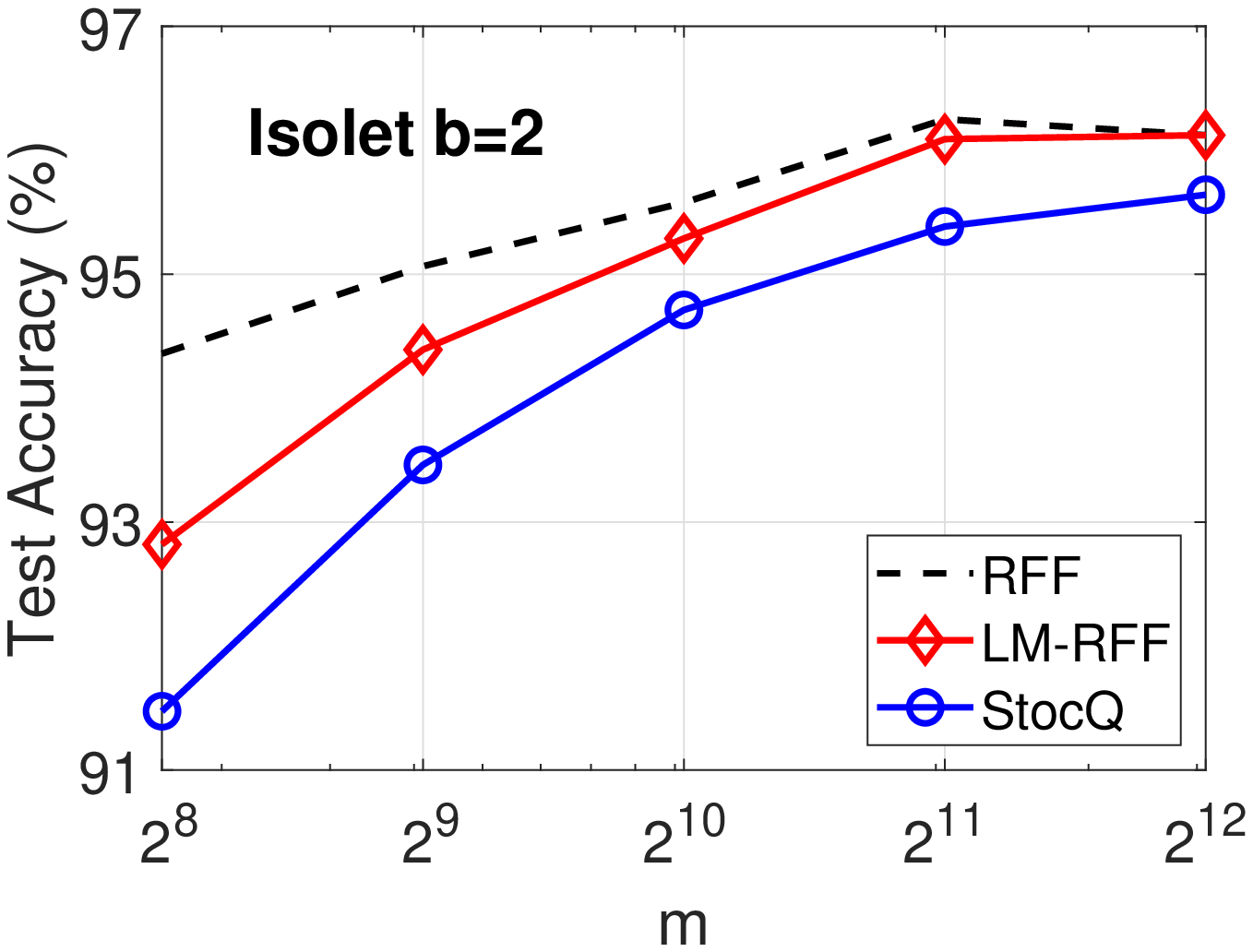}\hspace{-0.1in}
        \includegraphics[width=1.72in]{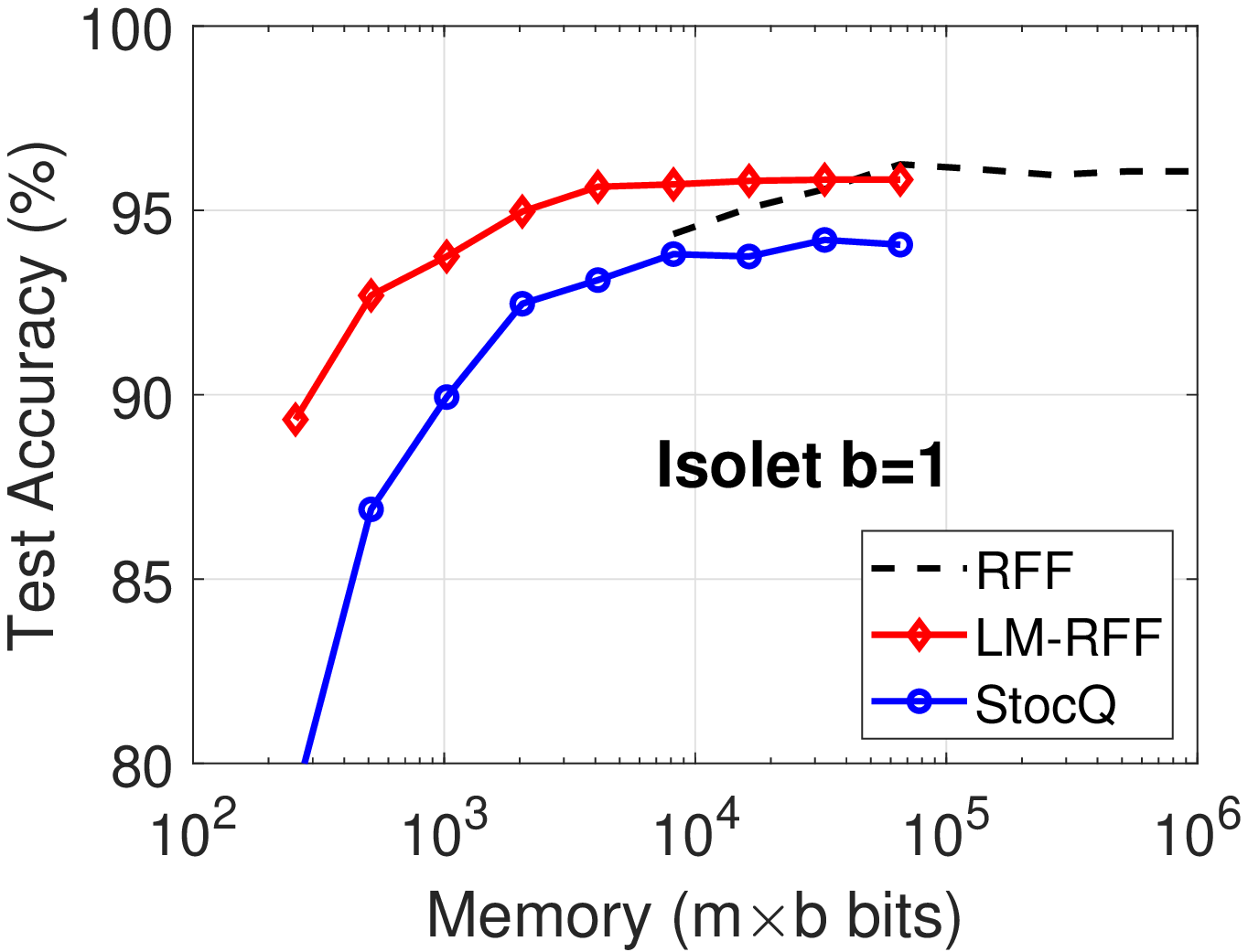}\hspace{-0.1in}
        \includegraphics[width=1.72in]{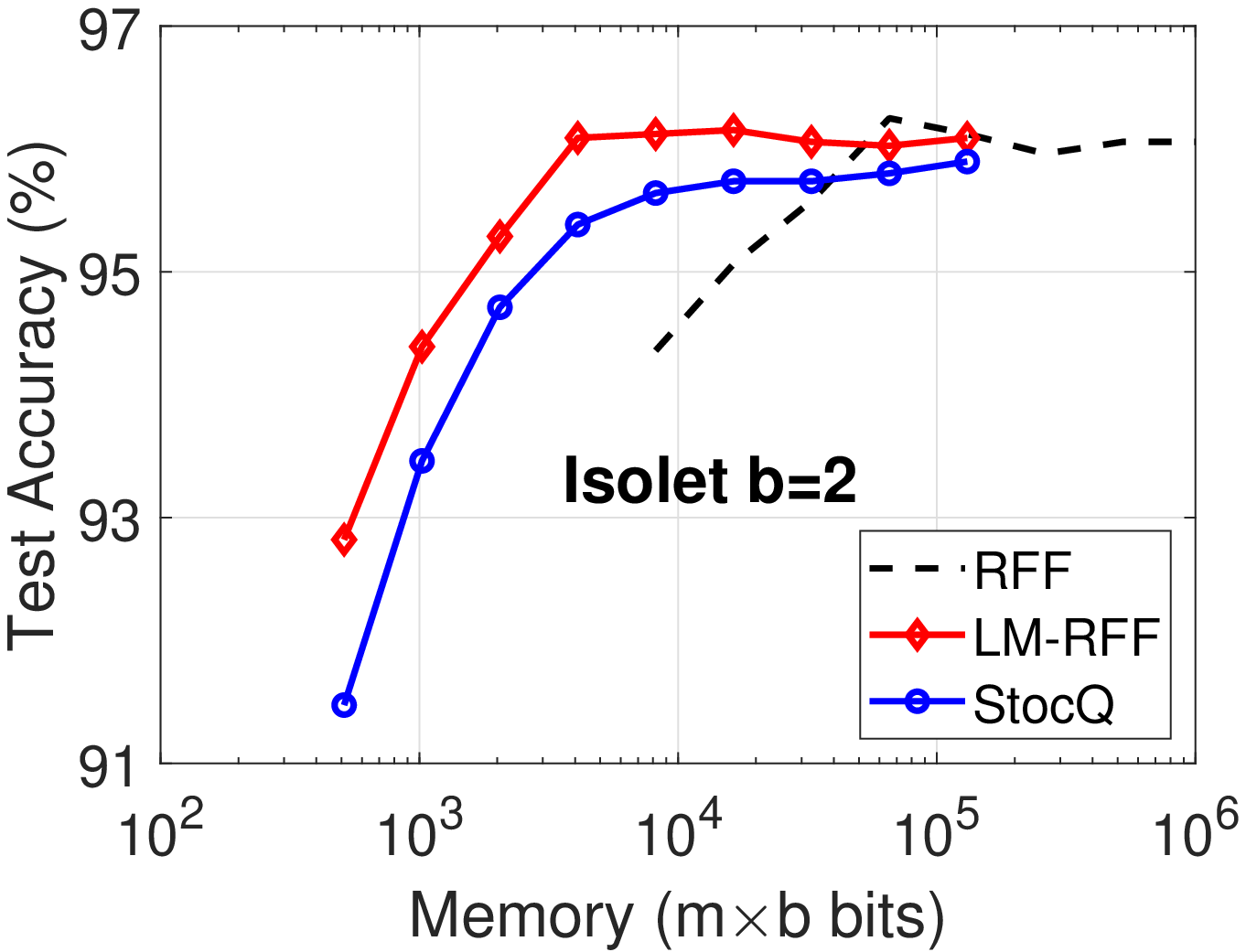}
        }
        \mbox{\hspace{-0.12in}
        \includegraphics[width=1.72in]{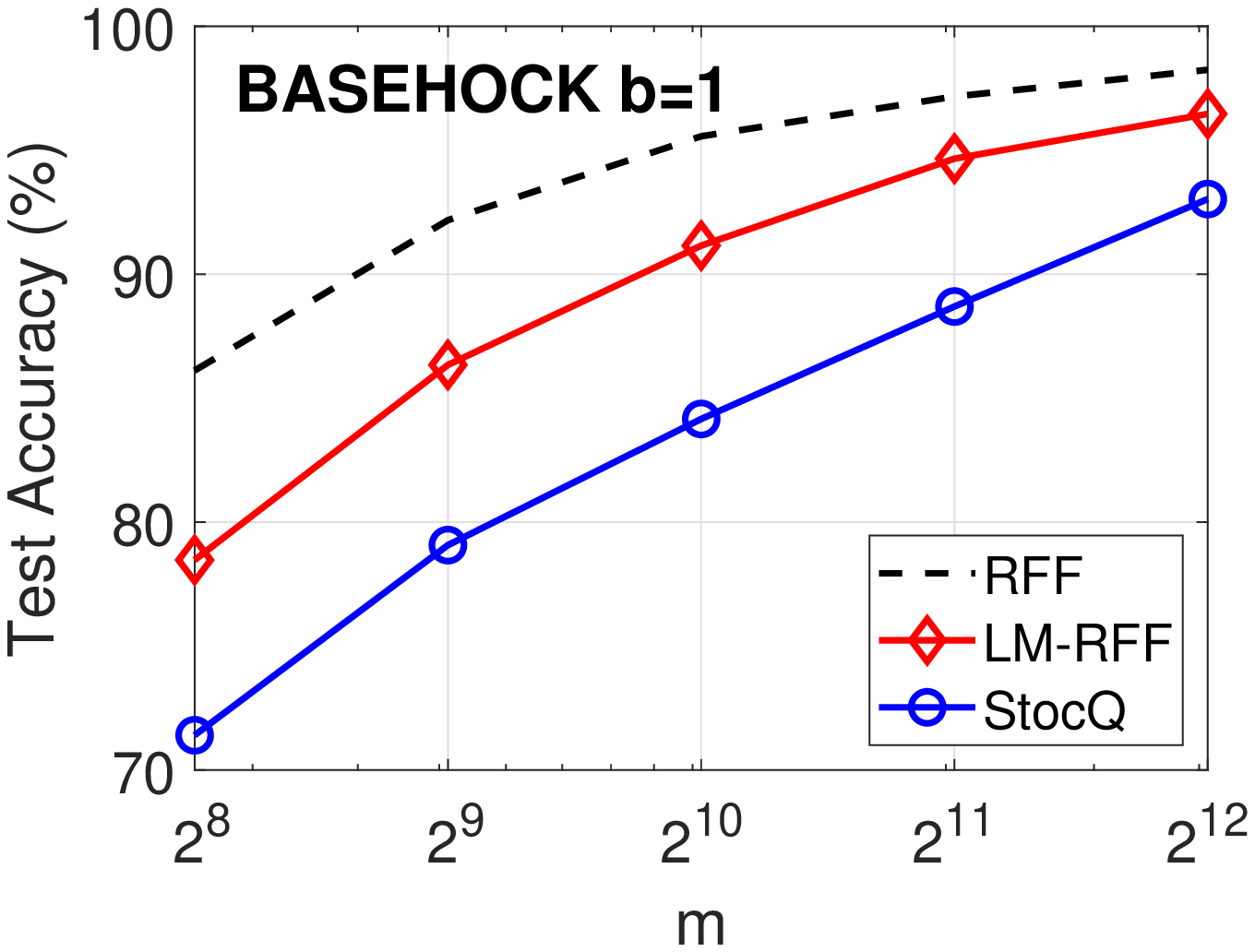}\hspace{-0.1in}
        \includegraphics[width=1.72in]{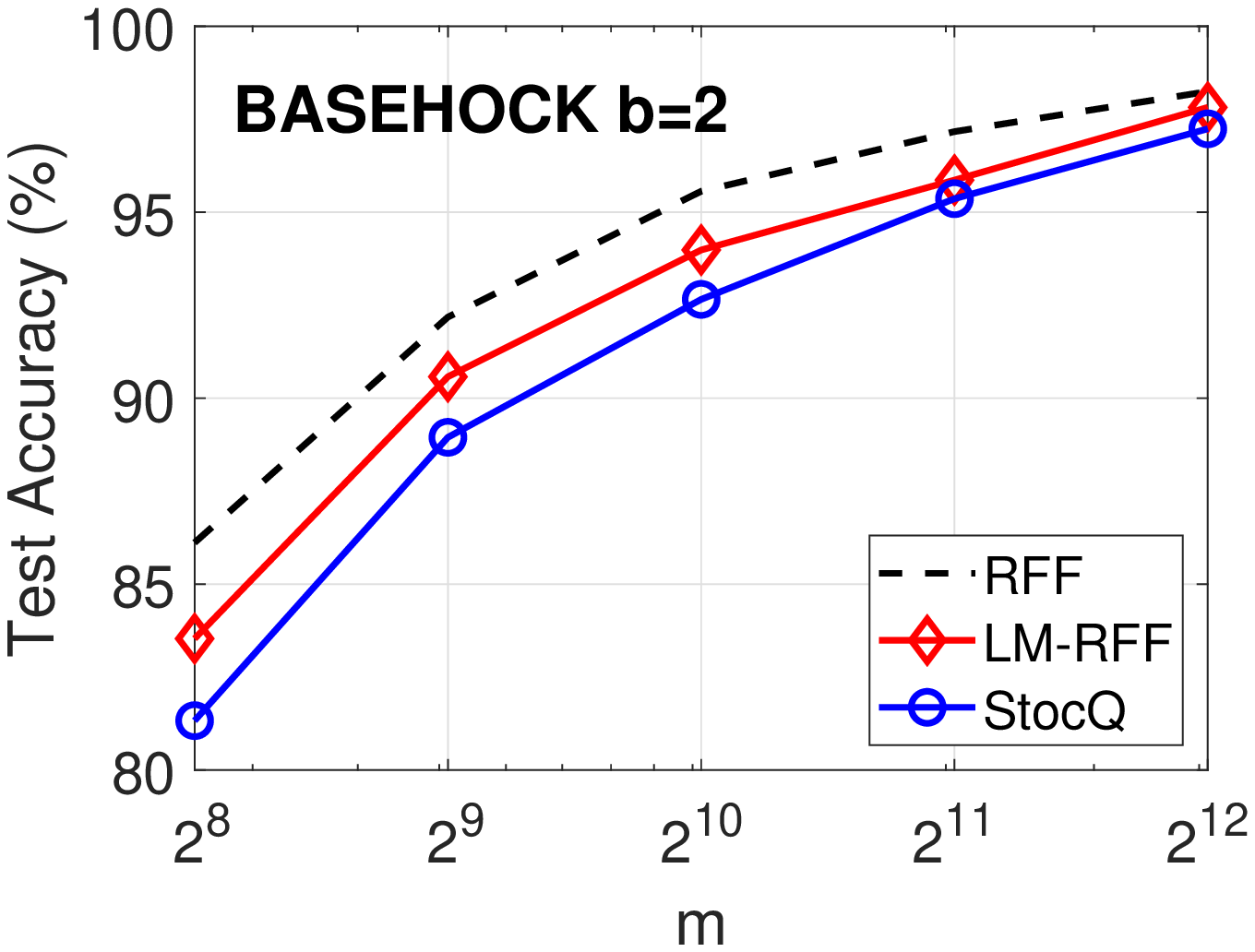}\hspace{-0.1in}
        \includegraphics[width=1.72in]{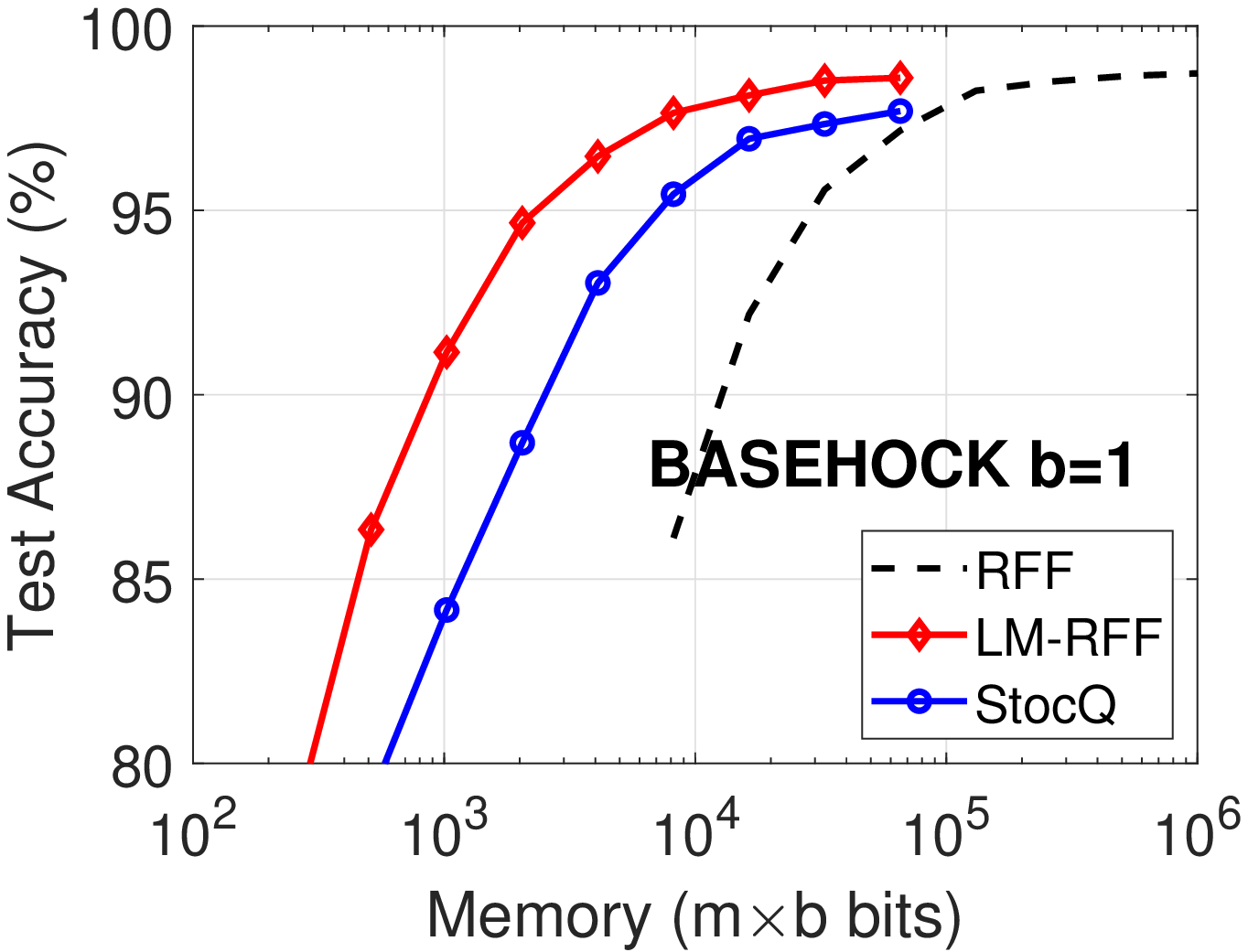}\hspace{-0.1in}
        \includegraphics[width=1.72in]{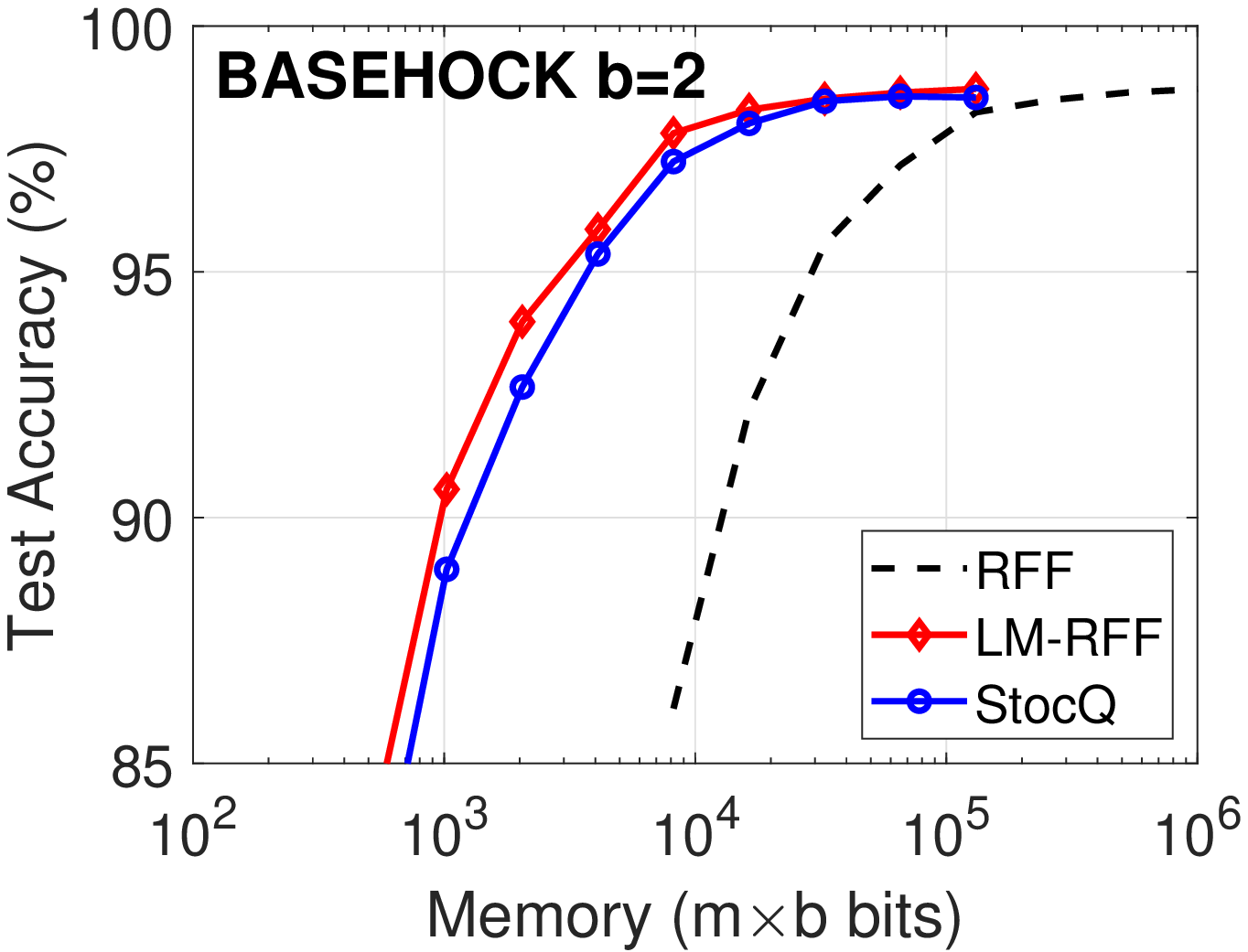}
        }
        \mbox{\hspace{-0.12in}
        \includegraphics[width=1.72in]{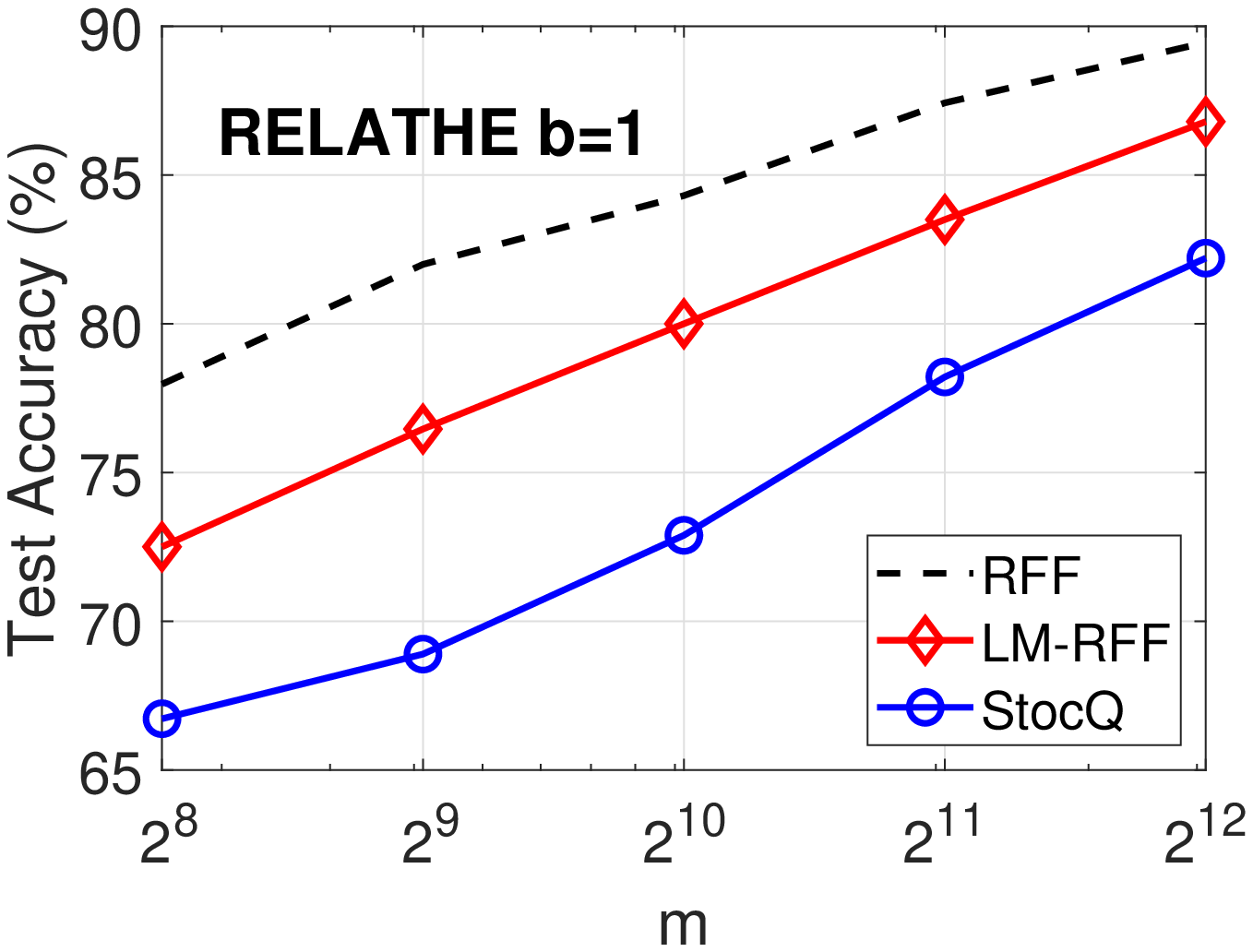}\hspace{-0.1in}
        \includegraphics[width=1.72in]{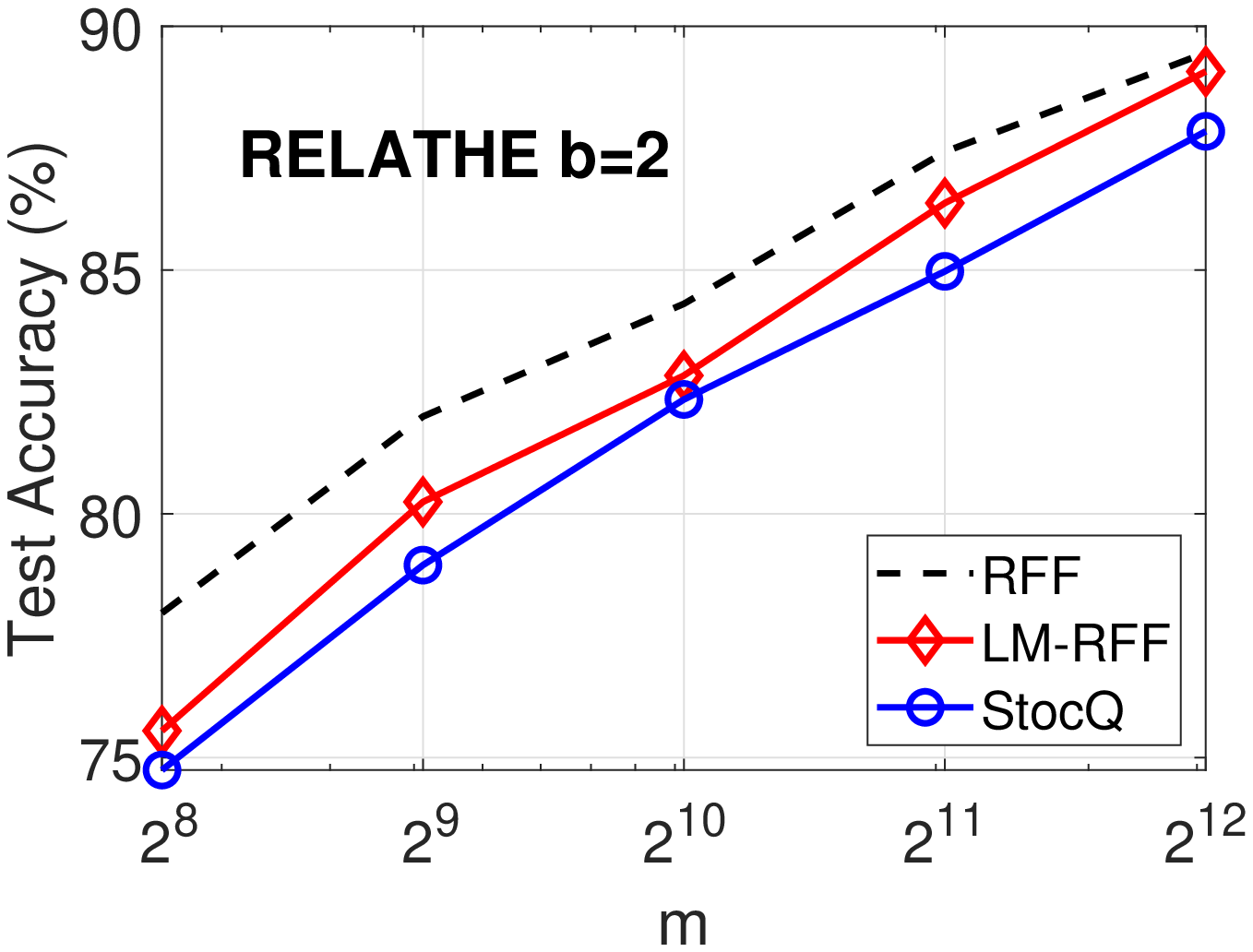}\hspace{-0.1in}
        \includegraphics[width=1.72in]{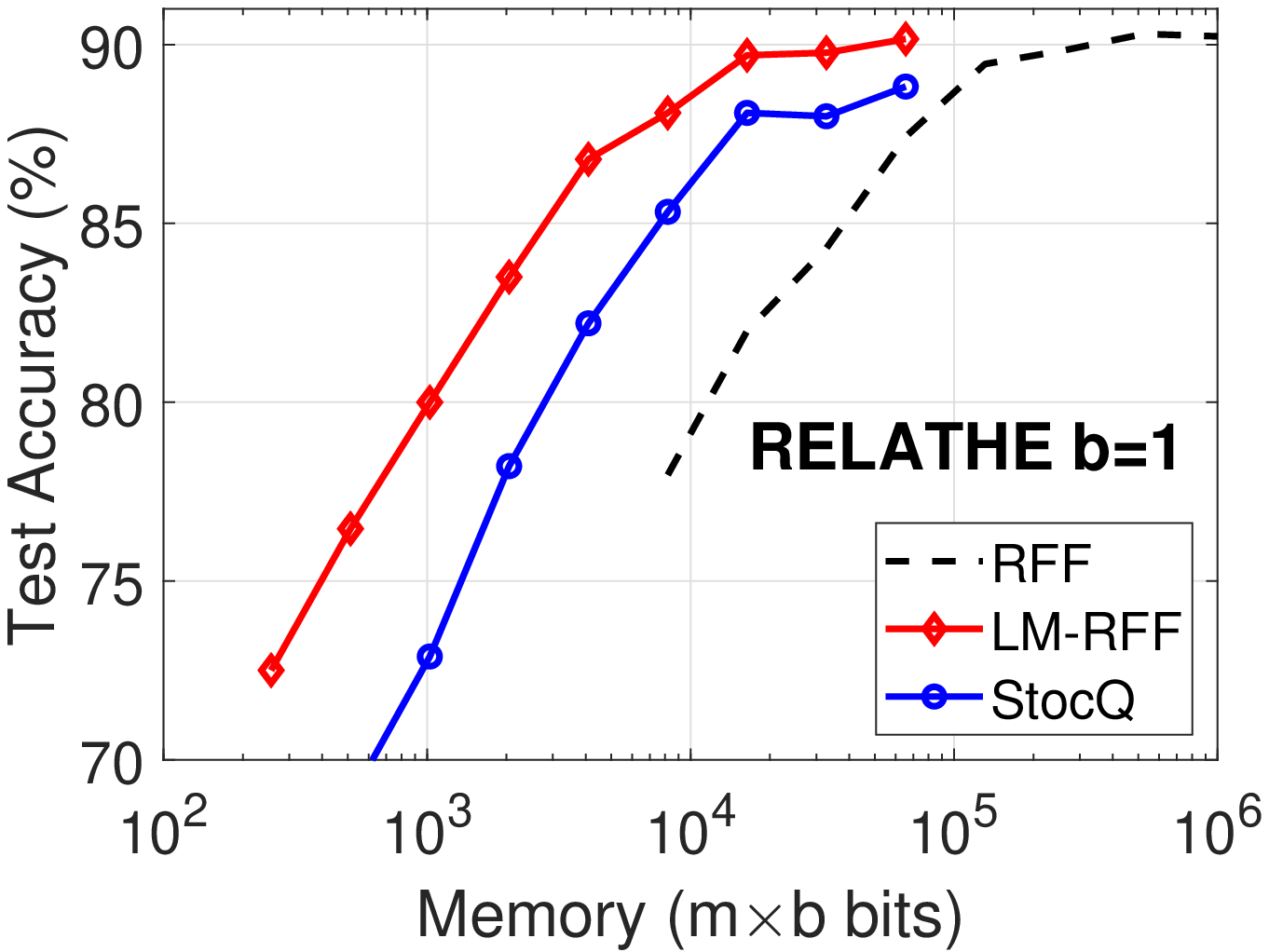}\hspace{-0.1in}
        \includegraphics[width=1.72in]{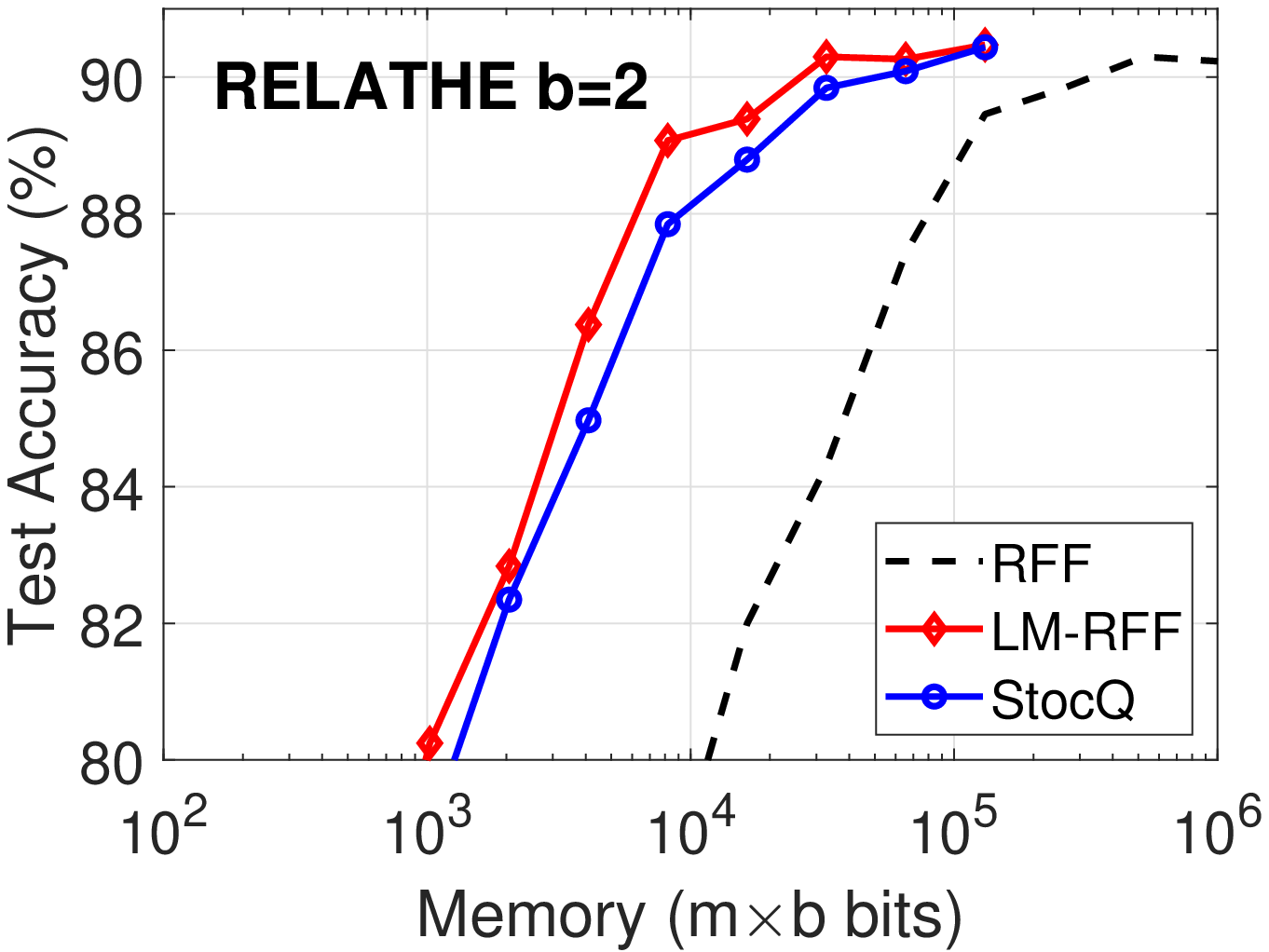}
        }
        \mbox{\hspace{-0.12in}
        \includegraphics[width=1.72in]{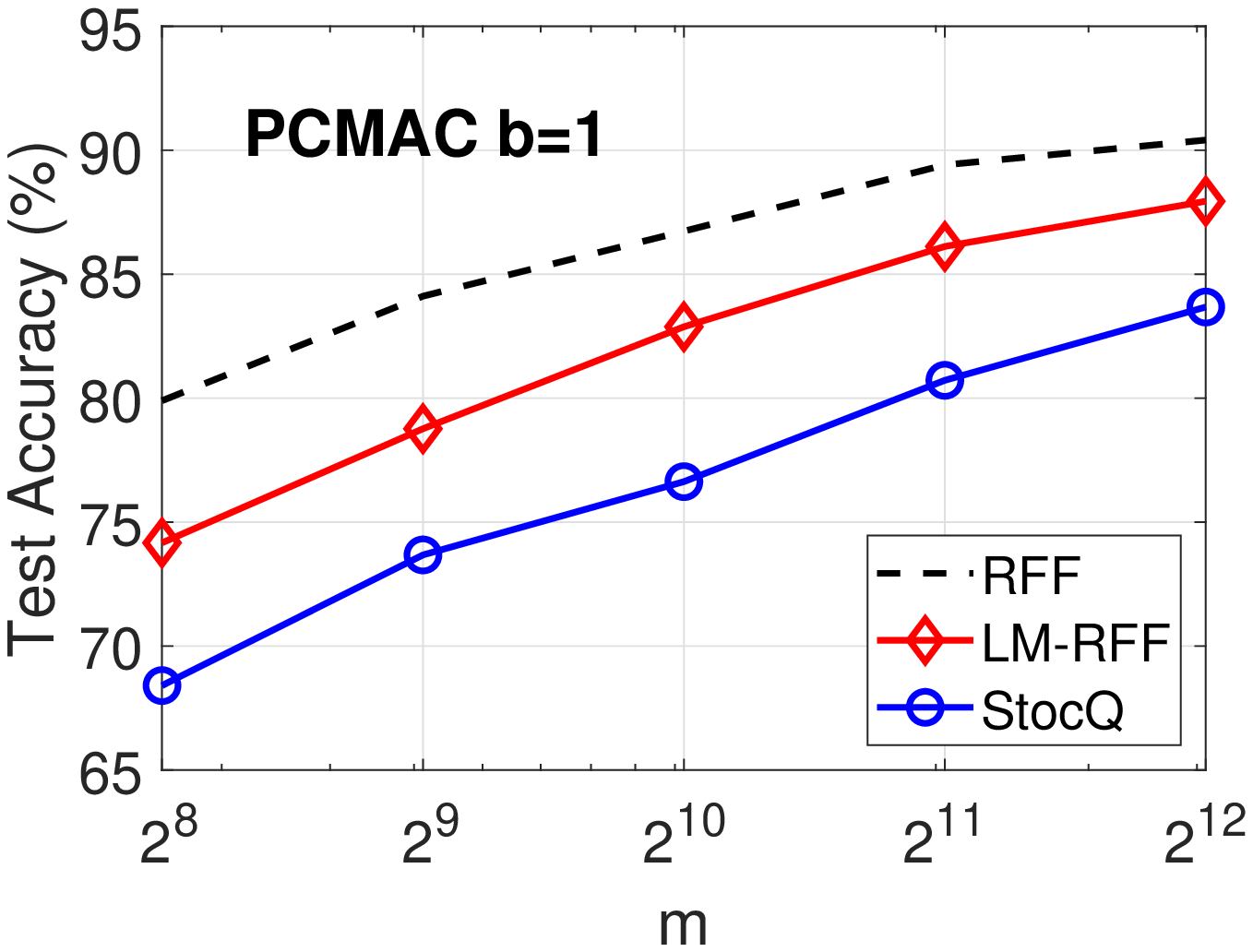}\hspace{-0.1in}
        \includegraphics[width=1.72in]{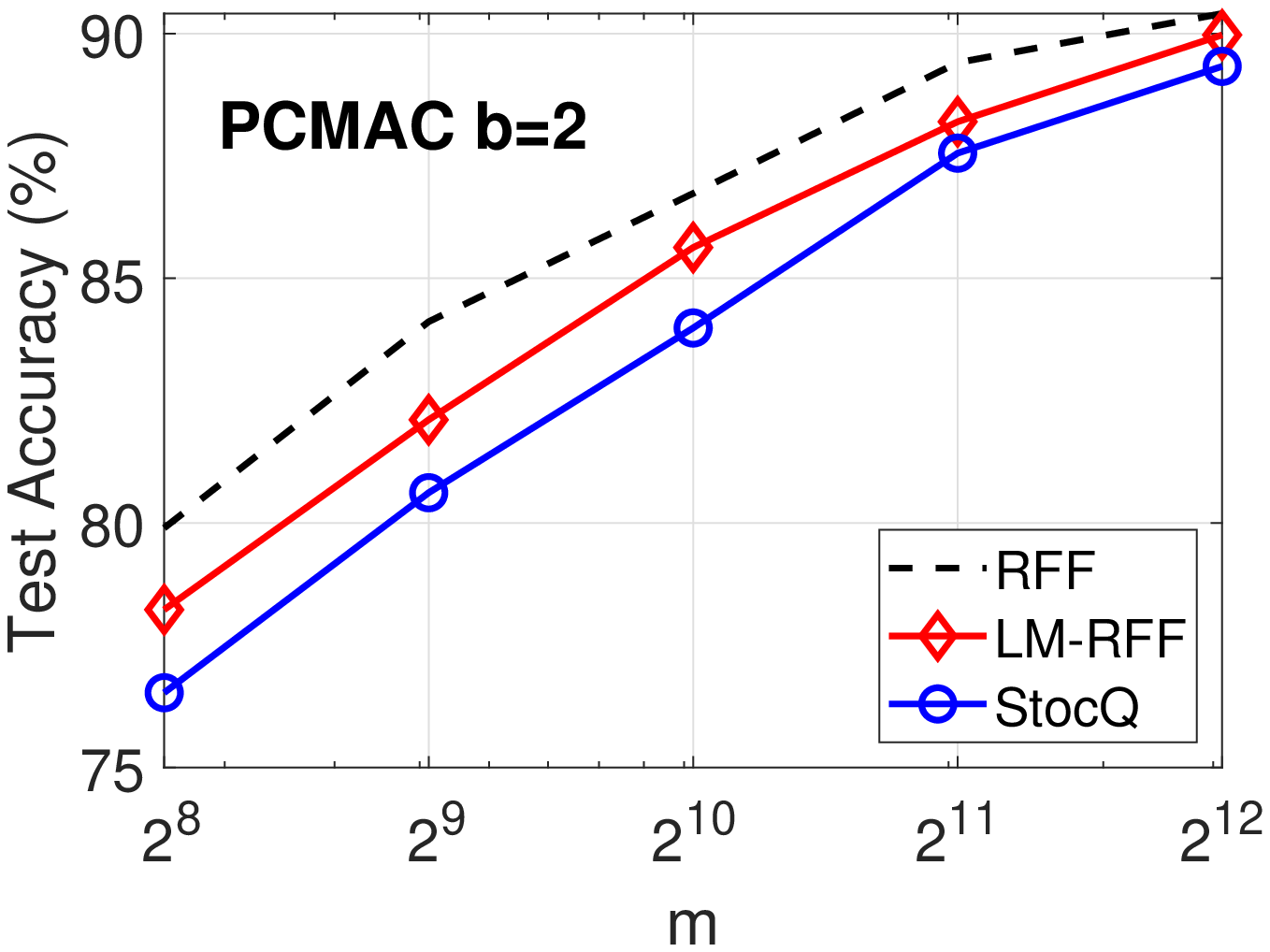}\hspace{-0.1in}
        \includegraphics[width=1.72in]{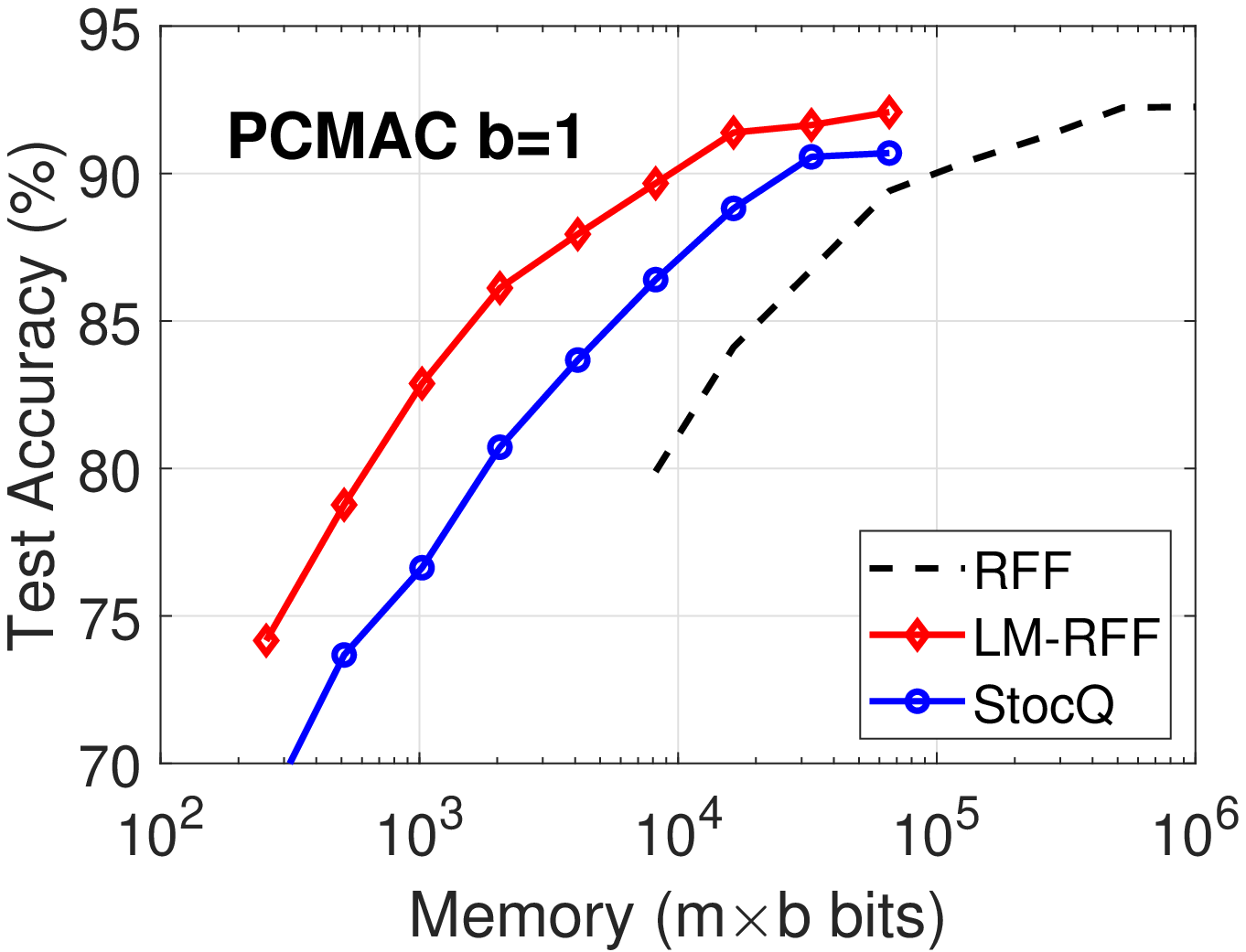}\hspace{-0.1in}
        \includegraphics[width=1.72in]{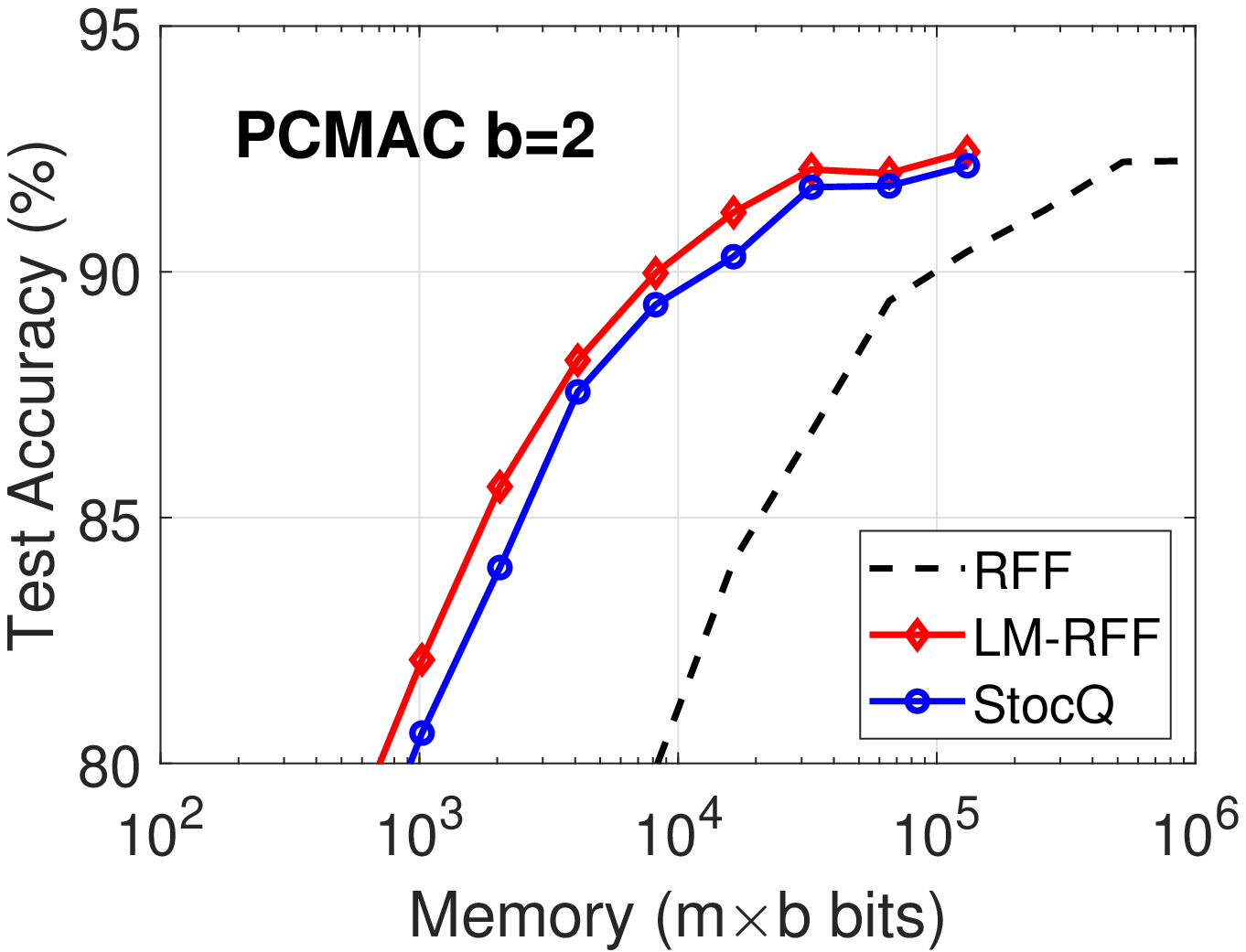}
        }
    \vspace{-0.2in}
	\caption{\textbf{Left two columns:} Test accuracy of kernel SVM using different compression schemes of RFFs vs. number of random features $m$. \textbf{Right two columns:} Test accuracy vs. memory per sample ($m\times b$).}
	\label{fig:SVM_all}
\end{figure}

\subsection{Kernel SVM (KSVM)} \label{sec:KSVM}

For this task, we use four popular public datasets from UCI repository\footnote{\url{https://archive.ics.uci.edu/ml/index.php}}~\citep{UCI} and ASU database\footnote{\url{https://jundongl.github.io/scikit-feature/datasets.html}}~\citep{ASU}. All the data samples are pre-processed by instance normalization, and we randomly split each dataset into 60\% for training and 40\% for testing. For each task and each quantization method, the best test tuned accuracy is reported, averaged over 10 independent runs.

To compare the learning power of different compression schemes, we provide the test accuracy vs. number of RFFs in the left two columns of Figure~\ref{fig:SVM_all}, with $b=1,2$. We observe: 1) LM-RFF substantially outperforms StocQ on all datasets with low bits. In particular, 1-bit StocQ performs very poorly, while 1-bit LM-RFF achieves similar accuracy as full-precision RFF; 2) On all datasets, LM-RFF with $b=2$ already approaches the accuracy of full-precision RFF with moderate $m\approx 4000$, indicating the superior learning capacity of LM-RFF under deep feature compression. As  expected, with larger $b$, the performance of StocQ approaches that of LM-RFF. In particular, when $b=4$, LM-RFF and StocQ perform  similarly on those~datasets.

To characterize the memory efficiency\footnote{For simplicity, we mainly consider the memory cost for (quantized) RFF storage, which dominates in large-scale learning.}, note that under $b$-bit compression, each data sample requires $m\times b$ bits in total for storage. If we assume that each full-precision RFF  is represented by 32 bits, then the storage cost per sample for full-precision RFF is $32m$. This allows us to plot the test accuracy against the total memory cost per sample, as shown in the right two columns of Figure~\ref{fig:SVM_all}. A curve near upper-left corner is more desirable, which means that the method requires less memory to achieve some certain test accuracy.
\begin{itemize}
    \item We observe significant advantage of LM-RFF over full-precision RFF in terms of memory efficiency. For example, to achieve $95\%$ accuracy on \texttt{Isolet}, LM-RFF (both 1-bit and 2-bit) requires $\approx 2000$ bits per sample, while RFF needs $\approx 18000$ bits, leading to a 9x compression ratio. Similar comparison holds for all datasets, and in general the compression ratio of LM-RFF is around 10x.

    \item When compared with StocQ, we see consistently advantage of LM-RFF in memory cost. In general, LM-RFF can further improve the compression ratio of StocQ by 2x$\sim$4x. Additionally, LM-RFF typically requires fewer-bit quantizers (smaller $b$) than StocQ to achieve satisfactory accuracy.
\end{itemize}

\subsection{Kernel Ridge Regression (KRR)}  \label{sec:KRR}

For kernel ridge regression (KRR), we use a synthetic dataset admitting high non-linearity. Precisely, each data sample $u\in \mathbb R^{10}$ is drawn from i.i.d. $N(0,1)$. We generate the response by $y_i=\sum_{p=1}^3 \beta_p x^p+\epsilon$, where $\beta_1=[1,2,...,10]$, $\beta_2=[1,1,...,1]$, $\beta_3$ and $\epsilon$ also follow i.i.d. $N(0,1)$. We simulate $40,000$ independent samples for training and $10,000$ for testing.

We summarize KRR results in Figure~\ref{fig:KRR}. Again, with same $b$ and number of RFFs, LM-RFF consistently beats StocQ especially with low bits. We see that 1-bit LM-RFF even outperforms 2-bit StocQ, and when $b=4$, we still observe considerable advantage of LM-RFF over StocQ. In the second row, we present the memory efficiency comparison. Note that, due to high-order terms in the true model, the test MSE of linear kernel is $20.8$, while learning with full-precision RFF significantly reduces it to $3.5$. With largest memory budget that is tested, 1-bit and 2-bit LM-RFF yield $5.9$ and $4.1$ test MSE respectively, which are already quite close to 3.5, while for 1-bit and 2-bit StocQ, the test losses are $14.5$ and $5.0$ respectively, much worse than those of LM-RFF. We again see significant storage saving of LM-RFF. For instance, to reach the same test MSE (e.g., 10), the compression ratio is about 5x for $b=4$ compared with full-precision RFF. Moreover, the advantage of LM-RFF over StocQ is also significant for this regression problem.

\begin{figure}[t]
    \begin{center}
        \mbox{
        \includegraphics[width=2.2in]{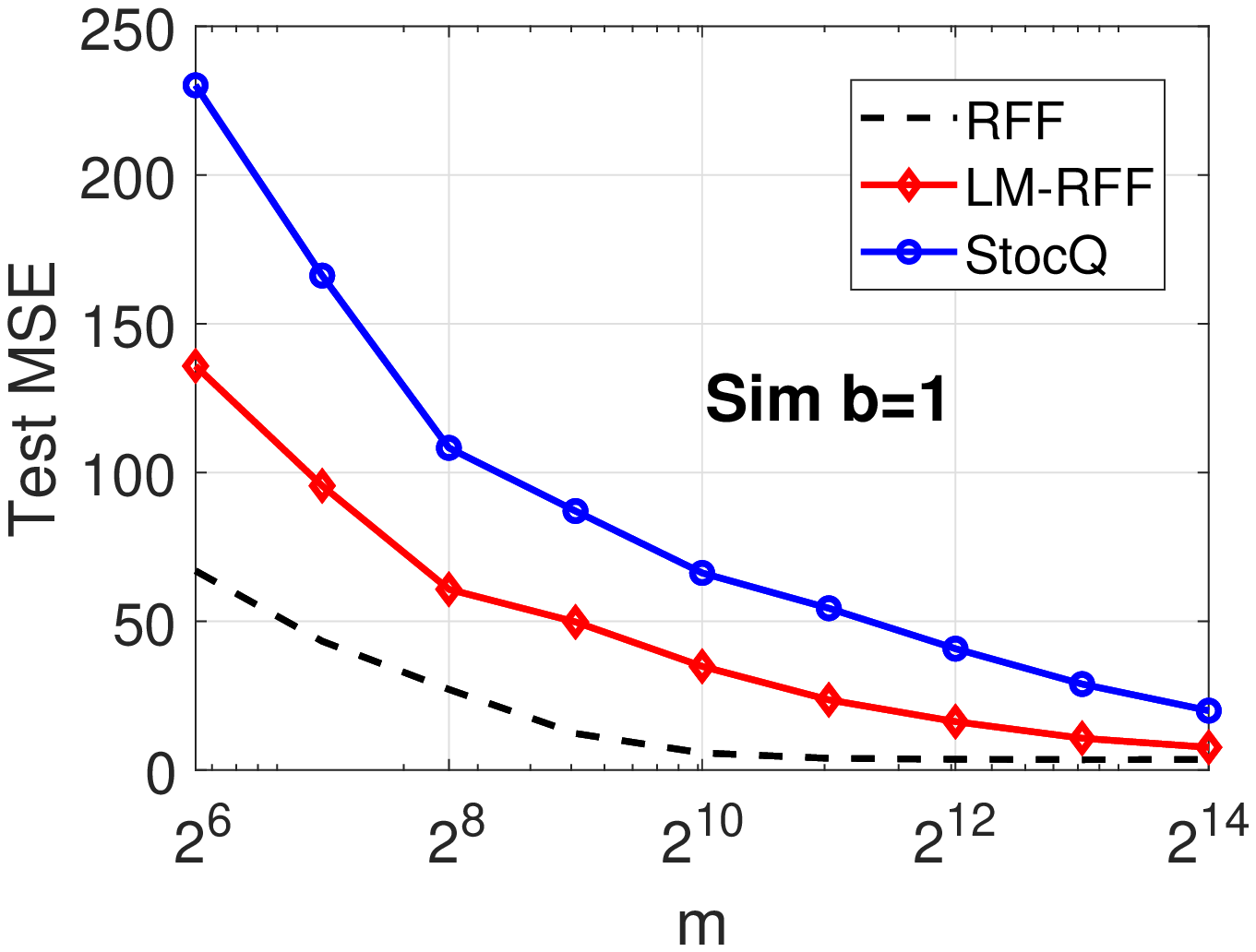}
        \includegraphics[width=2.2in]{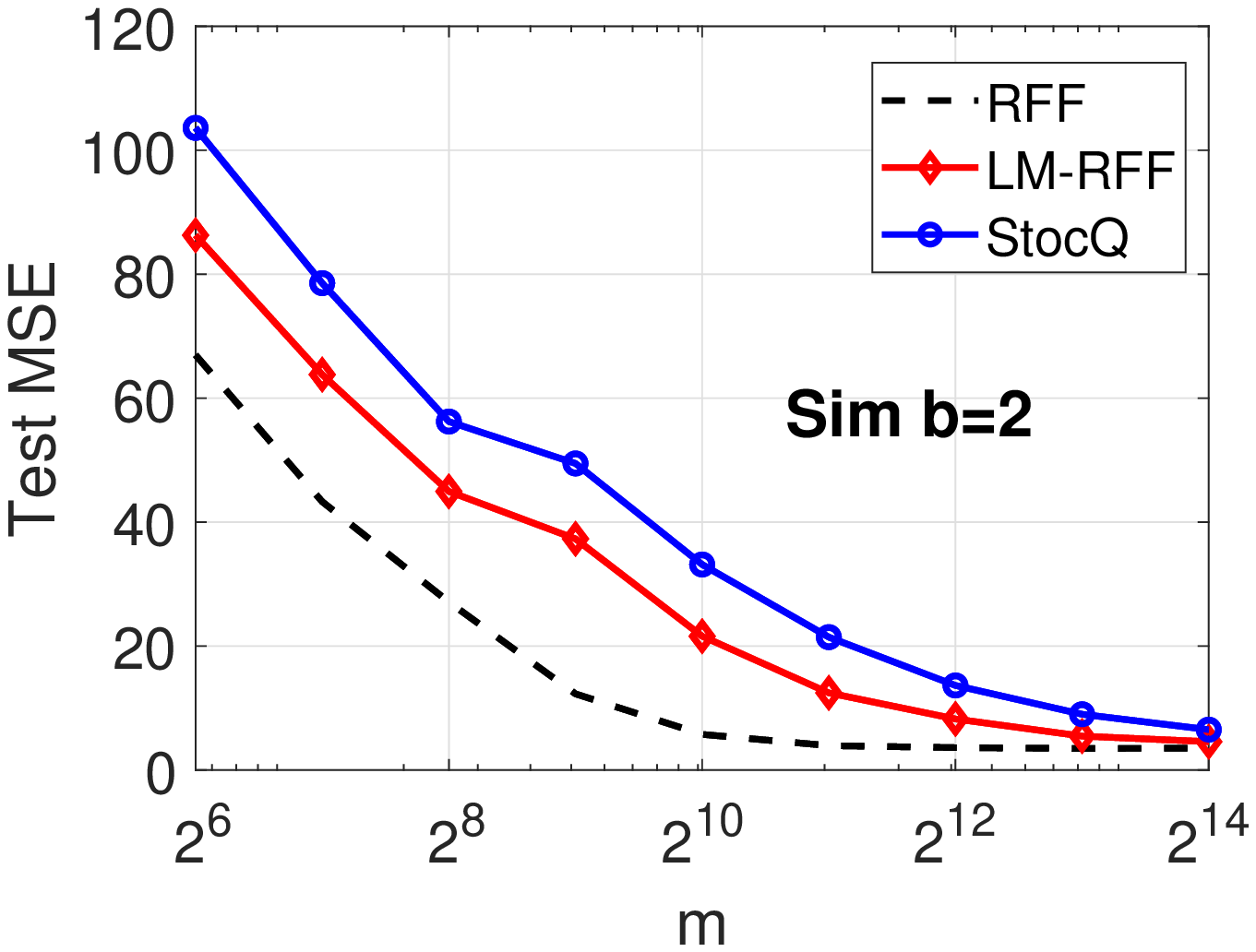}
        \includegraphics[width=2.2in]{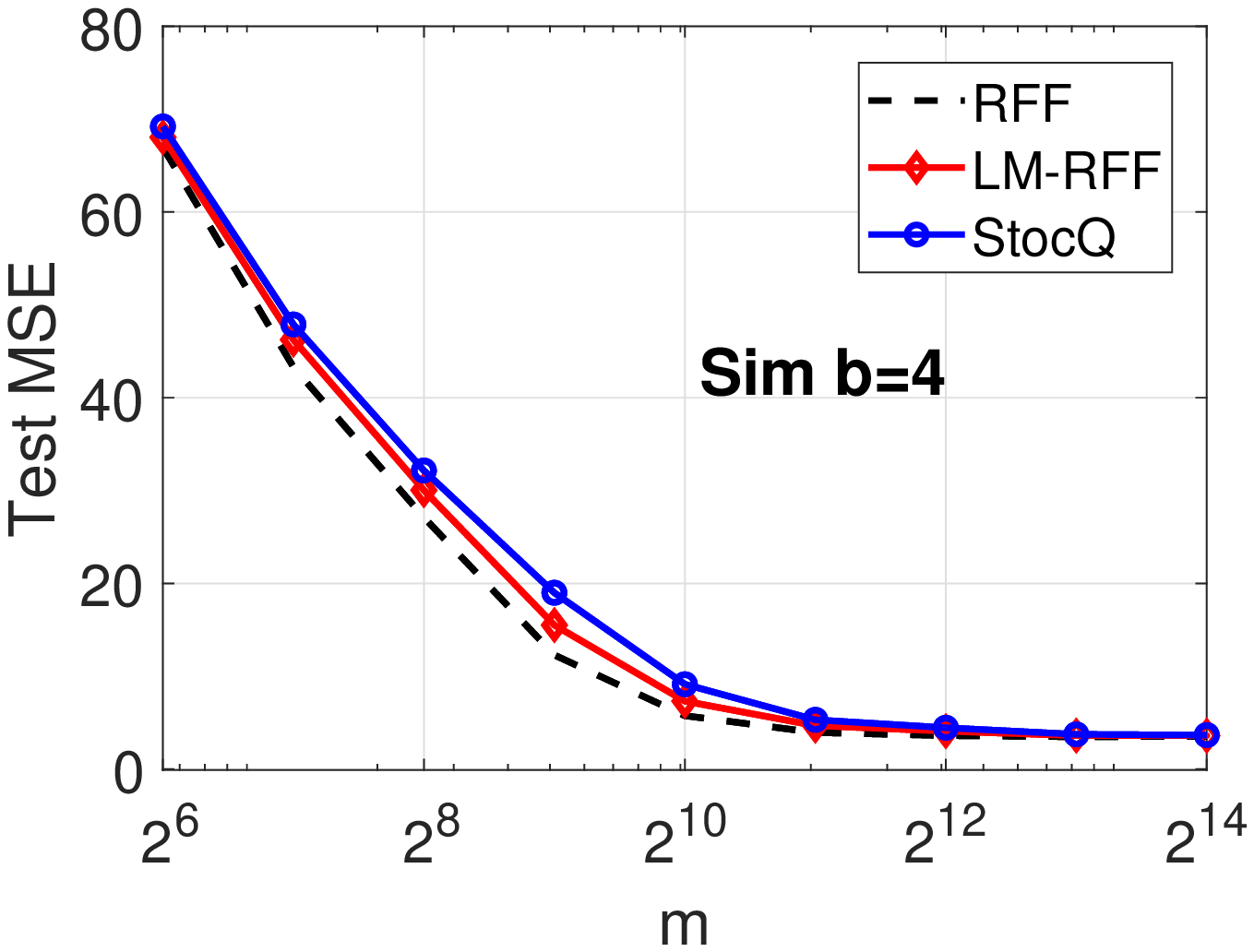}
        }
        \mbox{
        \includegraphics[width=2.2in]{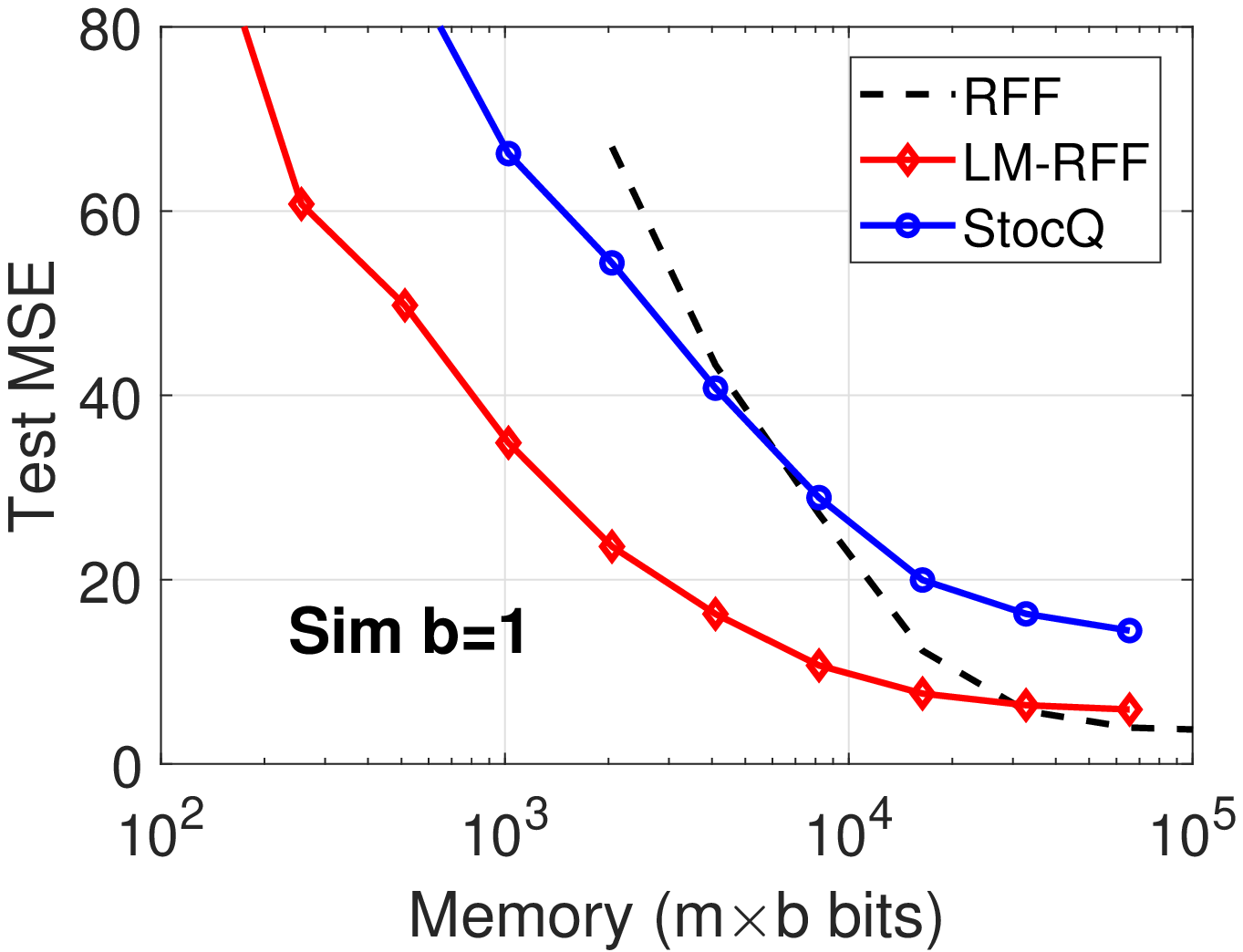}
        \includegraphics[width=2.2in]{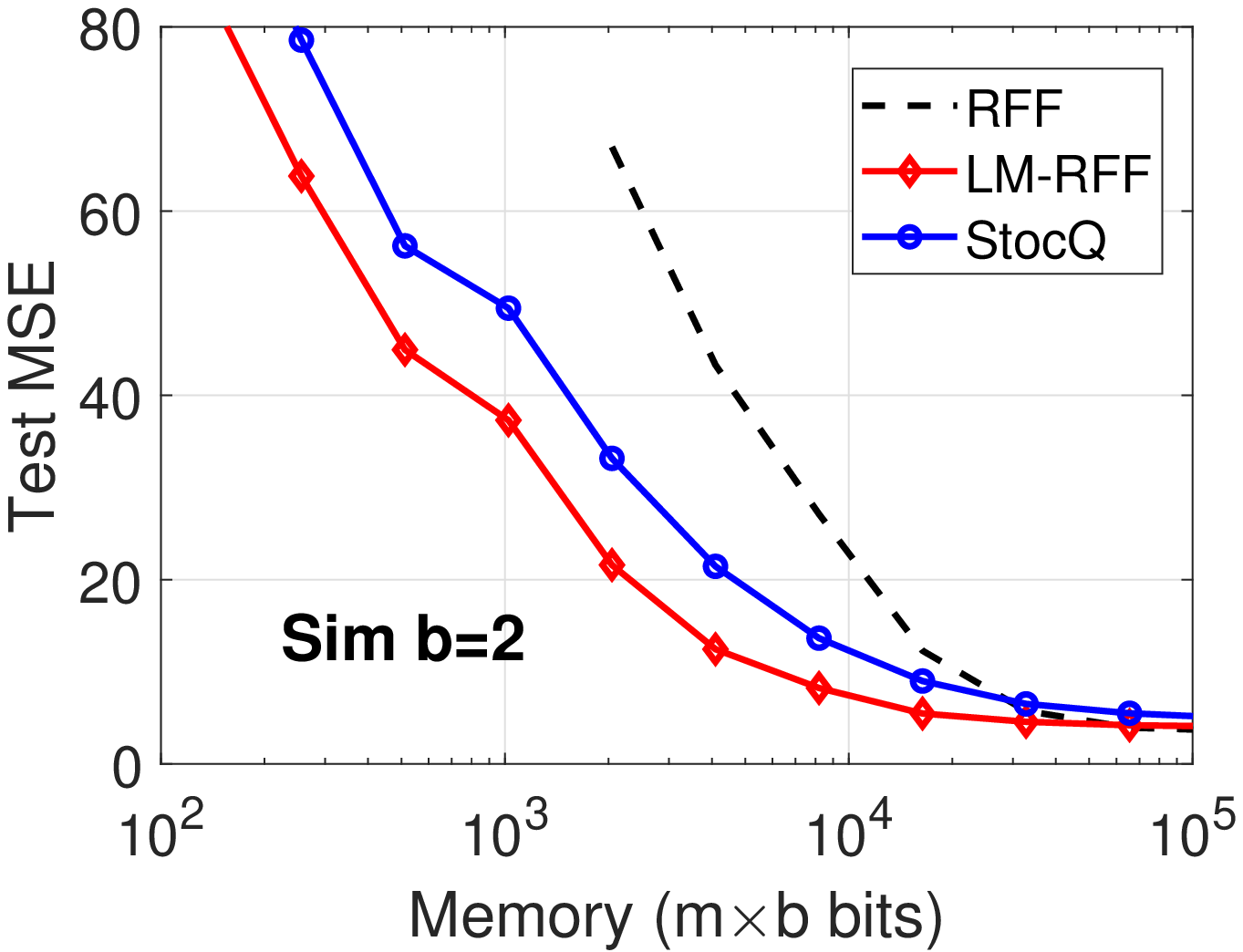}
        \includegraphics[width=2.2in]{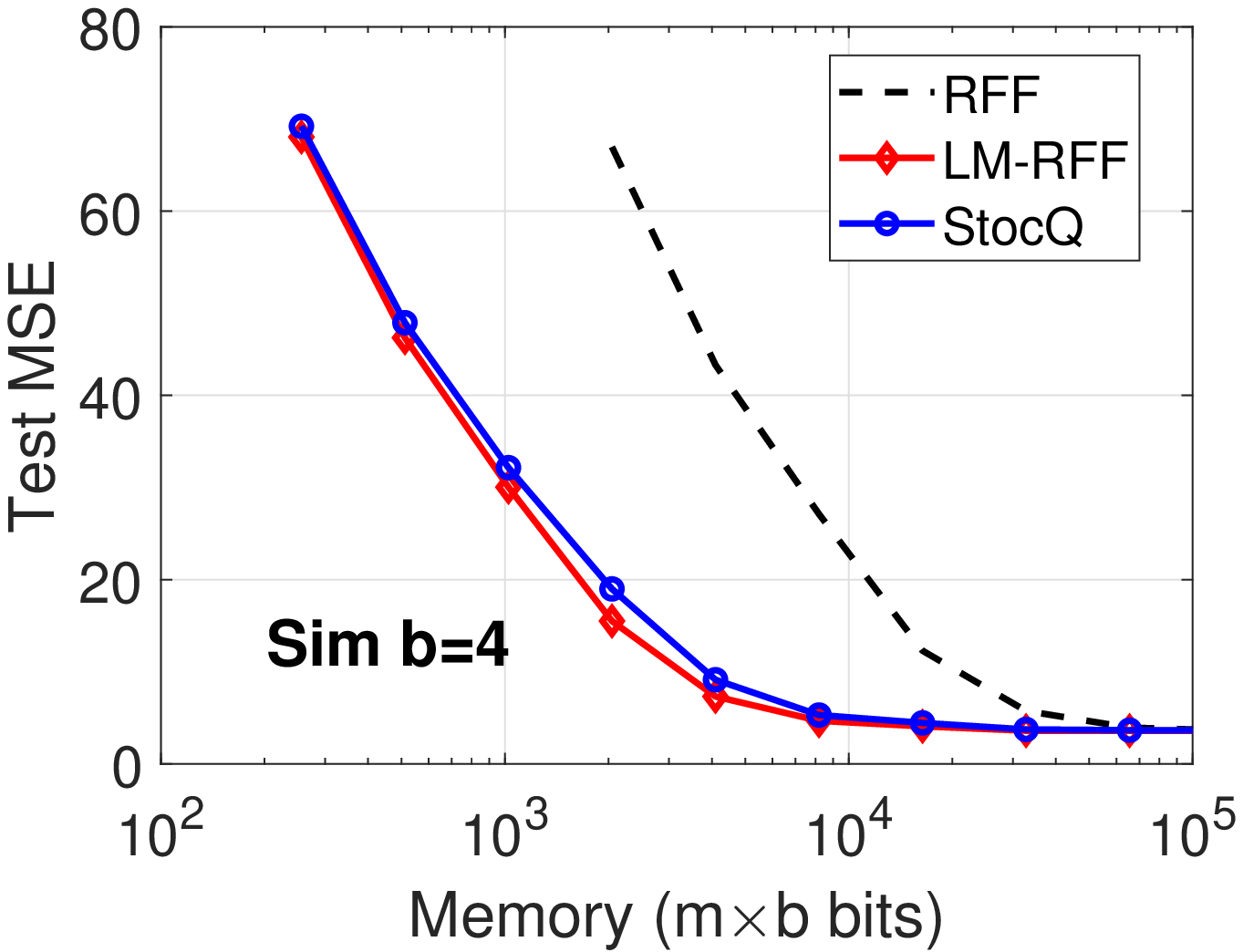}
        }
    \end{center}
	\caption{Kernel ridge regression: test MSE of KRR, StocQ vs. LM-RFF. \textbf{Upper panels:} Test MSE vs. number of random features $m$. \textbf{bottom panels} Test MSE vs. memory per sample ($m\times b$ bits).}
	\label{fig:KRR}
\end{figure}

\subsection{Scale-invariant Kernel Approximation Error} \label{sec:approximation error}

Recall the notation $U=[u_1,...,u_n]^T$ as the data matrix. Let $\mathcal K$ be the $n\times n$ Gaussian kernel matrix, with $\mathcal K_{ij}=K(u_i,u_j)$. Denote $\hat{\mathcal K}$ as the estimated kernel matrix by an approximation algorithm. Kernel Approximation Error (KAE) has been shown to play an important role in the generalization of learning with random features, including the norms~\citep{Proc:Cortes_AISTATS10,Proc:Gitten_ICML13,Proc:Sutherland_UAI15} of $\hat{\mathcal K}-\mathcal K$ and spectral approximations~\citep{Proc:Bach_COLT13,Proc:Alaoui_NIPS15,Proc:Avron_ICML17,Proc:Zhang_AISTATS19}. We investigate the KAEs to better justify the impressive generalization ability of LM-RFF from a theoretical aspect.

\vspace{0.05in}

However, existing KAE metrics are not robust to bias. Consider $\mathbb E[\hat{\mathcal K}]=\beta\mathcal K$ with some $\beta>0$. Obviously, learning with $\beta\mathcal K$ is equivalent to learning with $\mathcal K$ for kernel-distance based models like KSVM and KRR, since with proper scaling of model parameters, the objective functions/predictions are invariant of multiplying the input kernel matrix with a scalar. However, traditional KAEs do not generalize to this case. For example, when $\beta=0.1$, the 2-norm error $\Vert 0.1\hat{\mathcal K}-\mathcal K \Vert_2$ would be very large. To make the KAE metrics more robust, we define the scale-invariant KAE metrics as follows.

\vspace{0.1in}

\begin{definition}[Scale-Invariant KAE] \label{def:error_metric}
Let $\mathcal K$ be a kernel matrix and $\hat{\mathcal K}$ be its randomized approximation. We define
\begin{align*}
    &\Vert \hat{\mathcal K}-\mathcal K \Vert^*_2 = \min_{\beta>0} \Vert \beta\hat{\mathcal K}-\mathcal K \Vert_2,\ \ \Vert \hat{\mathcal K}-\mathcal K \Vert^*_F = \min_{\beta>0} \Vert \beta\hat{\mathcal K}-\mathcal K \Vert_F.
\end{align*}
Denote the minimizers as $\beta_2^*$ and $\beta_F^*$, respectively. Define
\begin{align*}
    &(\delta_1^*,\delta_2^*)= \inf_{\substack {\beta\in\{\beta_2^*,\beta_F^*\}\\(\delta_1,\delta_2)\geq 0}}\  \big\{\delta_1,\delta_2:(1-\delta_1)\mathcal K\preccurlyeq \beta \hat{\mathcal K} \preccurlyeq (1+\delta_2)\mathcal K \big\}.
\end{align*}
\end{definition}

\vspace{0.1in}

Our new KAE metrics are more general, adapted to the best scaling factor $\beta_2^*$ or $\beta_F^*$ of the estimated kernel. Since LM-RFF estimators are slightly biased (recall Observations~\ref{obv1} and~\ref{obv3}), Definition~\ref{def:error_metric} is important for appropriately evaluating our proposed LM-RFF kernel estimation approach. In Figure~\ref{fig_approx_error}, we provide scale-invariant $\Vert\cdot \Vert_2^*$, $\Vert\cdot \Vert_F^*$ and $\delta_2^*$ metrics\footnote{\cite{Proc:Zhang_AISTATS19} found that for kernel approximation methods, $\delta_2$ is fairly predictive of the generalisation performance.} on \texttt{Isolet} and \texttt{BASEHOCK} dataset as representatives. As we can see, LM-RFF always has smaller KAEs than StocQ with equal bits. In particular, with extreme 1-bit compression, StocQ has exceedingly large loss due to its large variance, while in many cases the KAEs of 1-bit LM-RFF are already quite small. The KAE comparison well aligns with, and to an extent explains, our experimental results in Section~\ref{sec:KSVM} and Section~\ref{sec:KRR} that 1) LM-RFF consistently outperforms StocQ, and 2) 1-bit StocQ generalizes very poorly. Thus, it provides a general justification of the superior effectiveness of LM-RFF in machine learning.

\begin{figure}[t]
    \begin{center}
        \mbox{
        \includegraphics[width=2.2in]{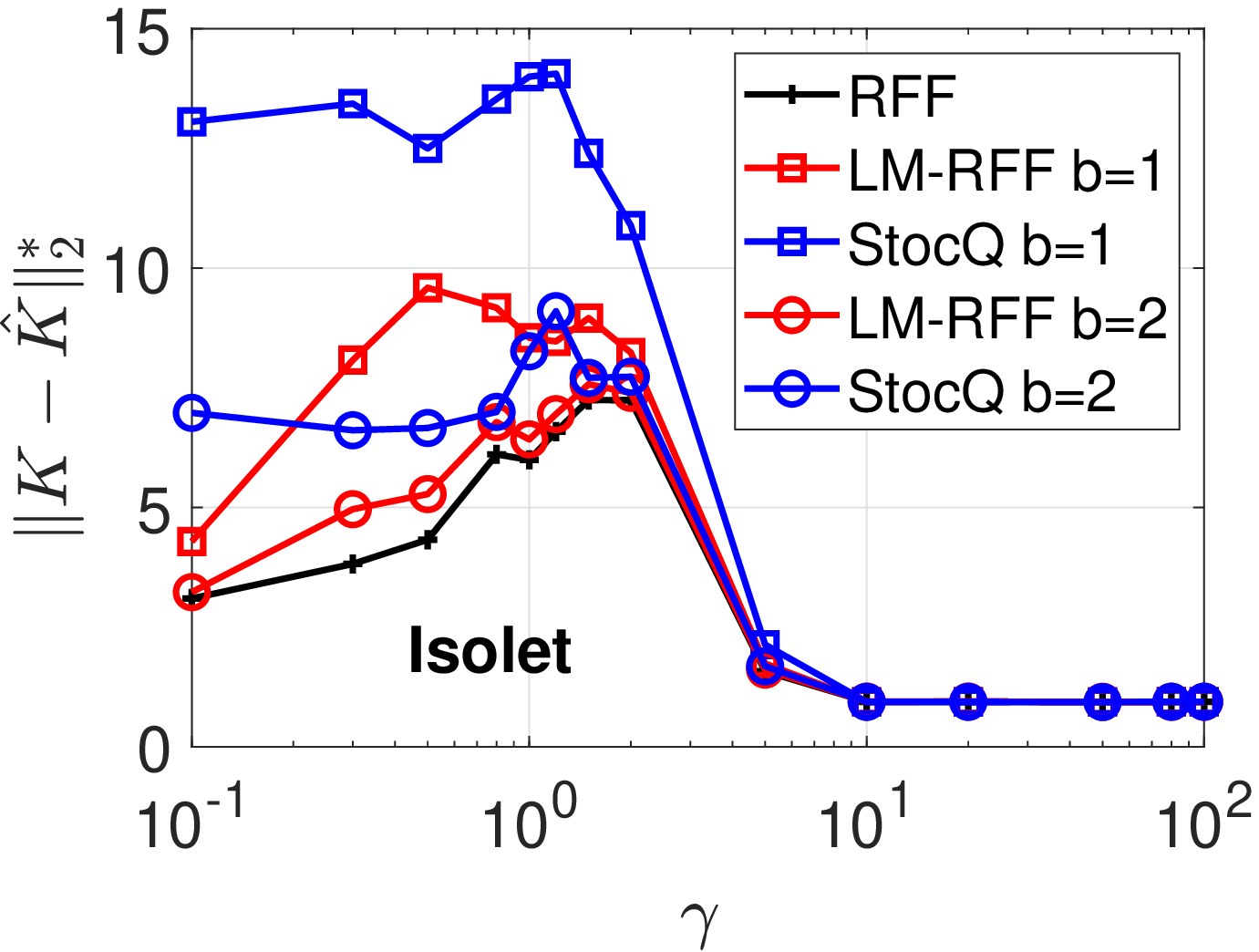}
        \includegraphics[width=2.2in]{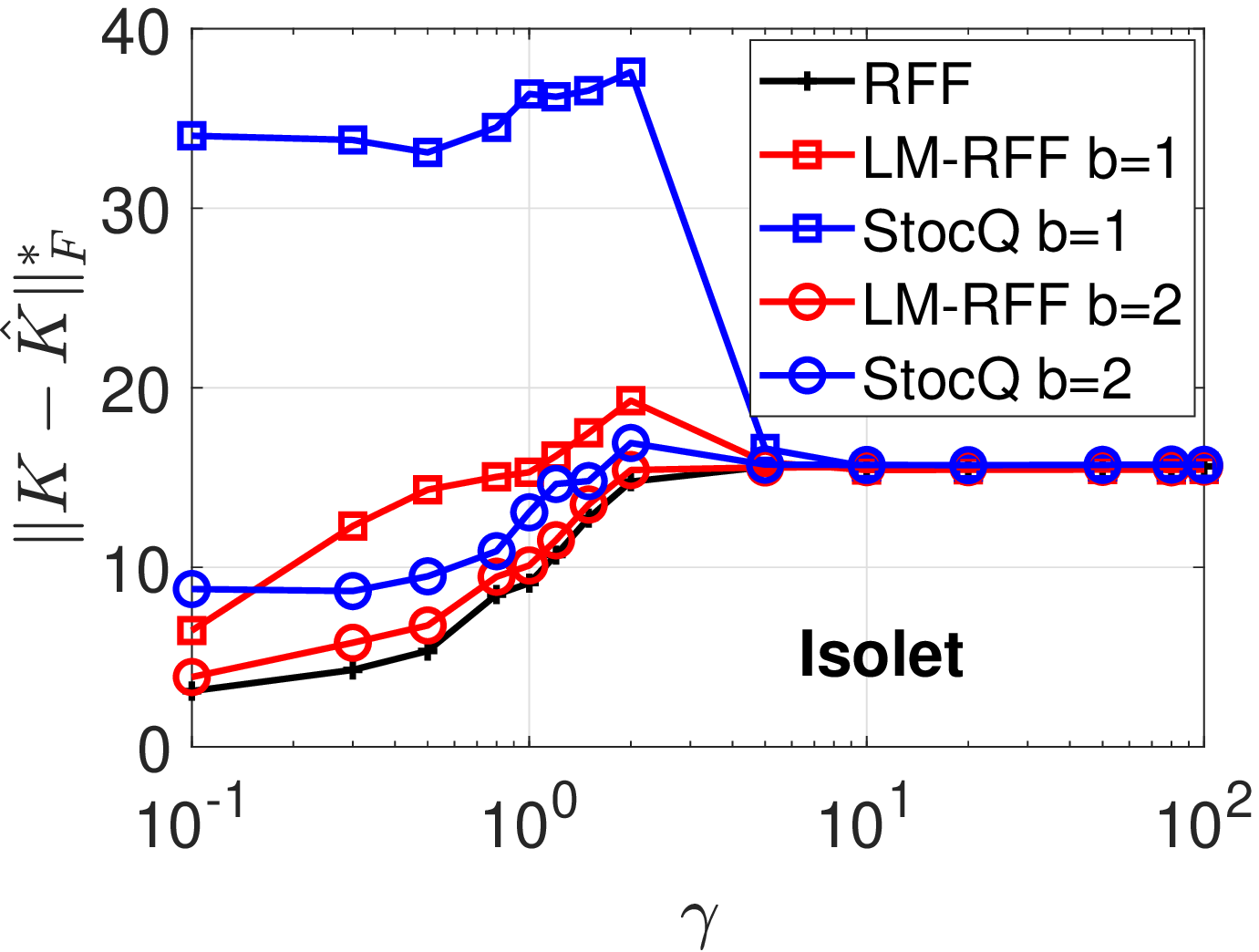}
        \includegraphics[width=2.2in]{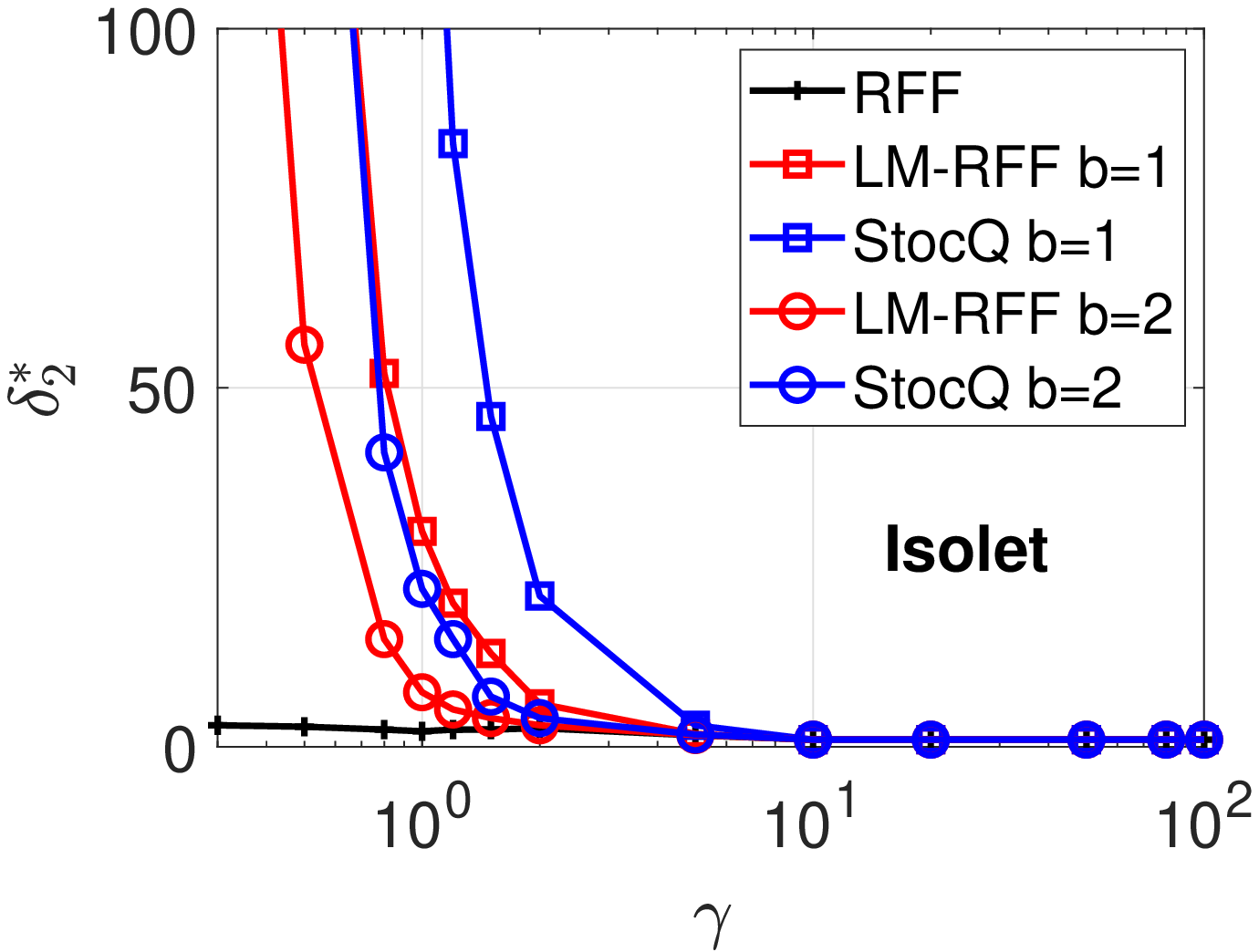}
        }
        \mbox{
        \includegraphics[width=2.2in]{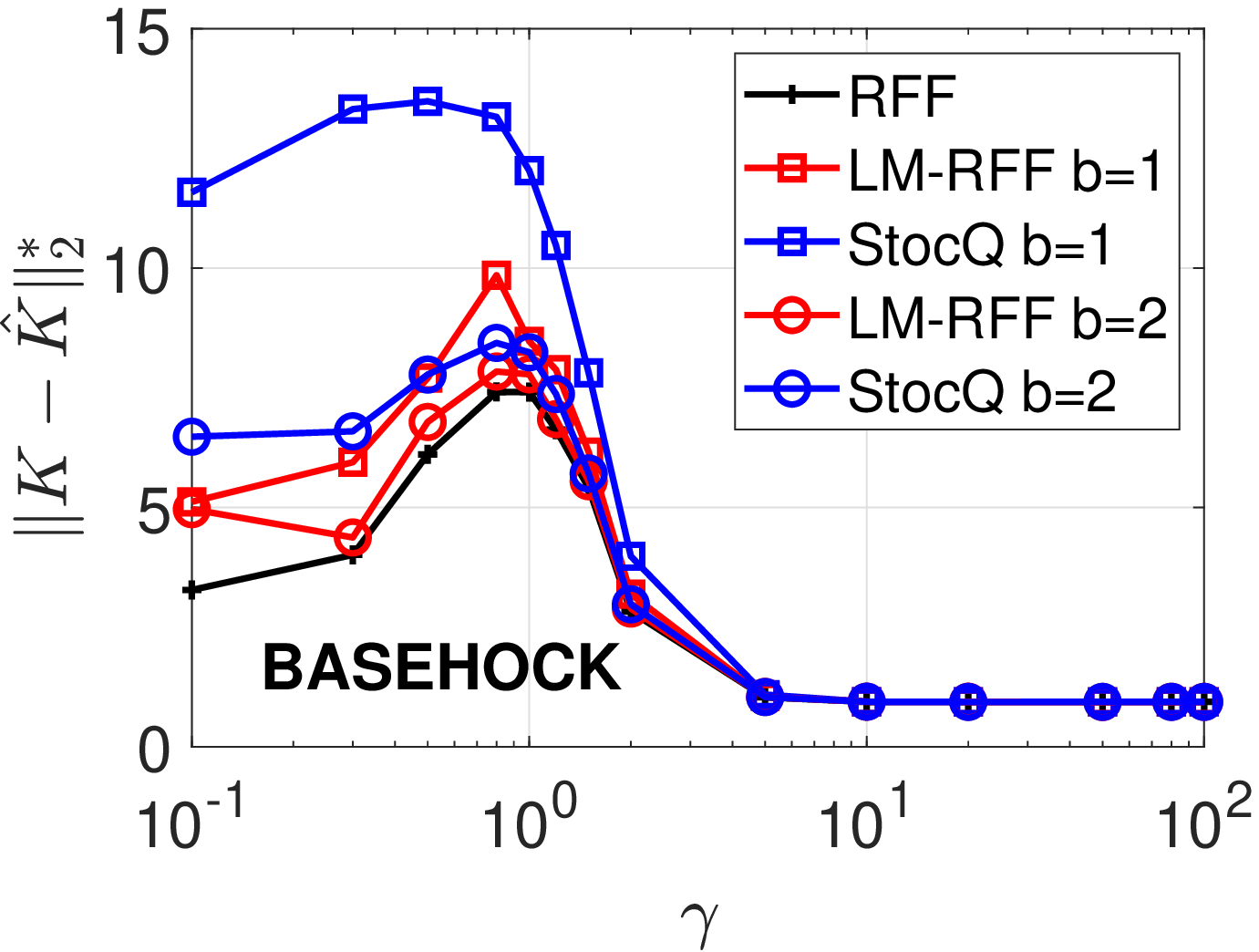}
        \includegraphics[width=2.2in]{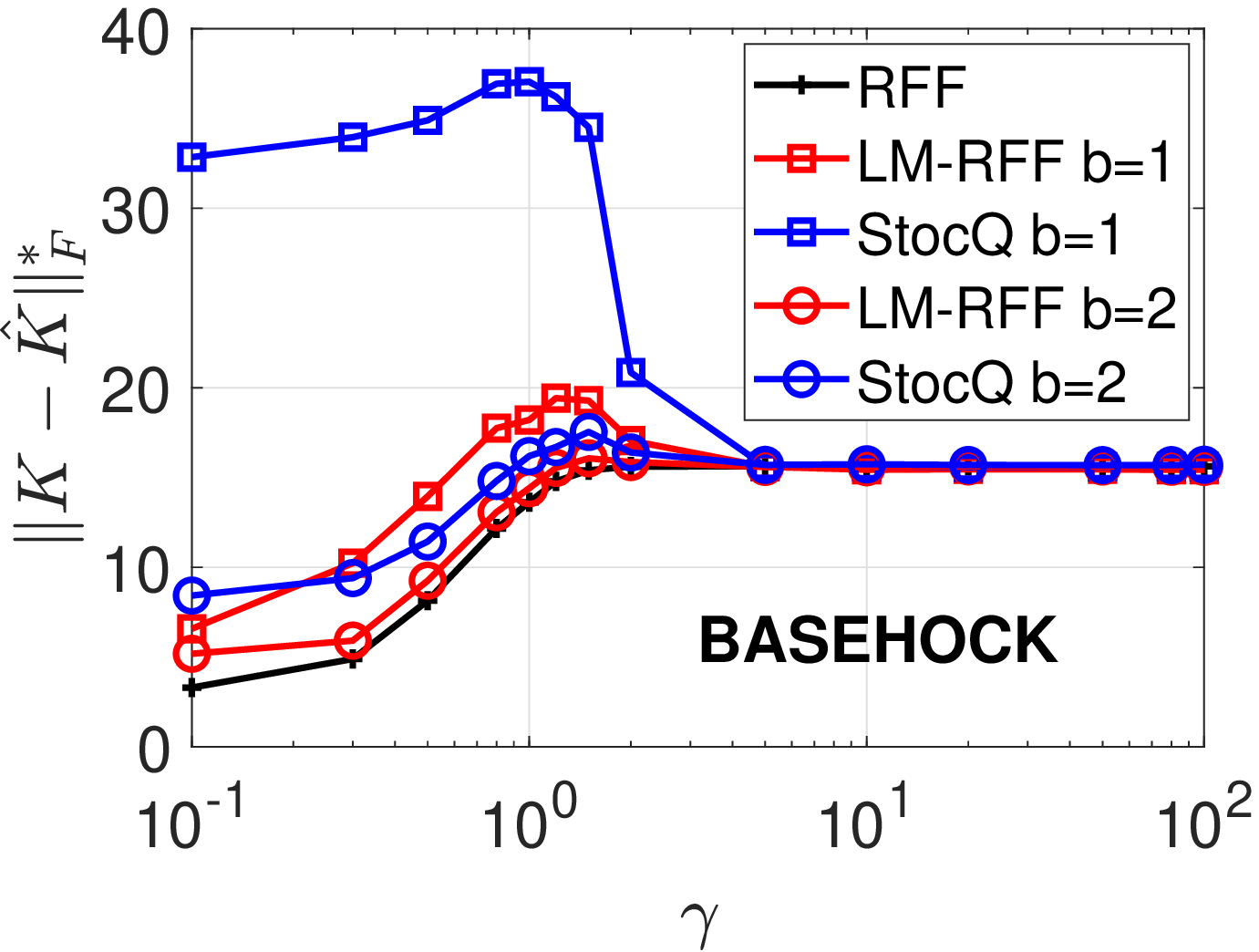}
        \includegraphics[width=2.2in]{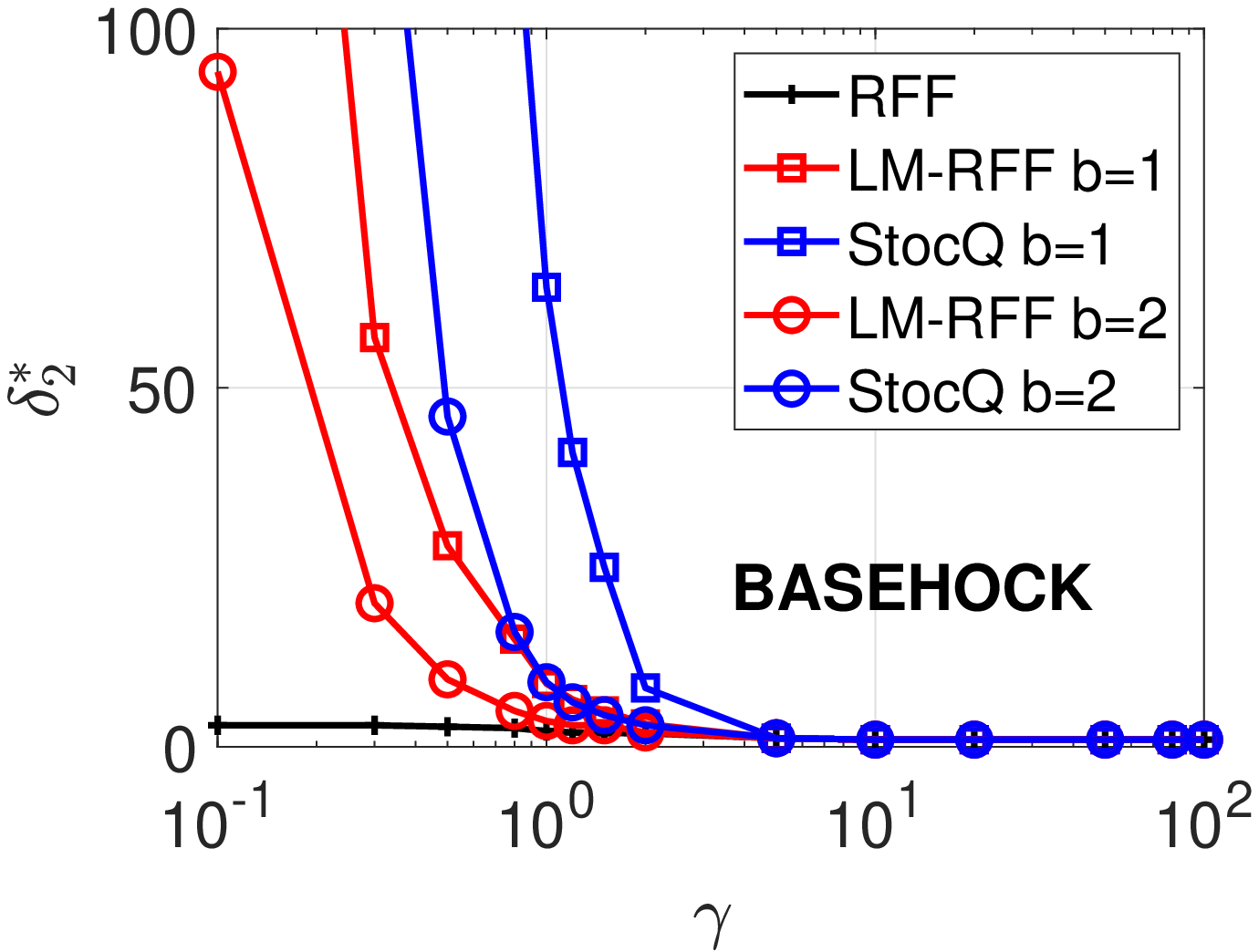}
        }
    \end{center}
    \vspace{-0.2in}
	\caption{Scale-invariant kernel approximation error (Definition~\ref{def:error_metric}) of LM-RFF vs. StocQ, $m=2^{10}$. Black curves are full-precision RFF, blue curves are StocQ, and curves in red represent our proposed LM-RFF. For all metrics, the smaller the better.}
	\label{fig_approx_error}
\end{figure}

\section{Conclusion}

The technique of random Fourier features (RFF) is a popular method to solve the computational bottleneck in large-scale (Gaussian) kernel learning tasks. In this paper, we study quantization methods to compress RFFs for substantial memory savings and efficient computations. In particular, we focus on developing quantization algorithms based on the Lloyd-Max (LM) framework and propose two methods named LM-RFF and LM$^2$-RFF. In addition, we also analyze a method based on stochastic rounding (StocQ). Both theoretically and empirically, LM-RFF significantly outperforms StocQ on many tasks, especially when the number of bits is not large. Compared to full-precision (e.g., 32- or 64-bit) RFFs, the experiments imply that often a 2-bit LM-RFF quantizer achieve comparable performance with full-precision, at a substantial (e.g., 10x) saving in memory cost, which would be highly beneficial in practical applications.

\newpage

\bibliography{standard}
\bibliographystyle{plainnat}

\appendix

\onecolumn
\newpage
\clearpage
\onecolumn
\renewcommand\thesection{\Alph{section}}
\setcounter{section}{0}

\newpage

\section{Lloyd-Max (LM) Quantization: Derivation and Properties} \label{append sec:LM-intro}

We provide a detailed derivation of Lloyd-Max (LM) quantization scheme and its properties, which would be useful to our analysis. Recall that our proposed LM-RFF quantizers minimize the distortion defined~as
\begin{align*}
    D_Q=\int_{\mathcal S} (Q(z)-z)^2 f(z) dz,
\end{align*}
where $f(z)$ is the signal distribution. Also, our $b$-bit fixed quantizer $Q$ has borders $t_0<...<t_{M}$ and reconstruction levels $\mu_1<...<\mu_M$, with $M=2^b$. Since the sine and cosine function are bounded within $[-1,1]$, we have $t_0=-1$ and $t_M=1$. Thus the distortion is
\begin{align*}
    D_Q=\sum_{i=1}^M \int_{t_{i-1}}^{t_i} (z-\mu_i)^2 f(z)dz.
\end{align*}
Lloyd's algorithm finds a stationary point of above system. By setting the derivative of $D_Q$ w.r.t. $\mu_i$ to 0
\begin{align*}
    \frac{\partial D_Q}{\partial \mu_i}=-2\int_{t_{i-1}}^{t_i} (z-\mu_i)f(z) dz=0,
\end{align*}
we obtain
\begin{align*}
    \mu_i=\frac{\int_{t_{i-1}}^{t_i} zf(z) dz}{\int_{t_{t-1}}^{t_i} f(z)dz}.
\end{align*}
We do the same thing for $t_i$ (i.e., setting $\frac{\partial D_Q}{\partial t_i}=0$) and get
\begin{align*}
    t_i=\frac{\mu_i+\mu_{i+1}}{2}.
\end{align*}
The following two useful properties hold for LM quantizers.\\

\begin{property}
$\mathbb E[z]=\mathbb E[Q(z)]$.
\end{property}
\begin{property}
$\mathbb E[Q(z)z]=\mathbb E[Q(z)^2]$.
\end{property}
\begin{proof}
For Property 1, we have
\begin{align}
    \mathbb E[Q(z)]&=\sum_{i=1}^{M}\int_{t_{i-1}}^{t_i} \frac{\int_{t_{i-1}}^{t_i} zf(z)dx}{\int_{t_{i-1}}^{t_i} f(z)dz} f(z) dz  \nonumber\\
    &=\sum_{i=1}^{M}\int_{t_{i-1}}^{t_i} \int_{t_{i-1}}^{t_i} zf(z)dz=\mathbb E[z].
\end{align}
For Property 2, similarly we have
\begin{align*}
    \mathbb E[Q(z)z]&=\sum_{i=1}^{M}\int_{t_{i-1}}^{t_i} \frac{\int_{t_{i-1}}^{t_i} zf(z)dx}{\int_{t_{i-1}}^{t_i} f(z)dz} z f(z) dz  \nonumber\\
    &=\sum_{i=1}^{M}\int_{t_{i-1}}^{t_i} \frac{(\int_{t_{i-1}}^{t_i} zf(z)dx)^2}{(\int_{t_{i-1}}^{t_i} f(z)dz)^2} f(z) dz=\mathbb E[Q(z)^2].
\end{align*}
\end{proof}

\clearpage\newpage

\section{More Analytical Figures in Section 4} \label{append sec:variance}

In Figure~\ref{append:fig_bias}, we present more figures on the bias of LM quantized estimators, corresponding to Theorem~\ref{theo:mean-var}, Theorem~\ref{theo: mean-var-norm}. Same as in the main paper, we see that the proposed surrogates (Observations~\ref{obv1} and~\ref{obv3}) align well with true biases. As $b$ increases, the bias vanishes towards $0$.

\begin{figure}[H]
    \begin{center}
        \mbox{\hspace{-0.1in}
        \includegraphics[width=2.2in]{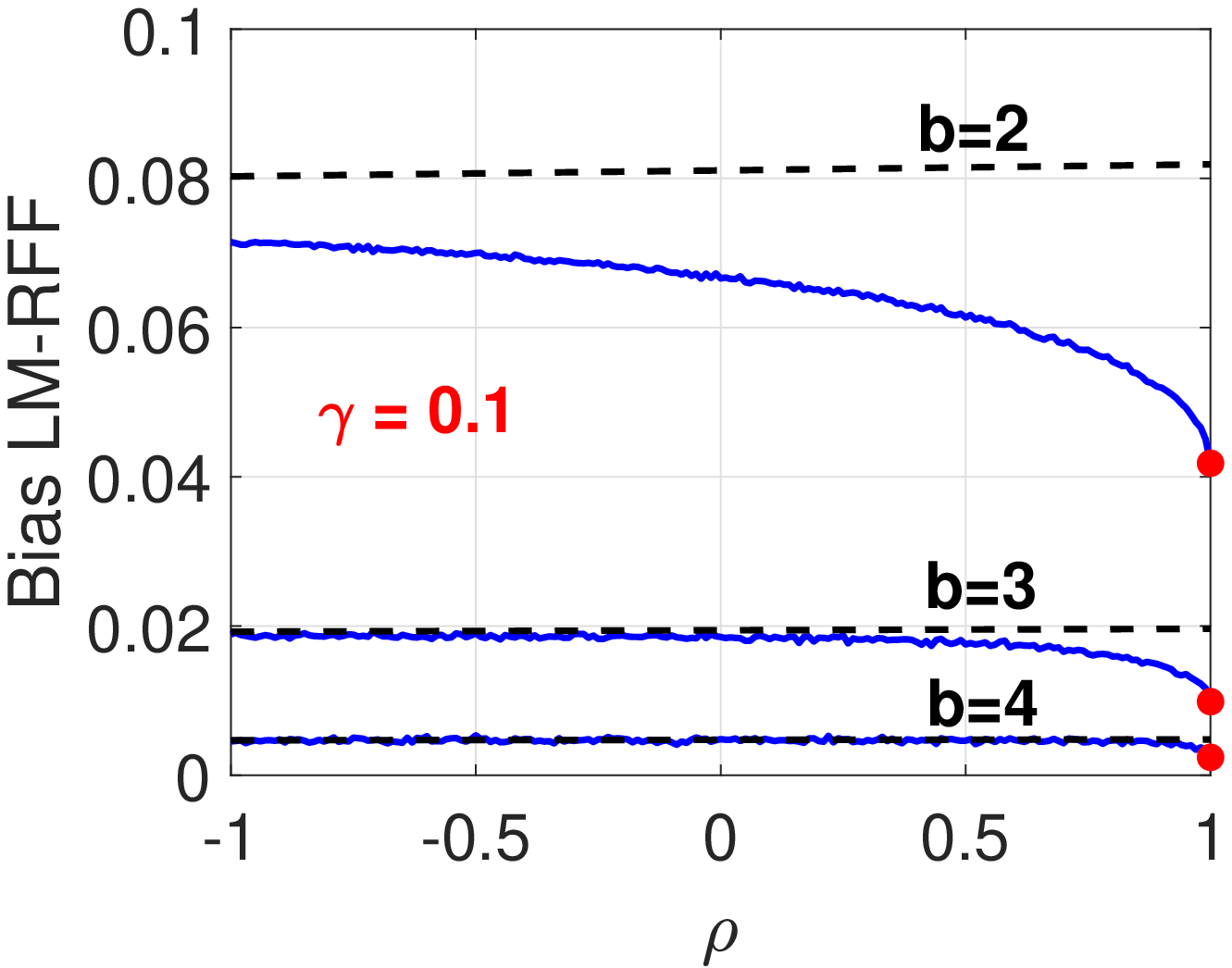}\hspace{-0.1in}
        \includegraphics[width=2.2in]{figure/theory/mean_type1_s05.eps}\hspace{-0.1in}
        \includegraphics[width=2.2in]{figure/theory/mean_type1_s1.eps}
        }
        \mbox{\hspace{-0.1in}
        \includegraphics[width=2.2in]{figure/theory/mean_type1_s2.eps}\hspace{-0.1in}
        \includegraphics[width=2.2in]{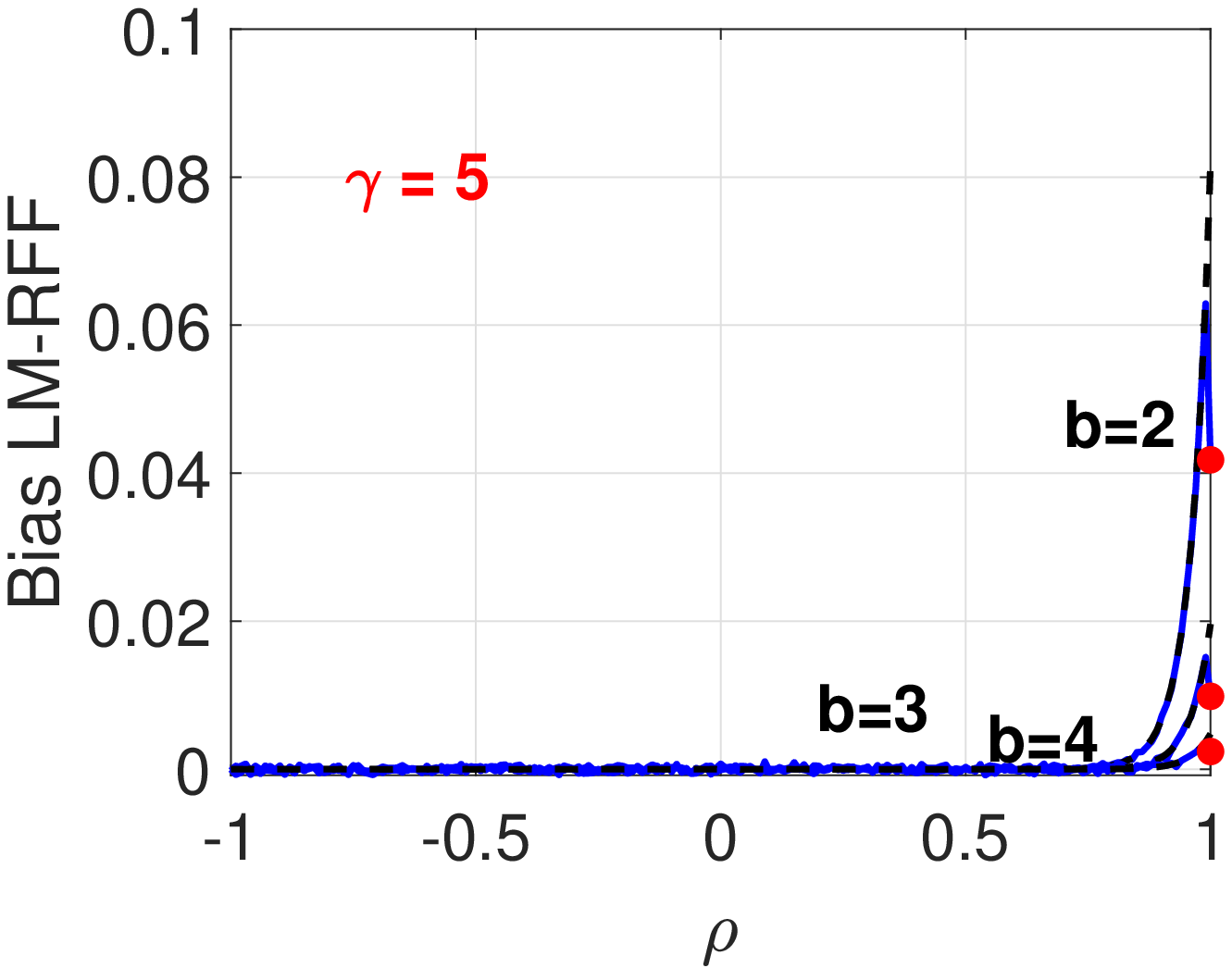}\hspace{-0.1in}
        \includegraphics[width=2.2in]{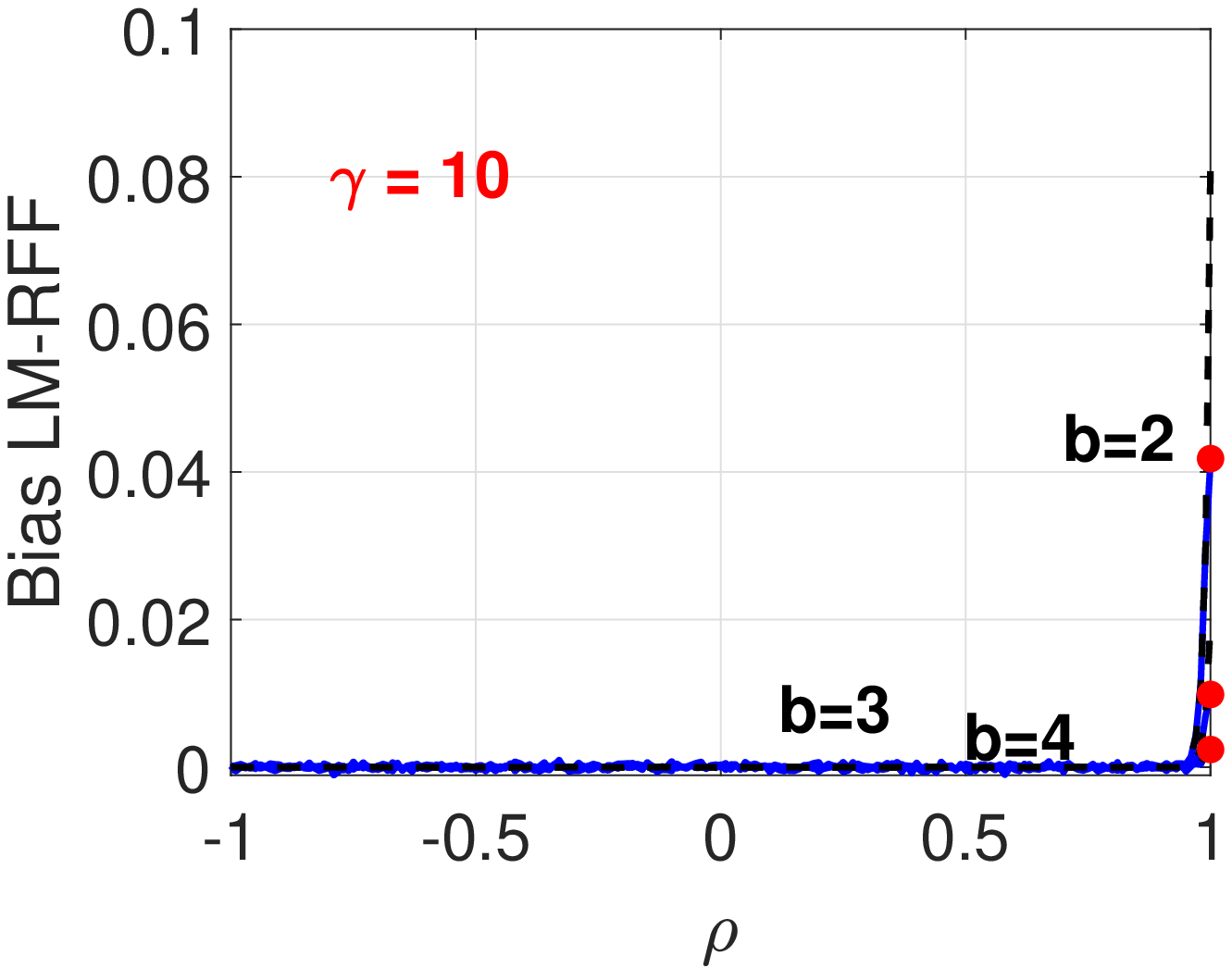}
        }
        \mbox{\hspace{-0.1in}
        \includegraphics[width=2.2in]{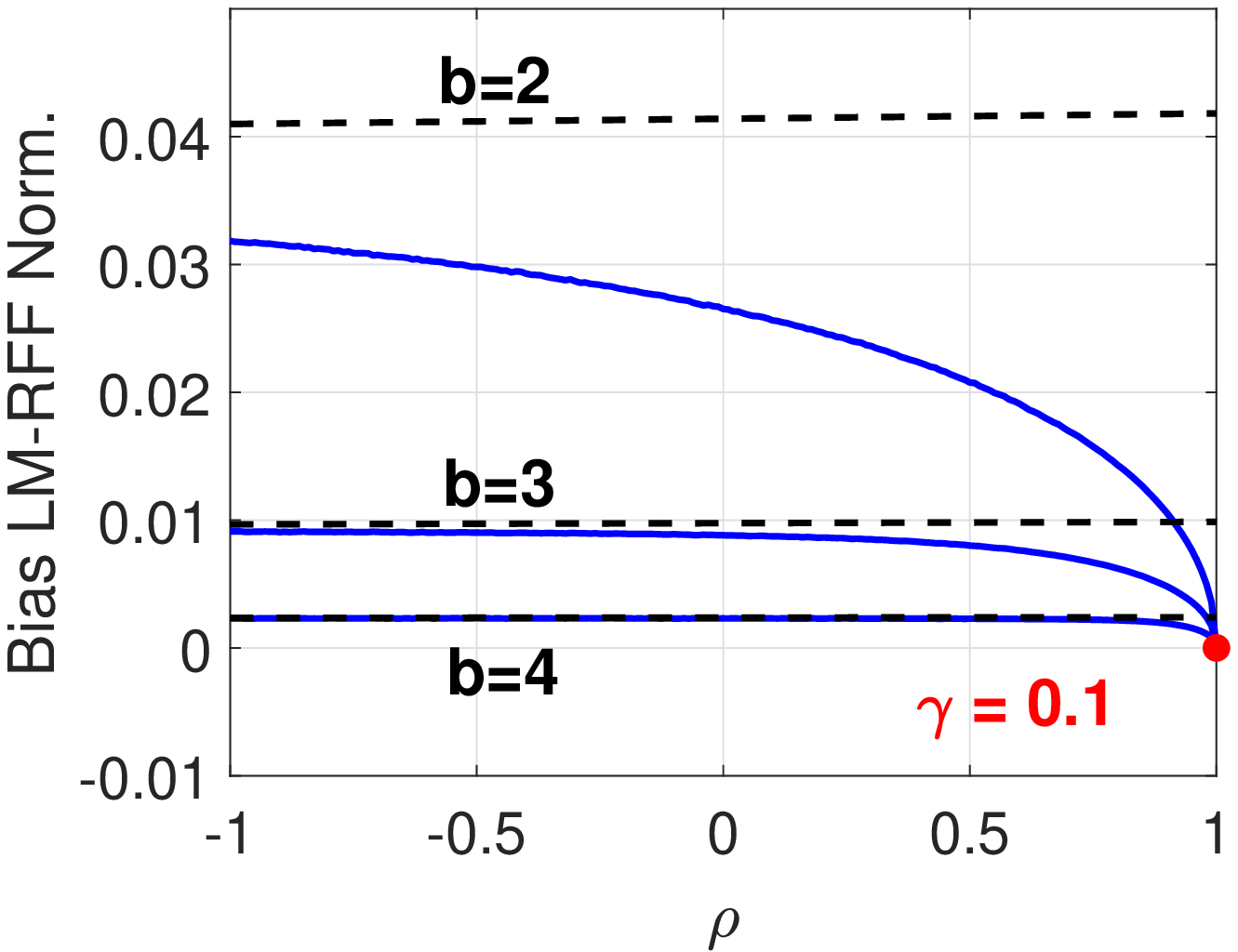}\hspace{-0.1in}
        \includegraphics[width=2.2in]{figure/theory/mean_type1_norm_s05.eps}\hspace{-0.1in}
        \includegraphics[width=2.2in]{figure/theory/mean_type1_norm_s1.eps}
        }
        \mbox{\hspace{-0.1in}
        \includegraphics[width=2.2in]{figure/theory/mean_type1_norm_s2.eps}\hspace{-0.1in}
        \includegraphics[width=2.2in]{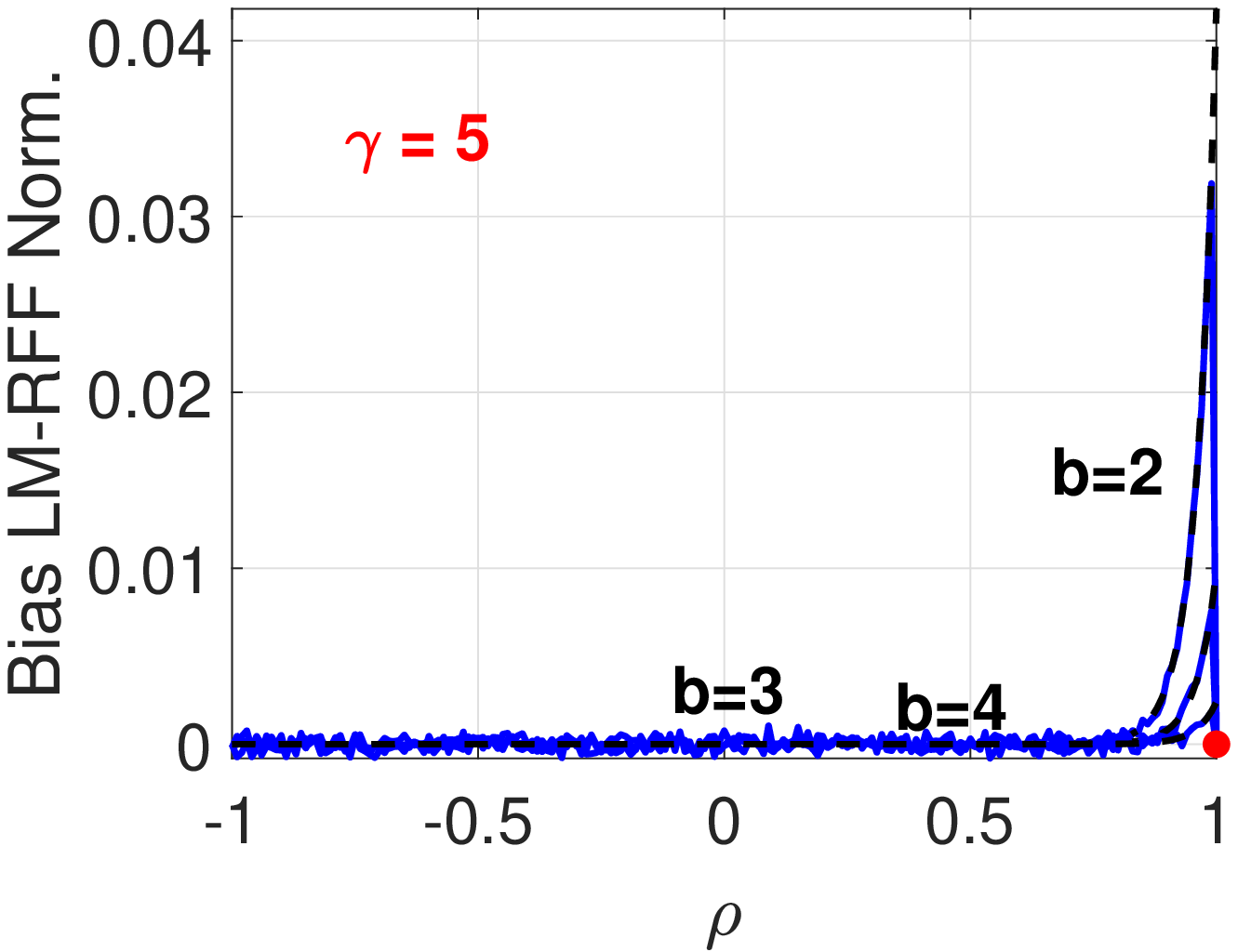}\hspace{-0.1in}
        \includegraphics[width=2.2in]{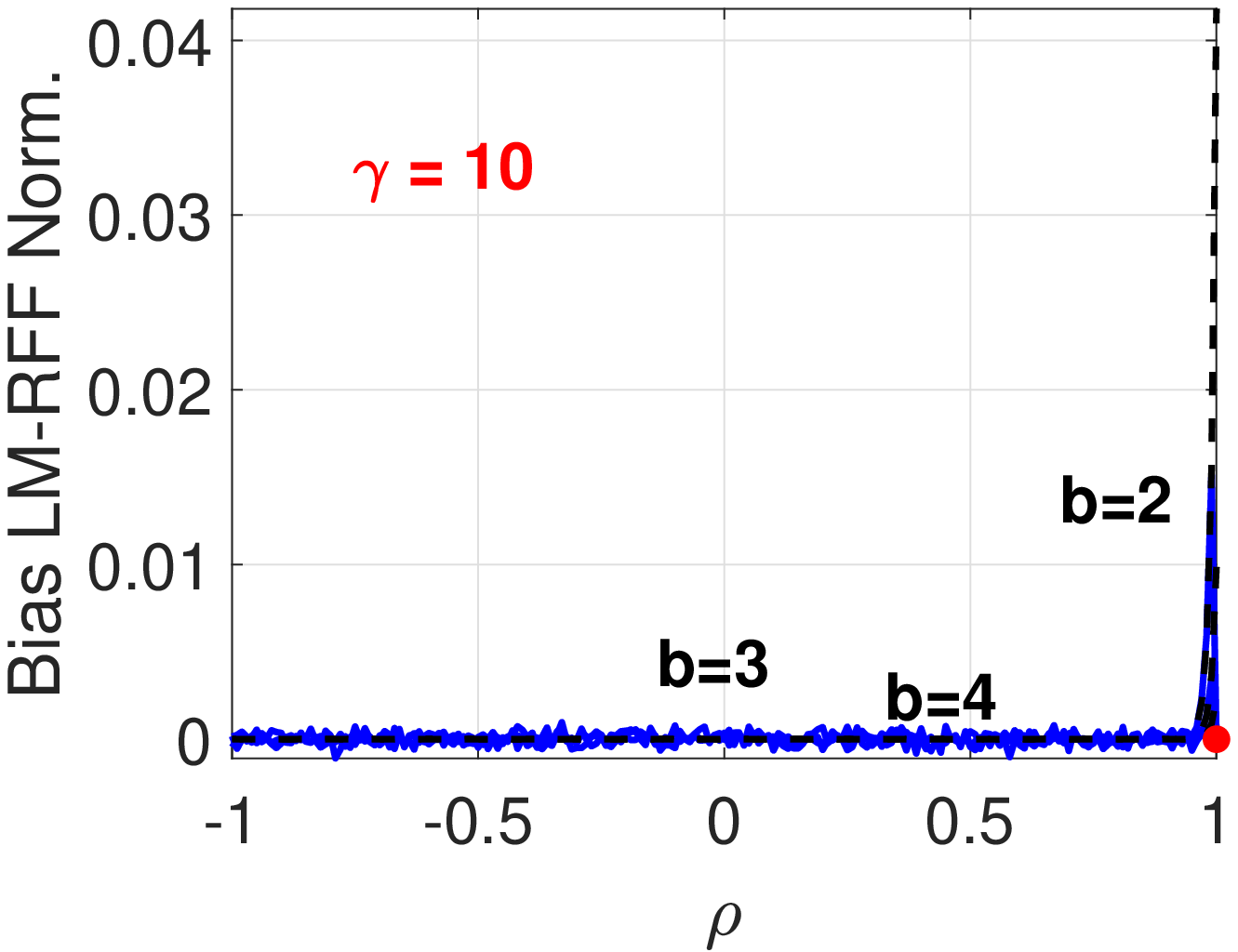}
        }
    \end{center}
    \vspace{-0.25in}
	\caption{Observation~\ref{obv1} and Observation~\ref{obv3} (black dash curves) vs. empirical bias (blue curves) of LM-RFF. Red dots are the biases given in the theorems at specific $\rho$ values.}
	\label{append:fig_bias}
\end{figure}

\newpage
In Figure~\ref{append:fig:variance-different b}, we provide more plots on variance of proposed LM-RFF estimators at more $\gamma$ levels. As we expect, the variances of LM-RFF quantized estimators converge to the corresponding full-precision estimators as the number of bits $b$ increases, i.e., $Var[\hat K_{Q}]\rightarrow Var[\hat K]$, $Var[\hat K_{n,Q}]\rightarrow Var[\hat K_{n}]$, as $b\rightarrow \infty$.

\begin{figure}[H]
    \begin{center}
        \mbox{
        \includegraphics[width=2.2in]{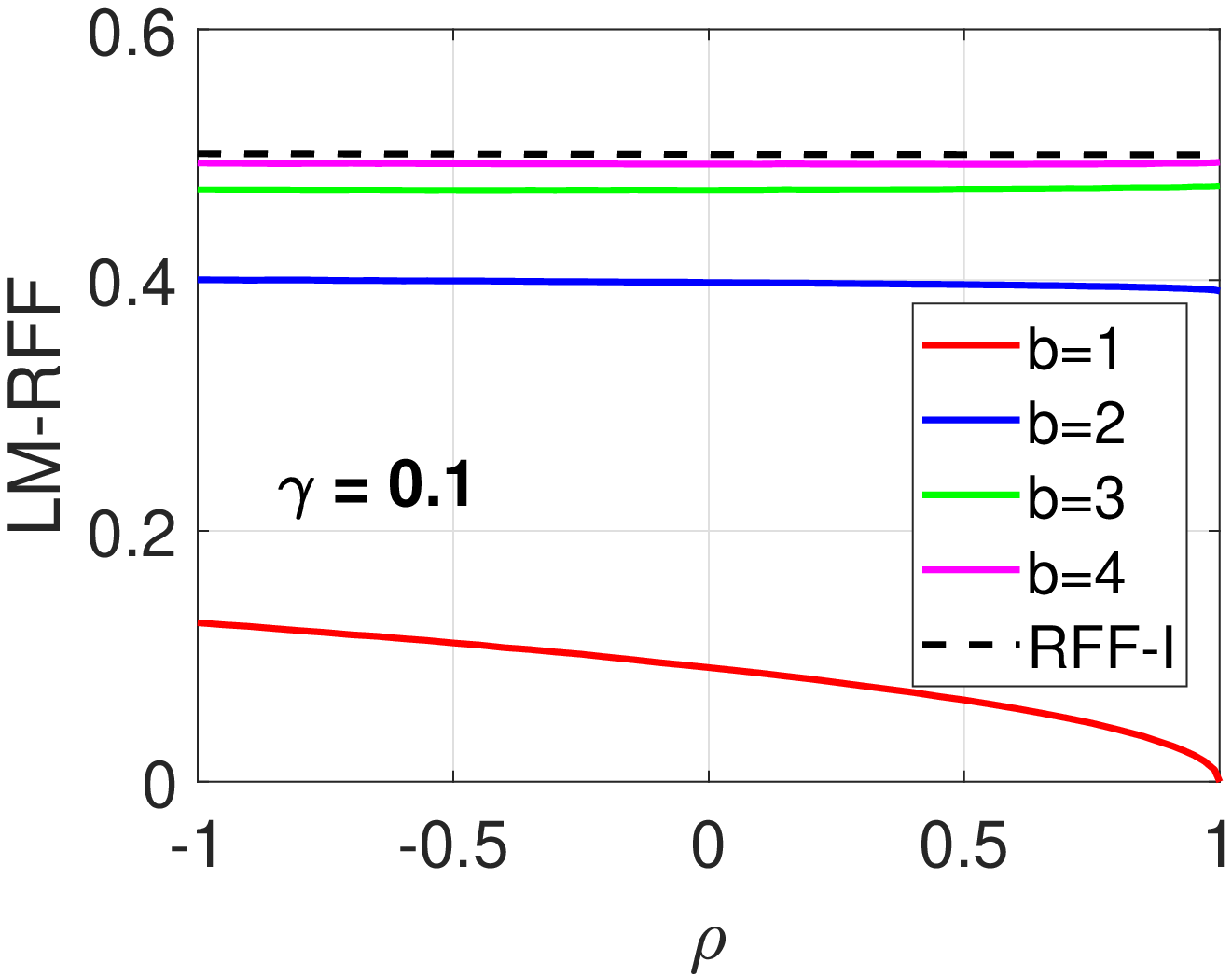}
        \includegraphics[width=2.2in]{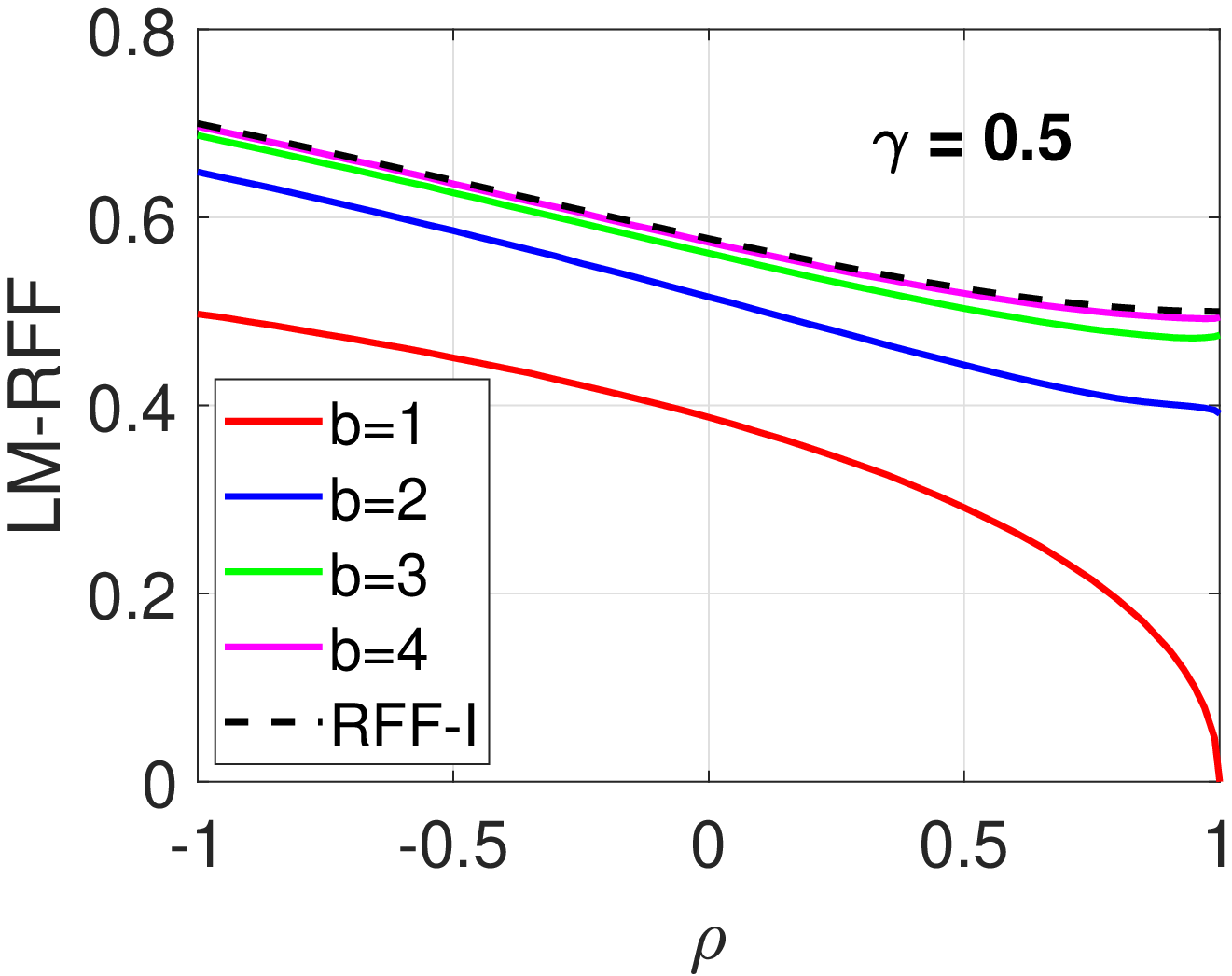}
        \includegraphics[width=2.2in]{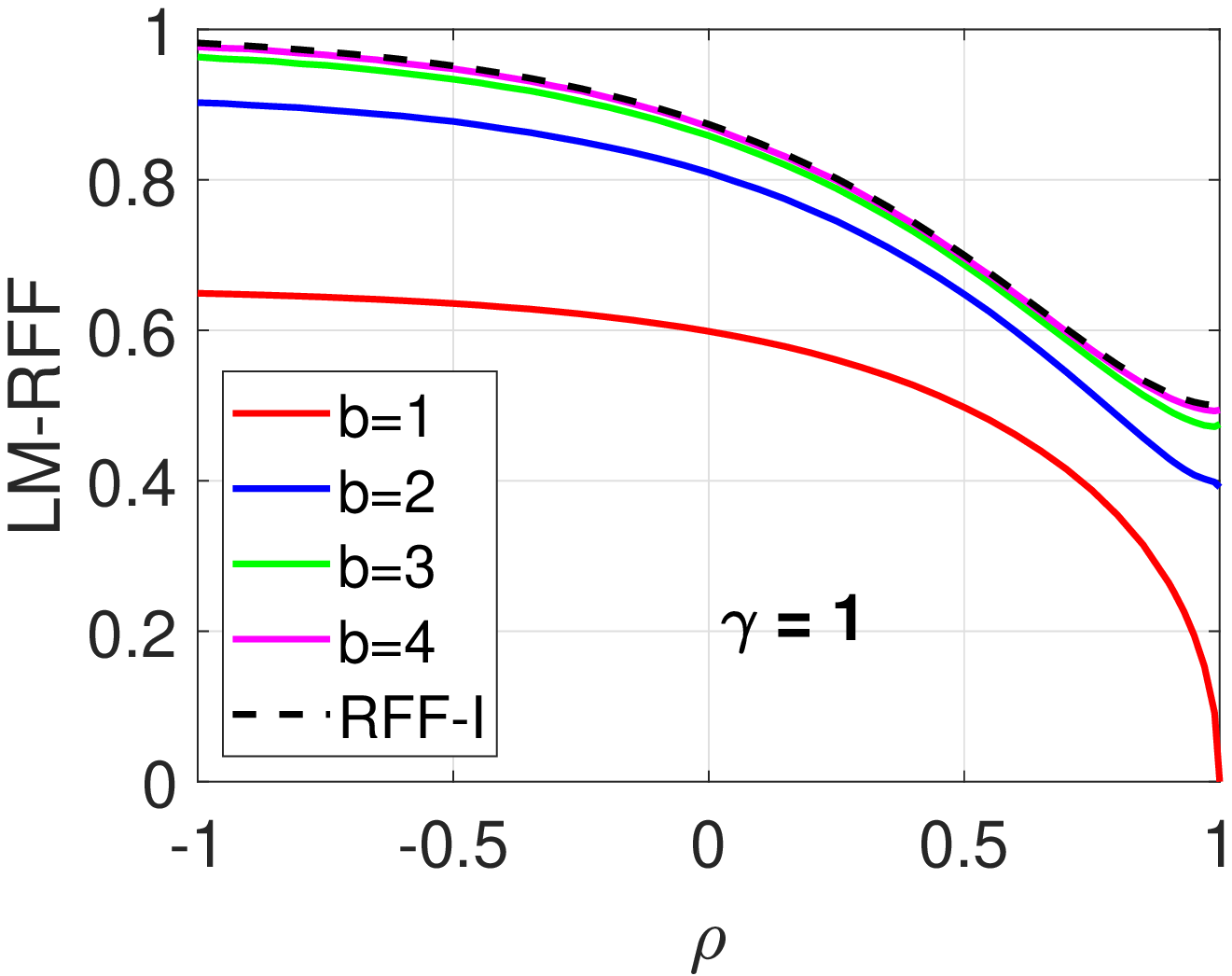}
        }
        \mbox{
        \includegraphics[width=2.2in]{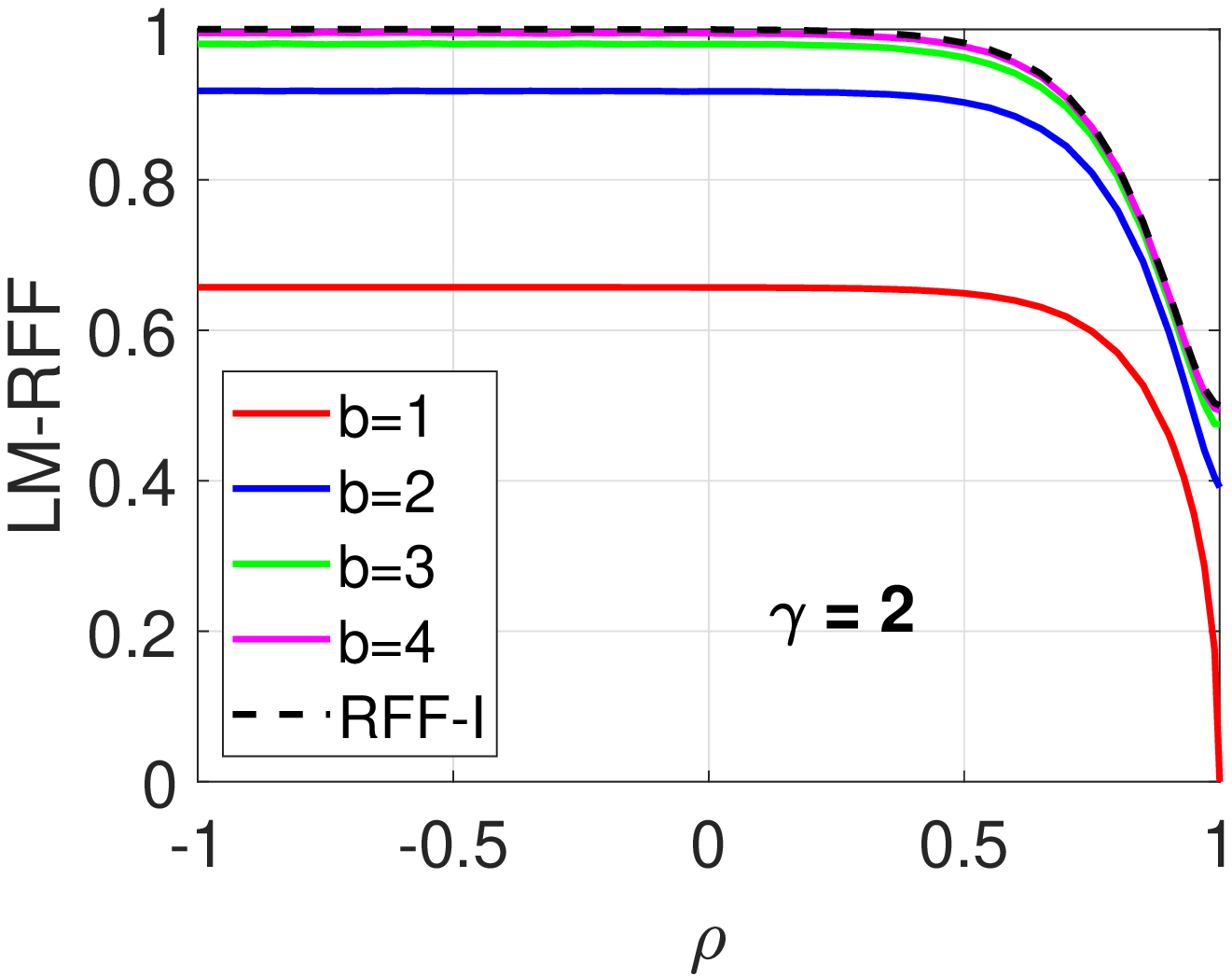}
        \includegraphics[width=2.2in]{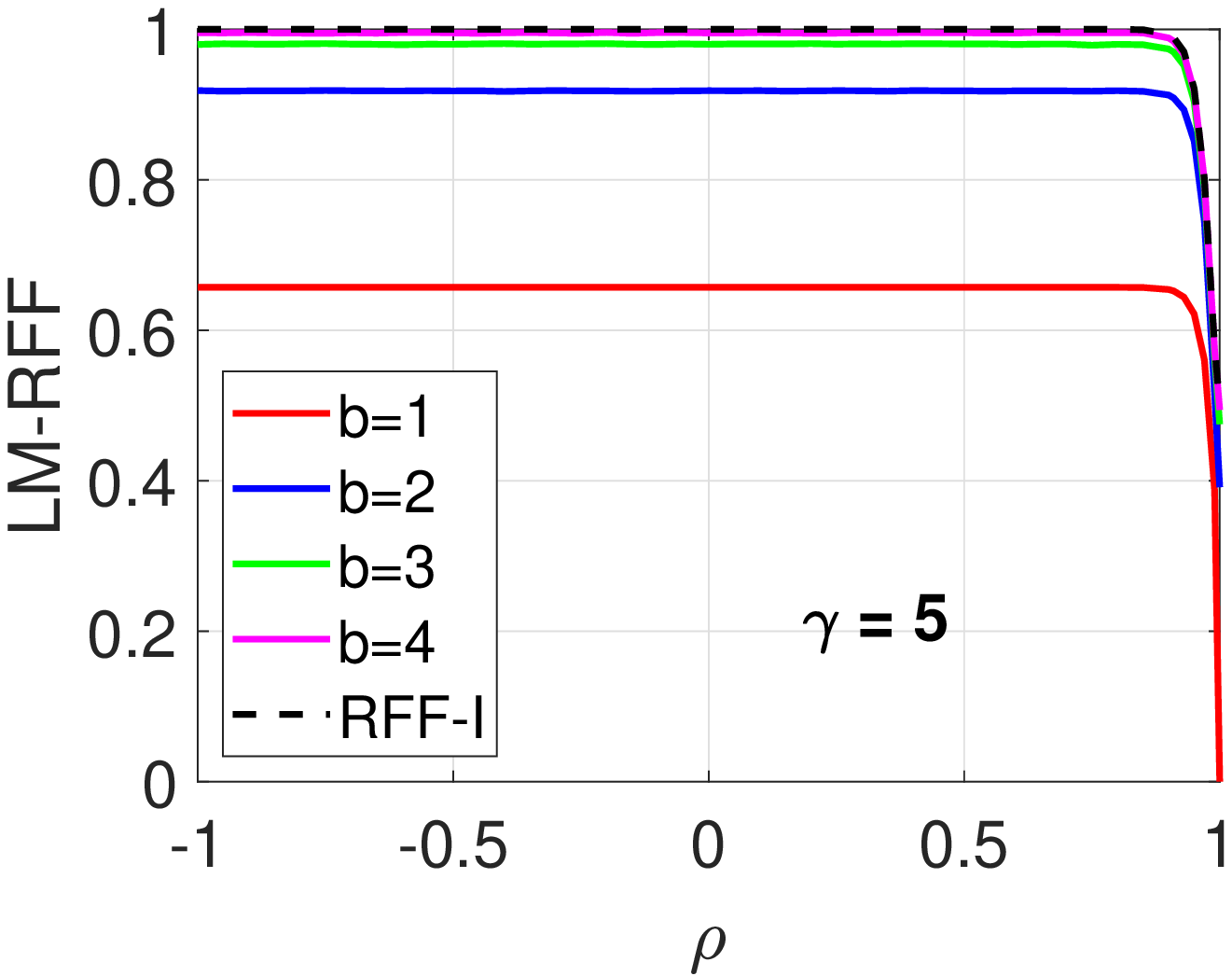}
        \includegraphics[width=2.2in]{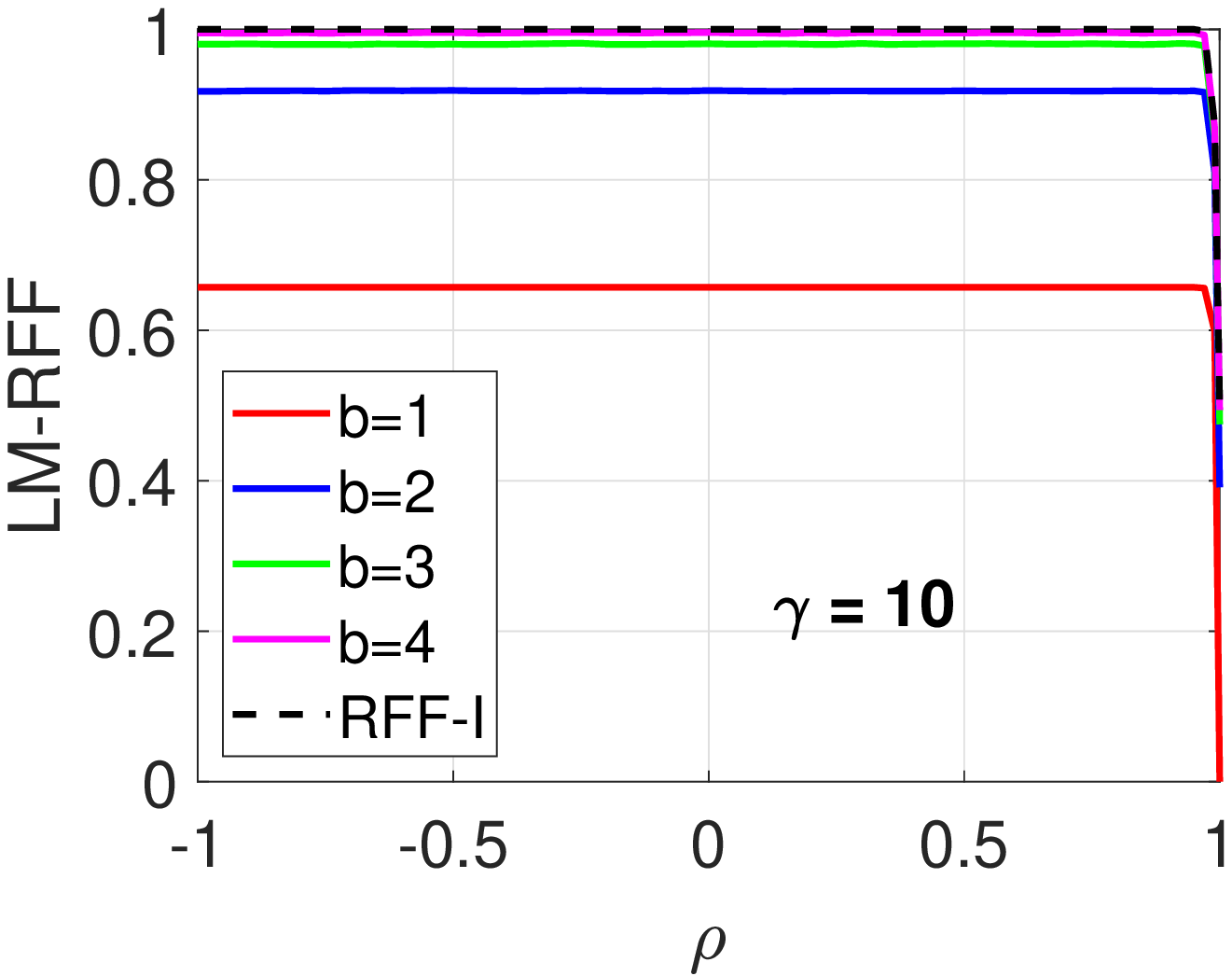}
        }
        \mbox{
        \includegraphics[width=2.2in]{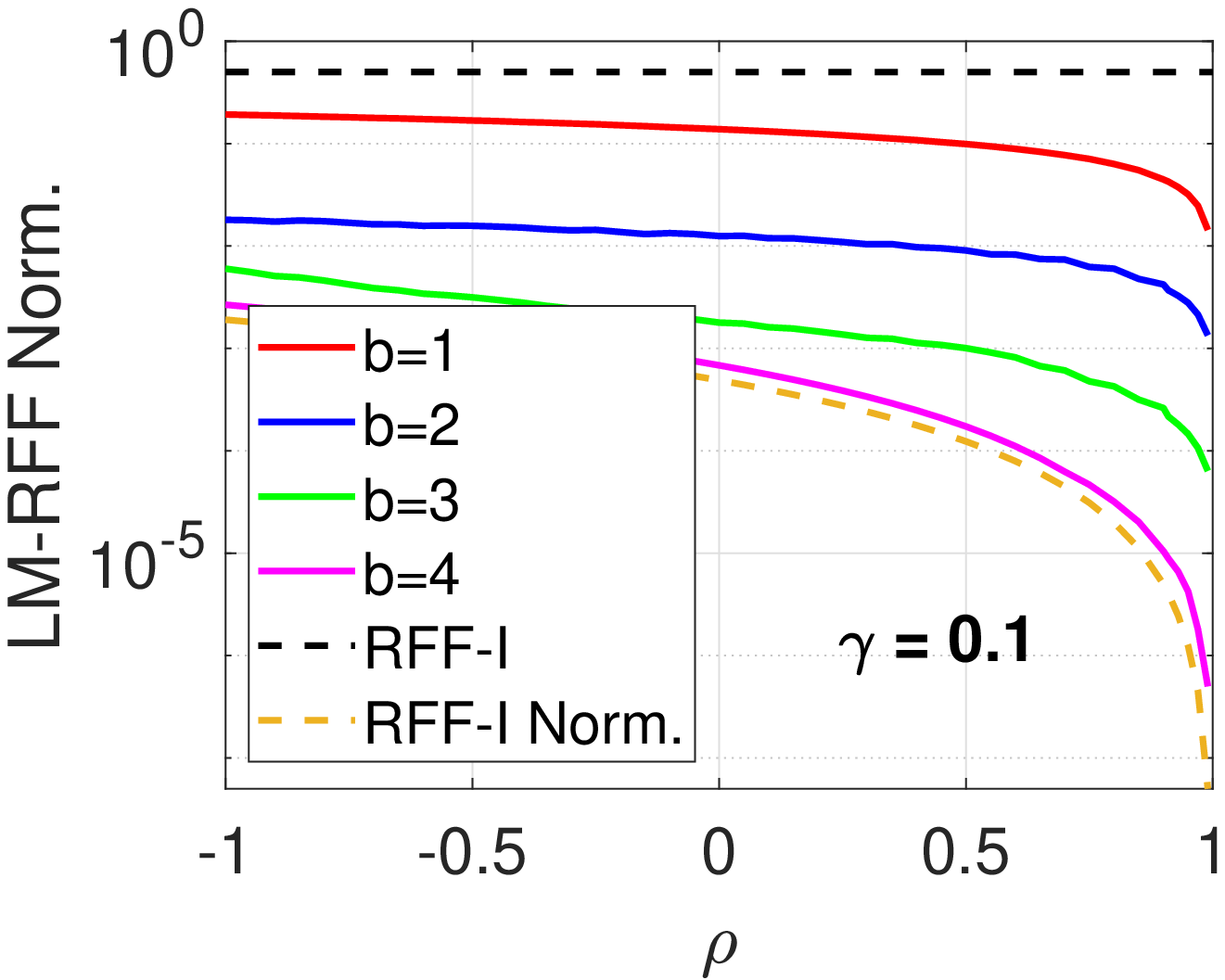}
        \includegraphics[width=2.2in]{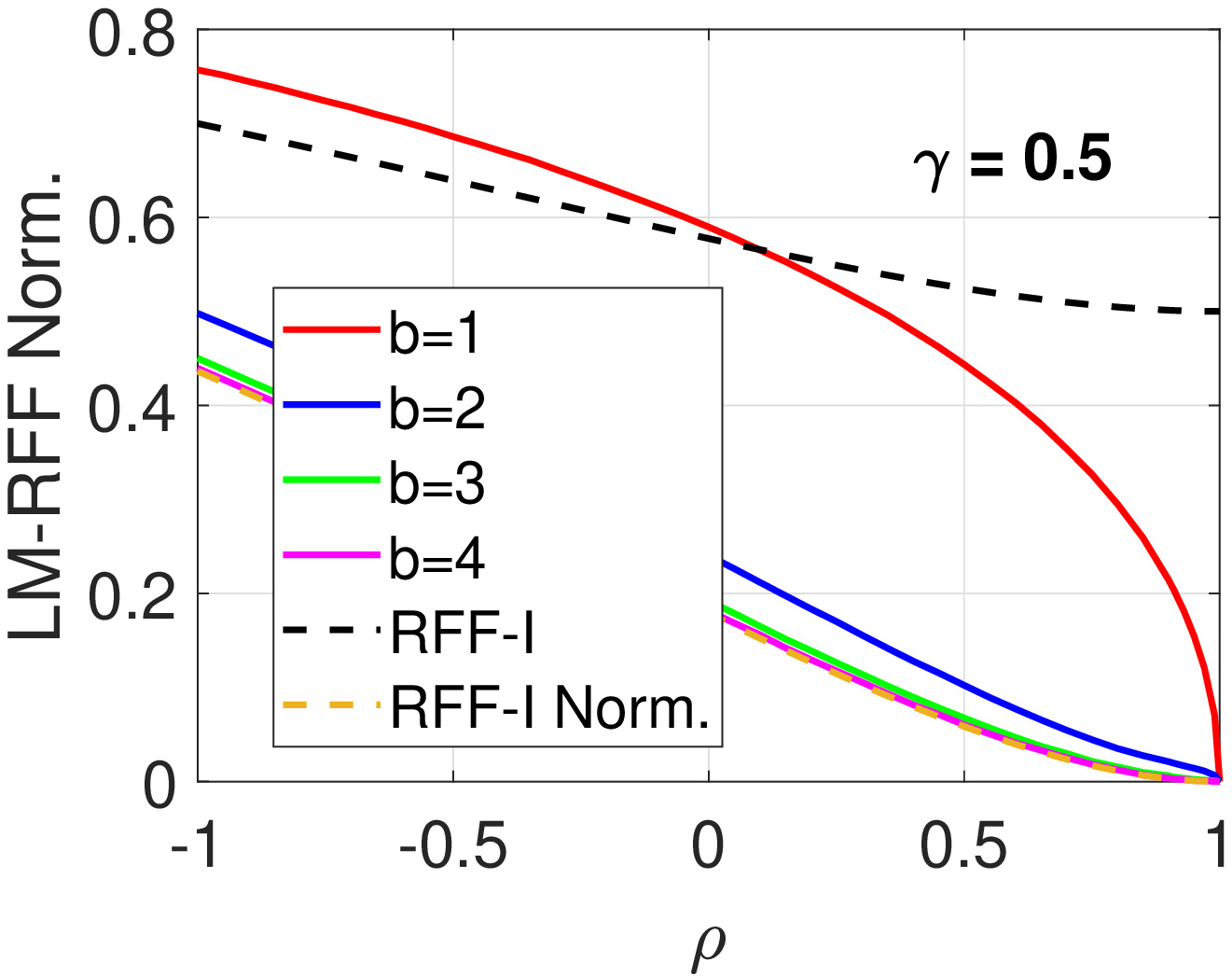}
        \includegraphics[width=2.2in]{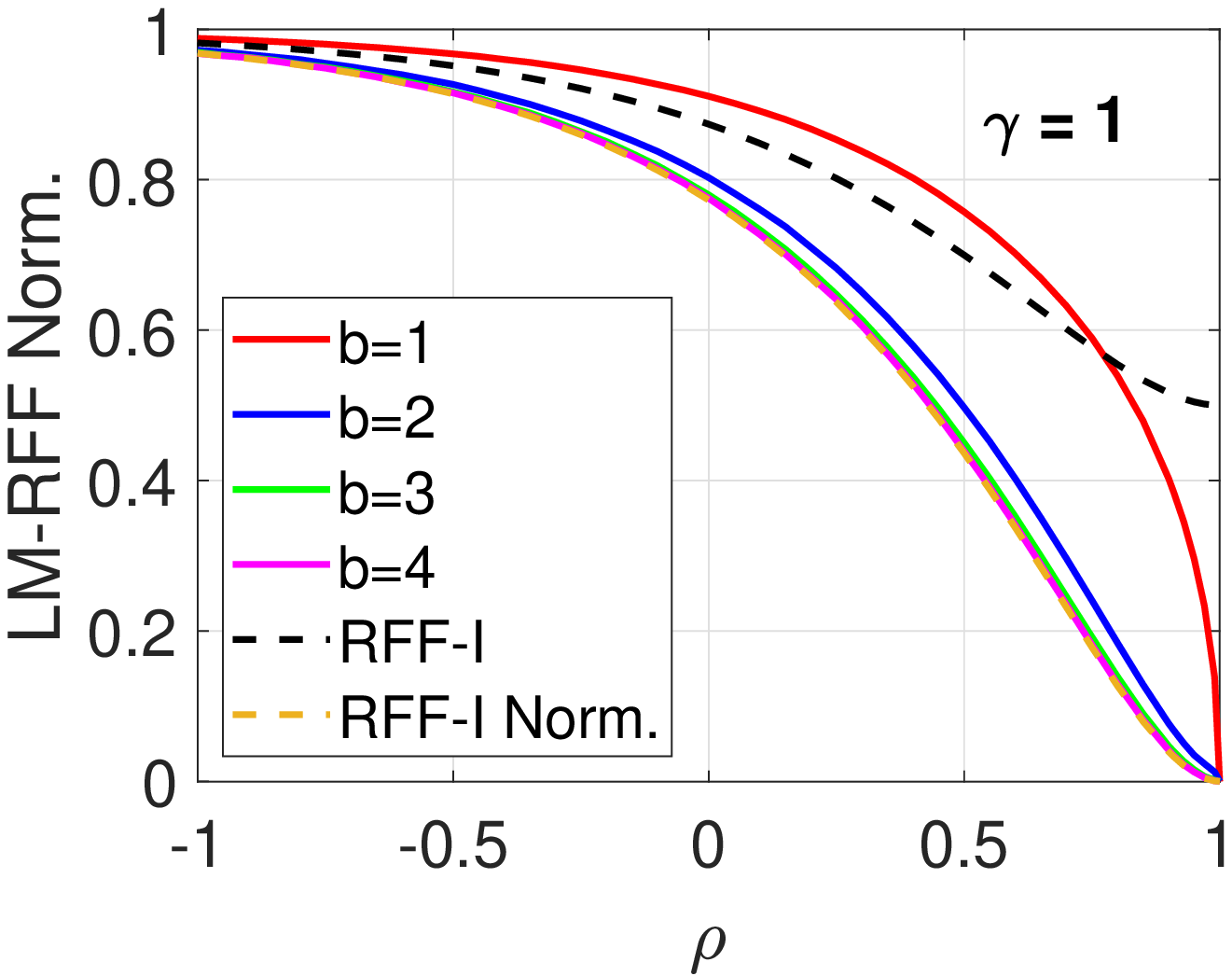}
        }
        \mbox{
        \includegraphics[width=2.2in]{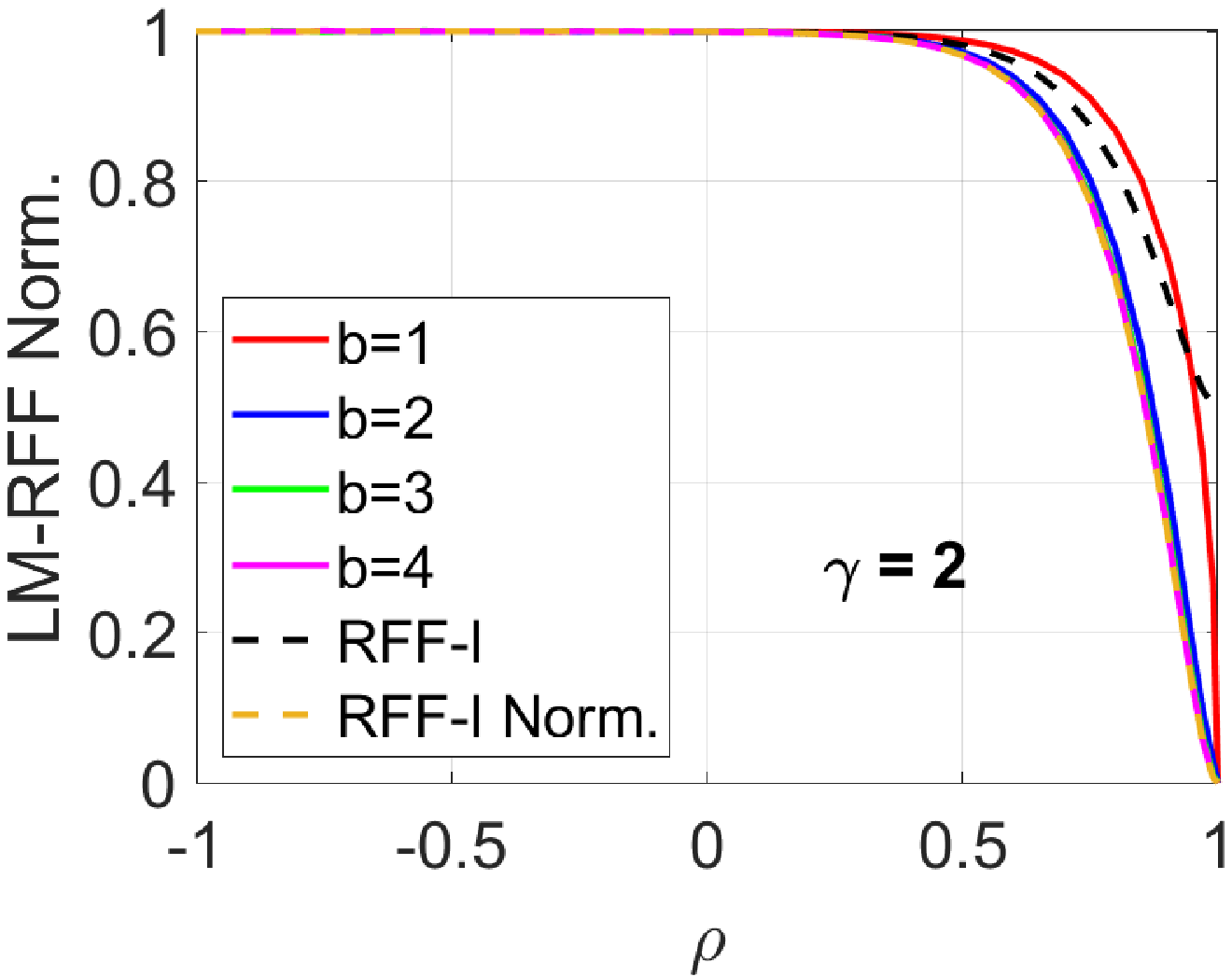}
        \includegraphics[width=2.2in]{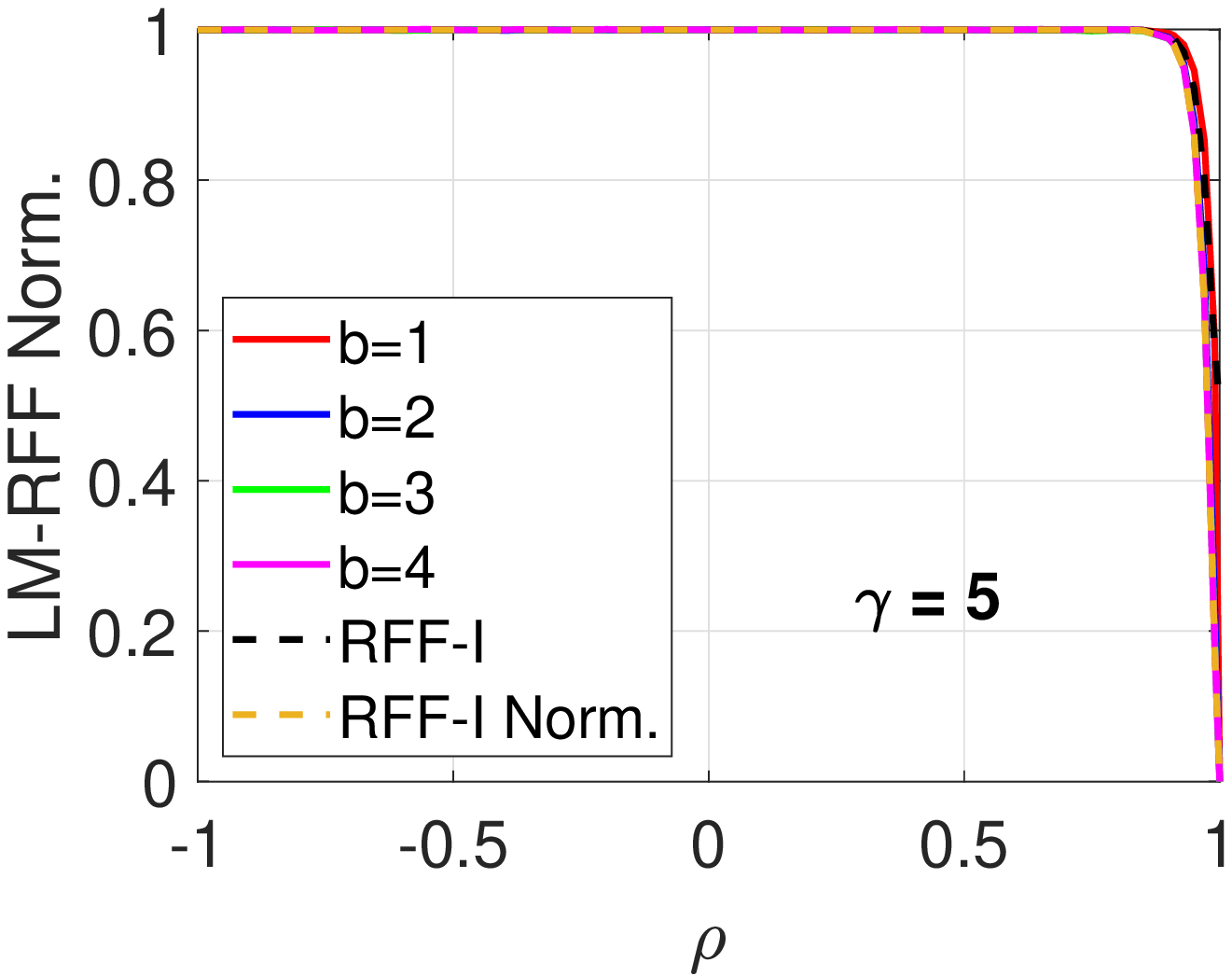}
        \includegraphics[width=2.2in]{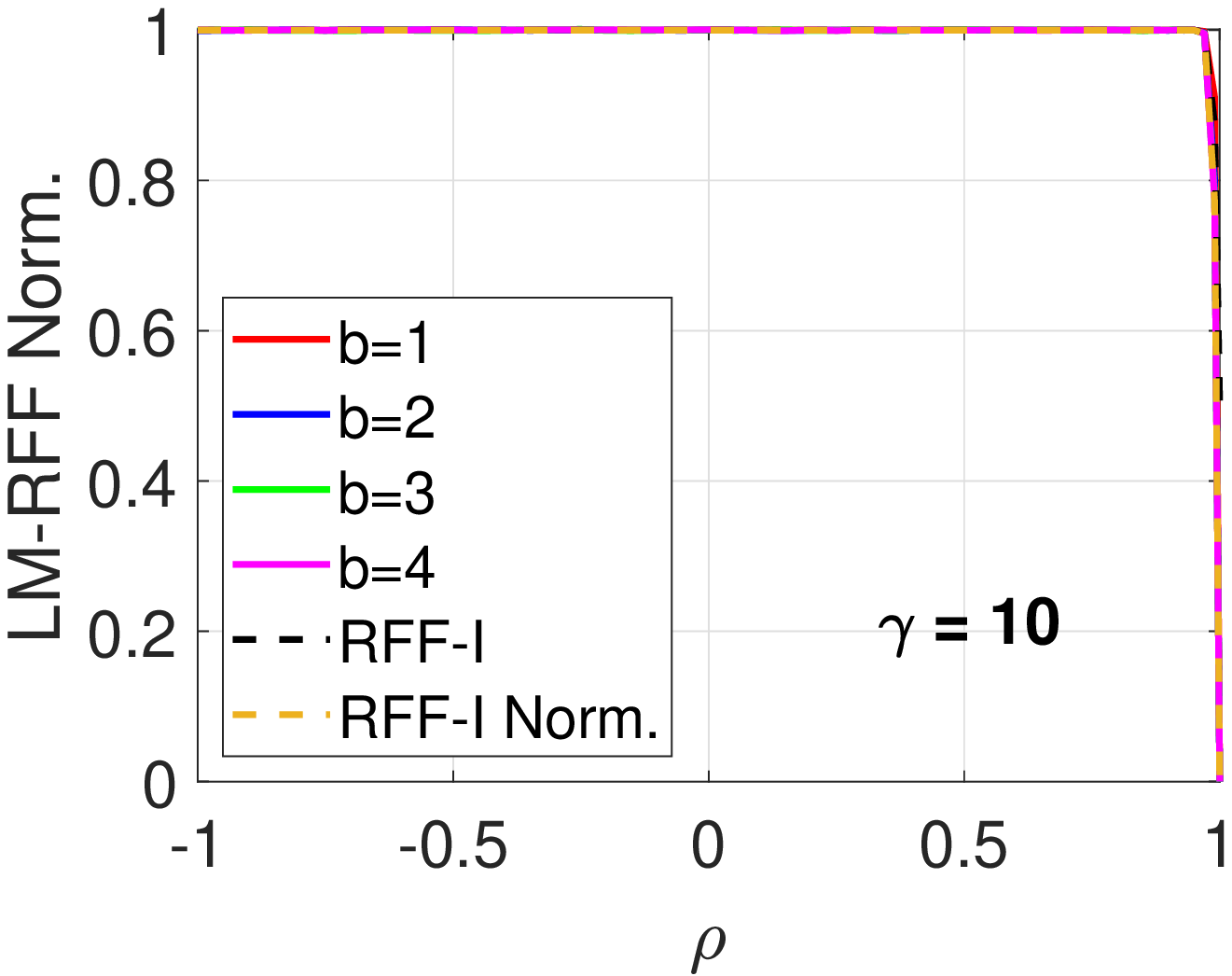}
        }
    \end{center}
    \vspace{-0.25in}
	\caption{Variance (scaled by $m$) of LM-RFF and LM-RFF Norm. estimators with different $\gamma$ and bits $b$. The dashed curves are the variances of full-precision counterparts.}
	\label{append:fig:variance-different b}
\end{figure}

\newpage
\section{Proofs}  \label{append sec:proof}

\subsection{Lemma~\ref{lemma1} \& Theorem~\ref{theo:density of RFF1}}

\begin{proof}(of Lemma~\ref{lemma1})
We have the convolution of uniform and Gaussian distribution as
\begin{align*}
    f_Y(y)&=\int_{-\infty}^\infty P(b=u, \gamma X=y-u)du\\
    &=\frac{1}{2\pi} \int_0^{2\pi} \frac{1}{\sqrt{2\pi}\gamma}e^{-\frac{(y-u)^2}{2\gamma^2}}du\\
    &=\frac{1}{2\pi}\left[ \Phi(\frac{2\pi-y}{\gamma})-\Phi(-\frac{y}{\gamma}) \right].
\end{align*}
\end{proof}

\begin{proof}(of Theorem~\ref{theo:density of RFF1})
Denote $Y=\gamma X+\tau$. We have
\begin{align*}
    P\Big(Z\leq z\Big)&=\sum_{k=-\infty}^\infty P\Big(2k\pi+\cos^{-1} z\leq Y\leq 2(k+1)\pi-\cos^{-1} z\Big) \nonumber\\
    &=\sum_{k=-\infty}^\infty \int_{2k\pi+\cos^{-1} z}^{2(k+1)\pi-\cos^{-1} z} f_Y(y)dy,
\end{align*}
where $f(y)$ is given by Lemma~\ref{lemma1}. Let the density of Z be $g_Z$, and denote $t^*=\cos^{-1} z$. It follows that
\begin{align}
    g_Z(z)&=\sum_{k=-\infty}^\infty \frac{1}{\sqrt{1-z^2}}\Big[  f_Y(2(k+1)\pi-t^*)+f_Y(2k\pi+t^*) \Big] \nonumber\\
    &=\frac{1}{2\pi\sqrt{1-z^2}} \sum_{k=-\infty}^\infty \Big[ \Phi(\frac{t^*-2k\pi}{\gamma})-\Phi(\frac{t^*-2(k+1)\pi}{\gamma}) \nonumber\\
    & \hspace{1.4in} +\Phi(\frac{-t^*-2(k-1)\pi}{\gamma})-\Phi(\frac{-t^*-2k\pi}{\gamma}) \Big] \nonumber\\
    &=\frac{1}{\pi\sqrt{1-z^2}}. \label{eqn:f(z)}
\end{align}
To prove the last line, denote the term in the bracket as $\alpha_k$. By cancellation, for any $k_1,k_2$, we have
\begin{align*}
    &\sum_{k=k_1}^{k_2} \alpha_k=\Big[ \Phi(\frac{t^*-2k_1\pi}{\gamma})+\Phi(\frac{-t^*-2(k_1 -1)\pi}{\gamma}) -\Phi(\frac{t^*-2(k_2 +1)\pi}{\gamma})-\Phi(\frac{-t^*-2k_2\pi}{\gamma}) \Big],
\end{align*}
which equals to $2$ in the limit $k_1\rightarrow -\infty,k_2\rightarrow \infty$. Using a similar approach, we can show that Eq. (\ref{eqn:f(z)}) is exactly the density of the cosine of a uniform random variable on $[0,2\pi]$. For $Z_2=\cos(\gamma X+\tau)=Z^2$, we have
\begin{align*}
    P[Z_2\leq z]&=P[|Z|\leq z]\\
    &=\frac{1}{\pi}\int_{\sqrt{z}}^{\sqrt{z}} \frac{1}{\sqrt{1-z^2}}dz\\
    &=\frac{1}{\pi} (\sin^{-1}(\sqrt{z})-\sin^{-1}(-\sqrt{z}))\\
    &=\frac{2}{\pi}\sin^{-1}(\sqrt{z}).
\end{align*}
Taking the derivative we get the p.d.f. as
\begin{align*}
    f_{Z_2}(z)=\frac{1}{\pi\sqrt{z-z^2}}.
\end{align*}
The proof is now complete.
\end{proof}

\subsection{Lemma~\ref{lemma2} \& Theorem~\ref{theo:joint_RFF1}}

\begin{proof}(of Lemma~\ref{lemma2})
Similar to the proof of Lemma~\ref{lemma1}, we have
\begin{align*}
    f(t_x,t_y)&=\frac{1}{2\pi}\int_0^{2\pi} P(\gamma x=t_x-u,\gamma y=t_y-u)du\\
    &=\frac{1}{4\pi^2\gamma^2\sqrt{1-\rho^2}}\int_0^{2\pi} e^{-\frac{(t_x-u)^2-2\rho(t_x-u)(t_y-u)+(t_y-u)^2}{2(1-\rho^2)\gamma^2}}du\\
    &=\frac{1}{4\pi^2\gamma^2\sqrt{1-\rho^2}}\int_0^{2\pi} e^{-\frac{2(1-\rho)(u^2-u(t_x+t_y))+t_x^2+t_y^2-2\rho t_xt_y}{2(1-\rho^2)\gamma^2}}du\\
    &=\frac{1}{4\pi^2\gamma^2\sqrt{1-\rho^2}}\int_0^{2\pi} e^{-\frac{2(1-\rho)(u-\frac{t_x+t_y}{2})^2+\frac{1+\rho}{2}(t_x-t_y)^2}{2(1-\rho^2)\gamma^2}}du\\
    &=\frac{1}{4\pi^2\gamma^2\sqrt{1-\rho^2}}e^{-\frac{(t_x-t_y)^2}{4(1-\rho)\gamma^2}}\int_0^{2\pi} e^{-\frac{(u-\frac{t_x+t_y}{2})^2}{(1+\rho)\gamma^2}}du\\
    &=\frac{1}{2\pi}\phi_{\sqrt{2(1-\rho)}\gamma}(t_x-t_y) \Big[\Phi(\frac{4\pi-(t_x+t_y)}{\gamma\sqrt{2(1+\rho)}})-\Phi(-\frac{t_x+t_y}{\gamma\sqrt{2(1+\rho)}}) \Big],
\end{align*}
where $\phi_{\sqrt{2(1-\rho)}\gamma}$ is the density of $N(0,2(1-\rho)\gamma^2)$.
\end{proof}

\begin{proof}(of Theorem~\ref{theo:joint_RFF1})
Denote $Z_x=\cos(t_x),Z_y=\cos(t_y)$. Let $a_x^*=\cos^{-1}(z_x),a_y^*=\cos^{-1}(z_y)$. Denote $\phi=\phi_{\sqrt{2(1-\rho)}\gamma}$ for simplicity. We have
\begin{align*}
    &P(Z_x\leq z_x,Z_y\leq z_y)=\sum_{k_x=-\infty}^\infty \sum_{k_y=-\infty}^\infty \int_{2k_x\pi+a_x^*}^{2(k_x+1)\pi-a_x^*}\int_{2k_y\pi+a_y^*}^{2(k_y+1)\pi-a_y^*} f(t_x,t_y)dt_x dt_y.
\end{align*}
By Lemma~\ref{lemma2}, it follows that
{\small
\begin{align*}
    &f(z_x,z_y)\\
    &=\frac{1}{2\pi}\sum_{k_x=-\infty}^\infty\sum_{k_y=-\infty}^\infty \int_{2k_x\pi+a_x^*}^{2(k_x+1)\pi-a_x^*}\frac{1}{\sqrt{1-z_y^2}}\Bigg\{ \phi(t_x-2(k_y+1)\pi+a_y^*)\Big[ \Phi(\frac{4\pi-(t_x+2(k_y+1)\pi-a_y^*)}{\gamma\sqrt{2(1+\rho)}})\\
    &\hspace{0.2in} -\Phi(-\frac{t_x+2(k_y+1)\pi-a_y^*}{\gamma\sqrt{2(1+\rho)}}) \Big]+\phi(t_x-2k_y\pi-a_y^*)\Big[ \Phi(\frac{4\pi-(t_x+2k_y\pi+a_y^*)}{\gamma\sqrt{2(1+\rho)}})-\Phi(-\frac{t_x+2k_y\pi+a_y^*}{\gamma\sqrt{2(1+\rho)}}) \Big] \Bigg\} dt_x\\
    &=\frac{1}{2\pi\sqrt{1-z_x^2}\sqrt{1-z_y^2}}\sum_{k_x}\sum_{k_y} \Bigg\{ \phi(-a_x^*+a_y^*+2(k_x-k_y)\pi)\Big[ \Phi(\frac{a_x^*+a_y^*-2(k_x+k_y)\pi)}{\gamma\sqrt{2(1+\rho)}})-\Phi(\frac{a_x^*+a_y^*-2(k_x+k_y+2)\pi}{\gamma\sqrt{2(1+\rho)}}) \Big]\\
    &\hspace{1.2in} +\phi(-a_x^*-a_y^*+2(k_x-k_y+1)\pi)\Big[ \Phi(\frac{a_x^*-a_y^*-2(k_x+k_y-1)\pi)}{\gamma\sqrt{2(1+\rho)}})-\Phi(\frac{a_x^*-a_y^*-2(k_x+k_y+1)\pi}{\gamma\sqrt{2(1+\rho)}}) \Big]\\
    &\hspace{1.2in} +\phi(a_x^*+a_y^*+2(k_x-k_y-1)\pi)\Big[ \Phi(\frac{-a_x^*+a_y^*-2(k_x+k_y-1)\pi)}{\gamma\sqrt{2(1+\rho)}})-\Phi(\frac{-a_x^*+a_y^*-2(k_x+k_y+1)\pi}{\gamma\sqrt{2(1+\rho)}}) \Big]\\
    &\hspace{1.2in} +\phi(a_x^*-a_y^*+2(k_x-k_y)\pi)\Big[ \Phi(\frac{-a_x^*-a_y^*-2(k_x+k_y-2)\pi)}{\gamma\sqrt{2(1+\rho)}})-\Phi(\frac{-a_x^*-a_y^*-2(k_x+k_y)\pi}{\gamma\sqrt{2(1+\rho)}}) \Big] \Bigg\}\\
    &\overset{(a)}{=}\frac{1}{2\pi\sqrt{1-z_x^2}\sqrt{1-z_y^2}}\sum_{k=-\infty}^\infty \Big[ \phi(-a_x^*+a_y^*+2k\pi)+\phi(-a_x^*-a_y^*+2k\pi)+\phi(a_x^*+a_y^*+2k\pi)+\phi(a_x^*-a_y^*+2k\pi) \Big]\\
    &=\frac{1}{\pi\sqrt{1-z_x^2}\sqrt{1-z_y^2}}\sum_{k=-\infty}^\infty \Big[ \phi(a_x^*-a_y^*+2k\pi)+\phi(a_x^*+a_y^*+2k\pi) \Big],
\end{align*}}
where (a) is derived by writing the summations $\sum_{k_x=-\infty}^\infty \sum_{k_y=-\infty}^\infty \{\cdot\}$ into $\sum_{l=-\infty}^\infty \sum_{k_x=-\infty}^\infty\{\cdot\}$  with $l=k_x-k_y$ and canceling terms, along with the symmetry of $\phi(\cdot)$. This gives the joint density of $z_x$ and $z_y$.

For the sine counterpart, with some abuse of notation, let us denote $z_x=\sin(t_x)$ and $z_y=\sin(t_y)$ from now on. Using similar argument, we have
\begin{align*}
    P(Z_x\leq z_x,Z_y\leq z_y)&=\sum_{k_x=-\infty}^\infty \sum_{k_y=-\infty}^\infty \int_{(2k_x+1)\pi-\sin^{-1}(z_x)}^{2(k_x+1)\pi+\sin^{-1}(z_x)}\int_{(2k_y+1)\pi-\sin^{-1}(z_y)}^{2(k_y+1)\pi+\sin^{-1}(z_y)} f(t_x,t_y)dt_x dt_y.
\end{align*}
After simplification, we finally arrive at
\begin{align}
    f(z_x,z_y)&=\frac{1}{\pi\sqrt{1-z_x^2}\sqrt{1-z_y^2}}\sum_{k=-\infty}^\infty \Big[ \phi(\sin^{-1}(z_x)-\sin^{-1}(z_y)+2k\pi)+\phi(\sin^{-1}(z_x)+\sin^{-1}(z_y)+(2k+1)\pi) \Big].  \label{eqn:joint-sin}
\end{align}
Considering $Z_x=\sin(t_x), Z_y=\sin(t_y)$. Since $\sin^{-1}(x)=\frac{\pi}{2}-\cos^{-1}(x)$, we can substitute into the density to derive
\begin{align*}
    f(z_x,z_y)&=\frac{1}{\pi\sqrt{1-z_x^2}\sqrt{1-z_y^2}}\sum_{k=-\infty}^\infty \Big[ \phi(\cos^{-1}(z_x)-\cos^{-1}(z_y)+2k\pi)+\phi(\cos^{-1}(z_x)+\cos^{-1}(z_y)+(2k+2)\pi) \Big]  \\
    &==\frac{1}{\pi\sqrt{1-z_x^2}\sqrt{1-z_y^2}}\sum_{k=-\infty}^\infty \Big[ \phi(a_x^*-a_y^*+2k\pi)+\phi(a_x^*+a_y^*+2k\pi) \Big],
\end{align*}
which is the same as the previous cosine transformation. This completes the proof.
\end{proof}

\subsection{Proposition~\ref{prop:joint inequality}}
\begin{proof}
Let us denote $\sigma=\sqrt{2(1-\rho)}\gamma$ for simplicity. By symmetry and exchangeability of $f$, to prove the desired result, it suffices to consider the case where both $z_x$ and $z_y$ are positive, i.e., $(z_x,z_y)\in (0,1]^2$. Define the notation $a_x^*=\sin^{-1}(z_x)\geq 0,a_y^*=\sin^{-1}(z_y)\geq 0$. From (\ref{eqn:joint-sin}), we deduct
\begin{align}
    &f(z_x,z_y)-f(z_x,-z_y) \nonumber\\
    &\propto \sum_{k=-\infty}^\infty \Big[ \phi_\sigma(a_x^*-a_y^*+2k\pi)+\phi_\sigma(a_x^*+a_y^*+(2k+1)\pi) \nonumber \\
    &\hspace{0.8in} -\phi_\sigma(a_x^*+a_y^*+2k\pi)-\phi_\sigma(a_x^*-a_y^*+(2k+1)\pi) \Big]  \nonumber\\
    &= \sum_{k=0}^{\infty} (-1)^k\Big[ \phi_\sigma(k\pi+d)-\phi_\sigma(k\pi+s)+\phi_\sigma((k+1)\pi-s)-\phi_\sigma((k+1)\pi-d) \Big],  \nonumber\\
    &= \phi_\sigma(d)-\phi_\sigma(s)+\sum_{k=1}^\infty \Big[ \phi_\sigma(k\pi-s)-\phi_\sigma(k\pi-d)-\phi_\sigma(k\pi+d)+\phi_\sigma(k\pi+s) \Big],  \nonumber \\
    &\triangleq \phi_\sigma(d)-\phi_\sigma(s)+\sum_{k=1}^\infty M_k, \label{eqn1}
\end{align}
where we let $d=a_x^*-a_y^*$ and $d=a_x^*+a_y^*$, and we use the fact that $\phi_\sigma(-x)=\phi_\sigma(x)$. Note that, we consider $z_y>0$ so that $d\neq s$, since when $z_y=0$ we trivially have $f(z_x,0)=f(z_x,0)$. For now, we assume that $z_x\geq z_y>0$, such that $d$ and $s$ are defined on the domain $0<s\leq \pi$ and $0\leq d< \min\{s,\pi-s\}$. Since
\begin{align*}
    \phi'_\sigma(x)=-\frac{x}{\sqrt{2\pi}\sigma^3}e^{-\frac{x^2}{2\sigma^2}}, \quad
    \phi''_\sigma(x)=-\frac{x^2-\sigma^2}{\sqrt{2\pi}\sigma^5}e^{-\frac{x^2}{2\sigma^2}},
\end{align*}
we know that $\phi_\sigma$ is piecewise concave on $(0,\sigma)$ and piecewise convex on $(\sigma,\infty)$. Thus,
\begin{align}
    \phi_\sigma(a)-\phi_\sigma(a+g)\geq \phi_\sigma(c)-\phi_\sigma(c+g) \label{eqn2}
\end{align}
for any $\sigma\leq a\leq c$ and $g\geq 0$. The equality holds only when $a=c$ or $g=0$. Consequently, under the assumption that $\sigma\leq \pi$, $M_k\geq 0$ for $k\geq 2$ since $2\pi-s\geq \sigma$, where the equality holds only when $d=s$, i.e., $z_y=0$. Furthermore, the piecewise convexity of $\phi_\sigma(\cdot)$ and (\ref{eqn2}) imply that for $\sigma\leq a< c$,
\begin{align}
    \frac{(c-a)c}{\sigma^2}e^{-\frac{c^2}{2\sigma^2}}< e^{-\frac{a^2}{2\sigma^2}}-e^{-\frac{c^2}{2\sigma^2}}< \frac{(c-a)a}{\sigma^2}e^{-\frac{a^2}{2\sigma^2}}. \label{eqn3}
\end{align}
Also note that the function $e^{-x}$ is convex on the real line, which gives for $\forall a<c$,
\begin{align}
    \frac{(c-a)(c+a)}{2\sigma^2}e^{-\frac{c^2}{2\sigma^2}}< e^{-\frac{a^2}{2\sigma^2}}-e^{-\frac{c^2}{2\sigma^2}}< \frac{(c-a)(c+a)}{2\sigma^2}e^{-\frac{a^2}{2\sigma^2}}. \label{eqn4}
\end{align}
Now that $M_k>0$ for $k\geq 2$, evaluating (\ref{eqn1}) we obtain
\begin{align*}
    (\ref{eqn1})& > \phi_\sigma(d)-\phi_\sigma(s)+ M_1\\
    & \overset{(a)}{>} \frac{1}{\sqrt{2\pi}\sigma^3}\Big[\frac{(s-d)(s+d)}{2}e^{-\frac{s^2}{2\sigma^2}}+\frac{(s-d)(2\pi-s-d)}{2}e^{-\frac{(\pi-d)^2}{2\sigma^2}}-(s-d)(\pi+d)e^{-\frac{(\pi+d)^2}{2\sigma^2}} \Big]\\
    &\overset{(b)}{\geq} \frac{s-d}{\sqrt{2\pi}\sigma^3}\Big[ \pi e^{-\frac{(\pi-d)^2}{2\sigma^2}}-(\pi+d)e^{-\frac{(\pi+d)^2}{2\sigma^2}} \Big],
\end{align*}
where (a) uses (\ref{eqn3}) and (\ref{eqn4}), and (b) is because $s\leq \pi-d$. It is easy to verify that the ratio
\begin{align*}
    \left. \Big( \pi e^{-\frac{(\pi-d)^2}{2\sigma^2}} \Big)\middle/ \Big( (\pi+d)e^{-\frac{(\pi+d)^2}{2\sigma^2}} \Big)=\frac{\pi}{\pi+d}e^{\frac{2\pi d}{\sigma^2}}\geq 1 \right.
\end{align*}
for $\sigma\leq \pi$ and $0\leq d< \min\{s,\pi-s\}< \frac{\pi}{2}$. Therefore, we have proved that $f(z_x,z_y)>f(z_x,-z_y)$, for $z_x\geq z_y>0$. Now, by exchangeability and symmetry of $f$, we have
\begin{align*}
    f(z_y,z_x)=f(z_x,z_y)>f(z_x,-z_y)=f(-z_x,z_y)=f(z_y,-z_x).
\end{align*}
Therefore, our result also holds for $z_y\geq z_x>0$. The proof is now complete.

\end{proof}

\subsection{Theorem~\ref{theo:StocQ}}
\begin{proof}
Denote the StocQ quantizer as $Q$. For each RFF $z$, assume $z\in [t_{i-1},t_i]$ for some $i$. We can then write $Q(z)=z+\epsilon$, where
$$\mathbb E[\epsilon]=t_i\frac{z-t_{i-1}}{t_i-t_{i-1}}+t_{i-1}\frac{t_i-z}{t_i-t_{i-1}}-z=0.$$
Thus, it follows that
\begin{align*}
    Var[\epsilon]=\mathbb E[\epsilon^2]
    &=t_i^2\frac{z-t_{i-1}}{t_i-t_{i-1}}+t_{i-1}^2\frac{t_i-z}{t_i-t_{i-1}}-z^2\\
    &=(t_i-z)(z-t_{i-1}).
\end{align*}
For two data vectors $u,v$, let $F^{StocQ}(u)=\sqrt 2Q(z_u)$ and $F^{StocQ}(v)=\sqrt 2Q(z_v)$, where $z_u=\cos(w^Tu+\tau)$ and $z_v=\cos(w^Tv+\tau)$ follows the distribution $f$ given by Theorem~\ref{theo:joint_RFF1}. We can write $Q(z_u)=z_u+\epsilon_u$, $Q(z_v)=z_v+\epsilon_v$ where $\epsilon_u$ and $\epsilon_v$ are independent. Let $\hat =F^{StocQ}(u)F^{StocQ}(v)$. We have
\begin{align*}
    \mathbb E[\hat K_{StocQ}]&=2\mathbb E[(z_u+\epsilon_u)(z_v+\epsilon_v)]\\
    &=2\mathbb E[z_uz_v]=K(u,v),
\end{align*}
implying that StocQ estimate is unbiased. The variance factor can be computed as
\begin{align}
    Var[\hat K_{StocQ}]&=4\mathbb E[(z_u+\epsilon_u)^2(z_v+\epsilon_v)^2]-K(u,v)^2  \nonumber\\
    &=4\mathbb E[z_u^2\epsilon_v^2+z_v^2\epsilon_u^2+\epsilon_u^2\epsilon_v^2]+Var[\hat K]\triangleq A+Var[\hat K],  \label{stoc-var0}
\end{align}
where $Var[\hat K]$ is the variance of full-precision RFF kernel estimator. Obviously, $A>0$, thus StocQ estimator always has larger variance than full-precision RFF. Continuing our analysis,
\begin{align*}
    \mathbb E[z_u^2\epsilon_v^2]&=\mathbb E_{z_u,z_v}z_u^2 \mathbb E[\epsilon_v^2|z_v]\\
    &=\int_{-1}^1dz_u \Big(\sum_{j=1}^{2^b-1}\int_{t_{j-1}}^{t_j}(t_j-z_v)(z_v-t_{j-1})z_u^2 f(z_u,z_v)dz_v\Big)\\
    &=\sum_{i=1}^{2^b-1}\sum_{j=1}^{2^b-1}\int_{t_{j-1}}^{t_j}\int_{t_{i-1}}^{t_i} \Big((t_{j-1}+t_j)z_vz_u^2-z_v^2z_u^2-t_{j-1}t_jz_u^2 \Big) f(z_u,z_v) dz_udz_v.
\end{align*}
By symmetry of density function $f$, we know that $\mathbb E[z_v^2\epsilon_u^2]=\mathbb E[z_u^2\epsilon_v^2]$. It remains to compute $\mathbb E[\epsilon_u^2\epsilon_v^2]$. By similar reasoning, we have
\begin{align*}
    \mathbb E[\epsilon_u^2\epsilon_v^2]&=\sum_{i=1}^{2^b-1}\sum_{j=1}^{2^b-1}\int_{t_{j-1}}^{t_j}\int_{t_{i-1}}^{t_i} (t_i-z_u)(z_u-t_{i-1})(t_j-z_v)(z_v-t_{j-1}) f(z_u,z_v) dz_udz_v\\
    &=\sum_{i=1}^{2^b-1}\sum_{j=1}^{2^b-1}\int_{t_{j-1}}^{t_j}\int_{t_{i-1}}^{t_i} \Big( (t_{i-1}+t_i)(t_{j-1}+t_j)z_uz_v- (t_{i-1}+t_i)z_uz_v^2-(t_{j-1}+t_j)z_vz_u^2+z_u^2z_v^2\\
    &\hspace{0.3in}  -(t_{i-1}+t_i)t_{j-1}t_jz_u-(t_{j-1}+t_j)t_{i-1}t_iz_v+t_{j-1}t_jz_u^2+t_{i-1}t_iz_v^2+t_{i-1}t_it_{j-1}t_j \Big) f(z_u,z_v) dz_udz_v\\
    &\overset{(a)}{=}\sum_{i=1}^{2^b-1}\sum_{j=1}^{2^b-1}\int_{t_{j-1}}^{t_j}\int_{t_{i-1}}^{t_i} \Big( (t_{i-1}+t_i)(t_{j-1}+t_j)z_uz_v-2(t_{j-1}+t_j)z_vz_u^2+z_u^2z_v^2\\
    &\hspace{3in} +2 t_{j-1}t_jz_u^2+t_{i-1}t_it_{j-1}t_j \Big) f(z_u,z_v) dz_udz_v,
\end{align*}
where equation $(a)$ is due to the symmetry of density $f$ and the borders $t_0<...<t_{2^b-1}$. Substituting above expressions into (\ref{stoc-var0}) and cancelling terms, we obtain
\begin{align*}
    A&=4 \sum_{i=1}^{2^b-1}\sum_{j=1}^{2^b-1}\int_{t_{j-1}}^{t_j}\int_{t_{i-1}}^{t_i} \Big( (t_{i-1}+t_i)(t_{j-1}+t_j)z_uz_v+ t_{i-1}t_it_{j-1}t_j-z_u^2z_v^2 \Big) f(z_u,z_v) dz_udz_v\\
    &=4 \sum_{i=1}^{2^b-1}\sum_{j=1}^{2^b-1} \Big[ (t_{i-1}+t_i)(t_{j-1}+t_j)\kappa_{i,j} + t_{i-1}t_it_{j-1}t_jp_{i,j} \Big]-4\mathbb E[z_u^2z_v^2].
\end{align*}
Therefore,
\begin{align*}
    Var[\hat K_{StocQ}]&=4 \sum_{i=1}^{2^b-1}\sum_{j=1}^{2^b-1} \Big[ (t_{i-1}+t_i)(t_{j-1}+t_j)\kappa_{i,j} + t_{i-1}t_it_{j-1}t_jp_{i,j} \Big]-K(u,v)^2.
\end{align*}
The proof is completed by noting that StocQ estimator is the average of i.i.d. Bernoulli random variables.
\end{proof}

\subsection{Theorem~\ref{theo:mean-var}}
\begin{proof}
For simplicity, we prove the result specifically for LM-RFF quantization. Similar arguments holds for general quantizers. The \textit{Chebyshev polynomials}~\citep{1995polynomials} of the first kind are defined through trigonometric identities
\begin{align*}
    T_n(\cos(x))=\cos(n\cos(x)),
\end{align*}
where admit the following recursion,
\begin{align*}
    &T_0(x)=1,\quad T_1(x)=x,\\
    &T_{i+1}(x)=2xT_i(x)-T_{i-1}(x),\quad i\geq 2.
\end{align*}
$\{T_0,T_1,...\}$ forms an orthogonal basis of the function space on $[-1,1]$ with finite number of discontinuities. Precisely, define the inner product w.r.t. measure $\frac{1}{\sqrt{1-x^2}}$ as
\begin{align*}
    \langle f(x),g(x) \rangle=\int_{-1}^1 f(x)g(x)\frac{1}{\sqrt{1-x^2}}dx.
\end{align*}
Then orthogonality holds:
\begin{align*}
    \int_{-1}^1 T_i(x) T_j(x)\frac{1}{\sqrt{1-x^2}}dx=\begin{cases}
    0, & i\neq j,\\
    \pi,& i=j=0,\\
    \frac{\pi}{2},&i=j\neq 0.
    \end{cases}
\end{align*}
By Chebyshev functional decomposition, our   LM quantizer can be written as
\begin{align*}
    Q(x)=\sum_{k=0}^\infty \alpha_k T_k(x),
\end{align*}
where $\alpha_k$ are computed through the inner products,
\begin{align*}
    &\alpha_0=\frac{2}{\pi}\int_{-1}^1 Q(x)T_0(x)\frac{dx}{\sqrt{1-x^2}}=0,\\
    &\alpha_1=\frac{2}{\pi}\int_{-1}^1 Q(x)T_1(x)\frac{dx}{\sqrt{1-x^2}}=1-2D,\\
    &\alpha_2=\frac{2}{\pi}\int_{-1}^1 Q(x)T_2(x)\frac{dx}{\sqrt{1-x^2}}=0,\\
    &\alpha_3=\frac{2}{\pi}\int_{-1}^1 Q(x)T_3(x)\frac{dx}{\sqrt{1-x^2}},\\
    &\hspace{0.3in} ...
\end{align*}
with $D$ the distortion of $Q$ given in equation~(7) of the main paper. Firstly, it is easy to show that $\vert\mathbb E[T_i(z_x)T_j(z_y)]\vert\leq \mathbb E[T_i(z_x)^2]=\frac{1}{2}$. Note that $\alpha_k=0$ when $k$ is even because $T_k(x)$ is even function and $Q(x)$ is odd. Recall $u,v$ are two normalized data vectors with correlation $\rho$. Denote $z_x=\cos(\gamma x+\tau)$ and $z_y=\cos(\gamma y+\tau)$ with distribution $f(z_x,z_y)$, where $(x,y)\sim N\big(0,\begin{pmatrix}
1 & \rho \\
\rho & 1
\end{pmatrix} \big)$, $\tau\sim uniform(0,2\pi)$. It follows that
\begin{align}
    \mathbb E[\sqrt 2Q(z_x)\sqrt 2Q(z_y)]&=2\int_{-1}^1 \int_{-1}^1 Q(z_x)Q(z_y)f(z_x,z_y) dz_xdz_y \nonumber\\
    &=2\int_{-1}^1 \int_{-1}^1 (\sum_{i=1,odd}^\infty \alpha_i T_i(z_x))(\sum_{j=1,odd}^\infty \alpha_j T_j(z_y))f(z_x,z_y) dz_xdz_y \nonumber\\
    &=(1-2D)^2K(u,v)+2\sum_{i=1,odd}^\infty \sum_{j=3,odd}^\infty \alpha_i\alpha_j \int_{-1}^1 \int_{-1}^1 T_i(z_x)T_j(z_y)f(z_x,z_y) dz_xdz_y. \label{eqn:dither-general}
\end{align}
This proves the first part. There is an intrinsic constraint on $\alpha_i$, $i=3,5,...$. First, we can compute the cosine of $Q(x)$ and each $T_i(x)$ as
\begin{align*}
    c_i&=\frac{\int_{-1}^1 Q(x)T_i(x)\frac{dx}{\sqrt{1-x^2}}}{\sqrt{\int_{-1}^1 Q(x)^2\frac{dx}{\sqrt{1-x^2}}\int_{-1}^1 T_i(x)^2\frac{dx}{\sqrt{1-x^2}}}}\\
    &=\frac{\frac{\pi}{2}\alpha_i}{\sqrt{(\frac{1}{2}-D)\pi}\sqrt{\frac{\pi}{2}}}\\
    &=\frac{\alpha_i}{\sqrt{1-2D}}.
\end{align*}
Since the Chebyshev polynomials form an orthogonal basis of function space on $[-1,1]$, it holds that $\sum_{i=0}^\infty c_i^2=1$. Therefore, we have $\sum_{i=0}^\infty \alpha_i^2=1-2D$. Now that $\alpha_i=0$ when $i$ is even, and $\alpha_1=1-2D$, we then have $\sum_{i=3,odd}^\infty \alpha_i^2=1-2D-(1-2D)^2=2D(1-2D)$.

When $\rho=0$, from (\ref{eqn:dither-general}), it is easy to see that all the integrals would be zero by independence. Thus, the estimated kernel $\mathbb E[\sqrt 2Q(z_x)\sqrt 2Q(z_y)]=(1-2D)^2 K(u,v)$.

When $\rho=1$ ($K(u,v)=1$), we have $\int_{-1}^1 T_i(z_x)T_j(z_x) f(z_x)dz_x=0$ for $i\neq j$ by orthogonality of Chebyshev polynomials, where $f(z_x)$ is the marginal distribution of $z_x$. It follows that
\begin{align*}
    \mathbb E[\sqrt 2Q(z_x)\sqrt 2Q(z_y)]&=(1-2D)^2+\sum_{i=3,odd}^\infty \alpha_i^2\\
    &=(1-2D)^2+2D(1-2D)\\
    &=1-2D.
\end{align*}
This completes the proof of the theorem.
\end{proof}

\subsection{Theorem~\ref{theo: mean-var-norm}}
\begin{proof}
Denote $\bm {w}=\cos(\gamma \bm x+\tau)$, $\bm{z}=\cos(\gamma \bm{y}+\tau)$, with $(\bm{x,y})$ are random vectors with i.i.d. entries from $N(0,\begin{pmatrix}
1 & \rho \\
\rho & 1
\end{pmatrix})$, and $\tau\sim uniform(0,2\pi)$ is also a vector with i.i.d. entries. Recall the notation $\zeta_{s,t}=\mathbb E[Q(w_1)^s Q(z_1)^t]$, where $Q$ is our LM-RFF   quantizer. By Taylor expansion at the expectations, we have as $m\rightarrow \infty$,
\begin{align*}
    \mathbb E[\hat K_{n,Q}]&=\frac{\mathbb E[\frac{1}{m}\sum_{i=1}^m Q(w_i) Q(z_i)]}{\mathbb E[\sqrt{\frac{1}{m^2}\Vert Q(\bm w)\Vert^2 \Vert Q(\bm z)\Vert^2}]}+\mathcal O(\frac{1}{m})\\
    &\triangleq \frac{\zeta_{1,1}}{\mathbb E[\sqrt \Lambda]}+\mathcal O(\frac{1}{m}).
\end{align*}
Applying Taylor expansion again,
\begin{align*}
    \mathbb E[\sqrt \Lambda]&=\mathbb E\Big[\sqrt{\mathbb E[\Lambda]} +\frac{\Lambda-\mathbb E[\Lambda]}{2\sqrt{\mathbb E[\Lambda]}}+\mathcal O((\Lambda-\mathbb E[\Lambda])^2)\Big]\\
    &=\mathbb E[\Lambda]+\mathcal O(\frac{1}{m}),\quad m\rightarrow
    \infty.
\end{align*}
Furthermore, we have the expectation of $\Lambda$ as
\begin{align*}
    \mathbb E[\Lambda]&=\frac{1}{m^2}\mathbb E\Big[ \big(\sum_{i=1}^m Q(w_i)^2\big) \big(\sum_{i=1}^m Q(z_i)^2\big) \Big]\\
    &=\frac{1}{m^2}\Big[ \sum_{i\neq j}Q(w_i)^2Q(z_j)^2 +\sum_{i=1}^m Q(w_i)^2Q(z_i)^2 \Big]\\
    &=\frac{m-1}{m}\mathbb E[Q(w_1)^2Q(z_2)^2]+\frac{1}{m}\mathbb E[Q(w_1)^2Q(z_1)^2]\\
    &=\zeta_{2,0}^2,\quad m\rightarrow \infty.
\end{align*}
Consequently, we obtain
\begin{align*}
    \mathbb E[\hat K_{n,Q}]=\frac{\zeta_{1,1}}{\zeta_{2,0}},\quad m\rightarrow\infty.
\end{align*}
This completes the proof for asymptotic mean. With a little abuse of notation, let $\hat K_{n,Q}=\frac{a}{\sqrt{bc}}$, with
\begin{align*}
    a=\frac{\langle Q(\bm w),Q(\bm z) \rangle}{k},\ b=\frac{\Vert Q(\bm w) \Vert^2}{k},\ c=\frac{\Vert Q(\bm z) \Vert^2}{k}.
\end{align*}
We have
\begin{align*}
    \mathbb E[a]&=\zeta_{1,1},\quad Var[a]=\frac{\zeta_{2,2}-\zeta_{1,1}^2}{m},\\
    \mathbb E[b]&=\zeta_{2,0}=\mathbb E[c],\quad Var[b]=\frac{\zeta_{4,0}-\zeta_{2,0}^2}{m}=Var[c],\\
    Cov(a,b)&=\frac{1}{m^2}\mathbb E[(\sum_{i=1}^m Q(w_i)Q(z_i))(\sum_{i=1}^m Q(w_i)^2)]-\zeta_{1,1}\zeta_{2,0}\\
    &=\frac{m\zeta_{3,1}+m(m-1)\zeta_{1,1}\zeta_{2,0}}{m^2}-\zeta_{1,1}\zeta_{2,0}\\
    &=\frac{\zeta_{3,1}-\zeta_{1,1}\zeta_{2,0}}{m}=Cov(a,c),\\
    Cov(b,c)&=\frac{\zeta_{2,2}-\zeta_{2,0}^2}{m}.
\end{align*}
We can formulate the covariance matrix
\begin{align*}
    Cov(a,b,c)=\frac{1}{m}\begin{pmatrix}
    \zeta_{2,2}-\zeta_{1,1}^2 & \zeta_{3,1}-\zeta_{1,1}\zeta_{2,0} & \zeta_{3,1}-\zeta_{1,1}\zeta_{2,0}\\
    \zeta_{3,1}-\zeta_{1,1}\zeta_{2,0} & \zeta_{4,0}-\zeta_{2,0}^2 & \zeta_{2,2}-\zeta_{2,0}^2 \\
    \zeta_{3,1}-\zeta_{1,1}\zeta_{2,0} & \zeta_{2,2}-\zeta_{2,0}^2 & \zeta_{4,0}-\zeta_{2,0}^2
    \end{pmatrix}.
\end{align*}
The gradient vector at the expectations is
\begin{align*}
    \nabla \hat K_{n,Q}(\mathbb E[a],\mathbb E[b],\mathbb E[c])=(\frac{1}{\zeta_{2,0}},-\frac{\zeta_{1,1}}{2\zeta_{2,0}^2},-\frac{\zeta_{1,1}}{2\zeta_{2,0}^2}).
\end{align*}
By Taylor expansion, it holds that
\begin{align*}
    Var[\hat K_{n,Q}]=\nabla \hat K_{n,Q}(\mathbb E[a],\mathbb E[b],\mathbb E[c])^T Cov(a,b,c) \nabla \hat K_{n,Q}(\mathbb E[a],\mathbb E[b],\mathbb E[c])+\mathcal O(\frac{1}{m^2}).
\end{align*}
The theorem is proved by plugging in the expressions.
\end{proof}

\subsection{Theorem~\ref{theo:BD-variance}}
\begin{proof}
Let $z_x=\cos(\gamma X+\tau)$, $z_y=\cos(\gamma Y+\tau)$ where $(X,Y)\sim N\big(0,\begin{pmatrix}
1 & \rho \\
\rho & 1
\end{pmatrix}\big)$, $\tau\sim uniform(0,2\pi)$. Denote $\zeta_{s,t}=\mathbb E[Q(z_x)^sQ(z_y)^t]$. Recalling Theorem~\ref{theo:mean-var} and Theorem~\ref{theo: mean-var-norm}, we have asymptotically (omitting lower order terms)
\begin{align*}
    &\mathbb E[\hat K_Q]=2\zeta_{1,1},\quad Var[\hat K_Q]=\frac{4}{m}(\zeta_{2,2}-\zeta_{1,1}^2),\\
    &\mathbb E[\hat K_{n,Q}]=\frac{\zeta_{1,1}}{\zeta_{2,0}},\quad Var[\hat K_{n,Q}]=\frac{1}{m}\Big( \frac{\zeta_{2,2}}{\zeta_{2,0}^2}-\frac{2\zeta_{1,1}\zeta_{3,1}}{\zeta_{2,0}^3}+\frac{\zeta_{1,1}^2(\zeta_{4,0}+\zeta_{2,2})}{2\zeta_{2,0}^4} \Big).
\end{align*}
Thus, we can compute the debiased estimator variance as (after simplification)
\begin{align*}
    &Var^{db}[\hat K_Q]=\frac{K(u,v)^2}{m}\Big( \frac{\zeta_{2,2}}{\zeta_{1,1}^2}-1 \Big),\\
    &Var^{db}[\hat K_{n,Q}]=\frac{K(u,v)^2}{m}\Big( \frac{\zeta_{2,2}}{\zeta_{1,1}^2}-\frac{2\zeta_{3,1}}{\zeta_{1,1}\zeta_{2,0}}+\frac{\zeta_{4,0}\zeta_{2,2}}{2\zeta_{2,0}^2} \Big).
\end{align*}
Taking the difference, we obtain
\begin{align*}
    Var^{db}[\hat K_Q]-Var^{db}[\hat K_{n,Q}]&\propto 4\zeta_{2,0}\zeta_{3,1}+\zeta_{1,1}(\zeta_{4,0}+\zeta_{2,2})-2\zeta_{1,1}\zeta_{2,0}^2\\
    &\geq 4\zeta_{2,0}\zeta_{3,1}+\zeta_{1,1}(\zeta_{2,2}-\zeta_{2,0}^2)\triangleq M(\rho),
\end{align*}
where the inequality is due to the fact that $\zeta_{4,0}-\zeta_{2,0}^2=Var[Q^2(z_x)]\geq 0$. Here we denote $M$ as a function of $\rho$. At $\rho=0$, we have
\begin{align*}
    \zeta_{3,1}=0,\quad \zeta_{2,2}=\zeta_{2,0}^2,
\end{align*}
so that $M(0)=0$. At $\rho=1$, it holds that
\begin{align*}
    \zeta_{3,1}=\zeta_{2,2}=\zeta_{4,0},
\end{align*}
hence $M(1)>0$. Notice that $Q(\cdot)$ and $Q^3(\cdot)$ are non-decreasing odd functions, and $Q^2(\cdot)$ is a even function. For  $\rho\in[0,1]$, since $\sqrt{2(1-\rho)}\gamma\leq \sqrt 2\gamma \leq \pi$ by assumption, it follows from Theorem~\ref{theo:discrete monotone} that $\zeta_{1,1}$, $\zeta_{2,2}$ and $\zeta_{3,1}$ are all increasing in $\rho$ on $[0,1]$. Consequently, $M(\rho)>0$ for any $\rho\in[0,1]$. The desired result thus follows.
\end{proof}

\subsection{Lemma~\ref{lemma:continuous monotone}}

\begin{lemma}[Stein's Lemma] \label{stein lemma}
Suppose $X\sim N(\mu,\sigma^2)$, and $g$ is a differentiable function such that $\mathbb E[g(X)(X-\mu)]$ and $\mathbb E[g'(X)]$ exist. Then, $\mathbb E[g(X)(X-\mu)]=\sigma^2\mathbb E[g'(X)]$.
\end{lemma}

\begin{proof}(of Lemma~\ref{lemma:continuous monotone})
We use the technique of Gaussian interpolation and Stein's Lemma. First, we formulate $Y=\gamma\rho X+\gamma\sqrt{1-\rho^2}Z$ where $Z\sim N(0,1)$ independent of $X$. By continuity and boundedness of $g_1$ and $g_2$, it holds that
\begin{align*}
    &\frac{\partial \mathbb E[g_1(\cos(s_x))g_2(\cos(s_y))]}{\partial \rho}\\
    &=\frac{\partial \mathbb E_{X,Z,\tau}[g_1(\cos(\gamma X+\tau))g_2(\cos(\gamma\rho X+\gamma\sqrt{1-\rho^2}Z+\tau))]}{\partial \rho}\\
    &=-\mathbb E_{X,Z,\tau}\Big[\underbrace{ g_1(\cos(\gamma X+\tau))g_2'(\cos(\gamma\rho X+\gamma\sqrt{1-\rho^2}Z+\tau))\sin(\gamma\rho X+\gamma\sqrt{1-\rho^2}Z+\tau)}_{\Upsilon(X,Z;\rho)}\big( \gamma X-\frac{\gamma\rho Z}{\sqrt{1-\rho^2}} \big) \Big].
\end{align*}
We analyze two parts respectively. By Lemma~\ref{stein lemma} and law of total expectation, we have
\begin{align}
    &\mathbb E_{X,Z,\tau}[\Upsilon(X,Z;\rho)\gamma X]  \nonumber\\
    &=\mathbb E_{Z,\tau}\mathbb E_X\Big[ -\gamma^2 g_1'(\cos(\gamma X+\tau))\sin(\gamma X+\tau)g_2'(\cos(\gamma Y+\tau))\sin(\gamma\rho X+\gamma\sqrt{1-\rho^2}Z+\tau) \nonumber\\
    &\hspace{0.3in} - \gamma^2\rho g_1(\cos(\gamma X+\tau))g_2''(\cos(\gamma\rho X+\gamma\sqrt{1-\rho^2}Z+\tau))\sin^2(\gamma\rho X+\gamma\sqrt{1-\rho^2}Z+\tau)    \nonumber \\
    &\hspace{0.4in} +\gamma^2\rho g_1(\cos(\gamma X+\tau))g_2'(\cos(\gamma\rho X+\gamma\sqrt{1-\rho^2}Z+\tau))\cos(\gamma\rho X+\gamma\sqrt{1-\rho^2}Z+\tau) | Z,\tau \Big], \label{eqn5}
\end{align}
and
\begin{align}
    &\mathbb E_{X,Z,\tau}\big[\Upsilon(X,Z;\rho)\frac{\gamma\rho Z}{\sqrt{1-\rho^2}} \big]  \nonumber\\
    &=\mathbb E_{X,\tau}\mathbb E_Z\Big[ -\gamma^2\rho g_1(\cos(\gamma X+\tau)) g_2''(\cos(\gamma\rho X+\gamma\sqrt{1-\rho^2}Z+\tau))\sin^2(\gamma\rho X+\gamma\sqrt{1-\rho^2}Z+\tau)  \nonumber\\
    &\hspace{0.3in} +\gamma^2\rho g_1(\cos(\gamma X+\tau))g_2'(\cos(\gamma\rho X+\gamma\sqrt{1-\rho^2}Z+\tau))\cos(\gamma\rho X+\gamma\sqrt{1-\rho^2}Z+\tau)|X,\tau \Big].  \label{eqn6}
\end{align}
Combining (\ref{eqn5}) and (\ref{eqn6}), we get
\begin{align*}
    &\frac{\partial \mathbb E[g_1(\cos(s_x))g_2(\cos(s_y))]}{\partial \rho}\\
    &=\mathbb E_{X,Z,\tau}\Big[ \gamma^2 g_1'(\cos(\gamma X+\tau))\sin(\gamma X+\tau)g_2'(\cos(\gamma Y+\tau))\sin(\gamma\rho X+\gamma\sqrt{1-\rho^2}Z+\tau) \Big] \\
    &=\gamma^2\mathbb E_{X,Y,\tau}\big[g_1'(\cos(s_x))\sin(s_x) g_2'(\cos(s_y))\sin(s_y) \big],
\end{align*}
which gives the desired expression.

To prove the monotonicity, suppose that $g_1$ and $g_2$ are increasing odd or non-constant even functions. So, $g_1'(-x)g_2'(-x)=g_1'(x)g_2'(x)>0$, $\forall x\in [-1,1]$. Assume $\sqrt{2(1-\rho)}\gamma\leq\pi$, and denote $f(x,y)$ as the joint density given by Theorem~\ref{fig:joint density}. We can write
\begin{align*}
    \frac{\partial \mathbb E[g_1(\cos(s_x))g_2(\cos(s_y))]}{\partial \rho}
    &=\gamma^2\int_{-1}^1\int_{-1}^1 z_xz_yg_1'(\sqrt{1-z_x^2})g_2'(\sqrt{1-z_y^2})f(z_x,z_y)dz_x dz_y\\
    &\overset{(a)}{=}2\gamma^2\Big(\int_0^1 \int_{0}^1z_xz_yg_1'(\sqrt{1-z_x^2})g_2'(\sqrt{1-z_y^2})f(z_x,z_y)dz_x dz_y\\
    &\hspace{0.7in} + \int_0^1 \int_{-1}^0 z_xz_yg_1'(\sqrt{1-z_x^2})g_2'(\sqrt{1-z_y^2})f(z_x,z_y)dz_x dz_y  \Big)\\
    &=2\gamma^2\int_0^1 \int_0^1 z_xz_yg_1'(\sqrt{1-z_x^2})g_2'(\sqrt{1-z_y^2})[f(z_x,z_y)-f(z_x,-z_y)]dz_x dz_y\\
    &\overset{(b)}{>}0,
\end{align*}
where (a) is due to the symmetry of $f$ and $g$, and (b) is a consequence of Proposition~\ref{prop:joint inequality} that $f(z_x,z_y)>f(z_x,-z_y)$ for all $z_x,z_y\in (0,1]^2$, provided that $\sqrt{2(1-\rho)}\gamma\leq\pi$. The proof is complete.

\end{proof}

\subsection{Theorem~\ref{theo:discrete monotone}}
\begin{proof}
Since $Q_1$ and $Q_2$ both are non-decreasing and have finite number of discontinuities, by Baire's Characterization Theorem we know that each of them is the pointwise limit of a sequence of continuous increasing functions. Suppose that $\{g_{1,n}\}$ and $\{g_{2,n}\}$ are two sequences of continuous increasing functions such that as $n\rightarrow\infty$, $g_{1,n}\rightarrow Q_1$ and $g_{2,n}\rightarrow Q_2$ with pointwise convergence. By dominated convergence theorem, we have
\begin{align*}
    \frac{\partial \mathbb E[Q_1(z_x)Q_2(z_y)]}{\partial \rho}&= \frac{\partial \mathbb E[\displaystyle{\lim_{n\rightarrow \infty}} g_{1,n}(z_x) \displaystyle{\lim_{n\rightarrow \infty}} g_{2,n}(z_y)]}{\partial \rho}\\
    &=\displaystyle{\lim_{n\rightarrow \infty}} \frac{\partial \mathbb E[ g_{1,n}(z_x) g_{2,n}(z_y)]}{\partial \rho}>0, \\
\end{align*}
where Lemma~\ref{lemma:continuous monotone} is adopted for continuous $g_{1,n}$ and $g_{2,n}$ functions.
\end{proof}

\end{document}